%% file: likelihood_displacement_main.tex
\newcommand*{\ICLR}{}

\newcommand*{\CAMREADY}{}

\ifdefined\ARXIV
	\documentclass[10pt]{article}
	\usepackage{libertine} 
	\usepackage{fullpage}
	
	\usepackage{natbib}
	
	\renewcommand{\cite}[1]{\citep{#1}}
	\setcitestyle{authoryear,round,citesep={;},aysep={,},yysep={;}}
	
	\usepackage[utf8]{inputenc} 
	\usepackage[T1]{fontenc}    
	\usepackage{microtype}      
	
	\usepackage{tabularx}
        \usepackage{xcolor}
	\usepackage{thmtools}
	\usepackage{thm-restate}
	\usepackage[left=1.3in,right=1.3in,top=0.9in]{geometry}
	\usepackage{footmisc}
	\usepackage{comment}
	\usepackage{hyperref}
	\definecolor{mydarkblue}{rgb}{0,0.08,0.5}
	\hypersetup{ %
		pdftitle={},
		pdfauthor={},
		pdfsubject={},
		pdfkeywords={},
		colorlinks=true,
		linkcolor=mydarkblue,
		citecolor=mydarkblue,
		filecolor=mydarkblue,
		urlcolor=mydarkblue}
	
	\skip\footins 9pt plus 4pt minus 2pt
	\def\footnoterule{\kern-3pt \hrule width 12pc \kern 2.6pt }
	\leftmargini\leftmargin \leftmarginii 2em
	\renewenvironment{abstract}%
	{%
		\vskip 0in%
		\centerline%
		{\large\bf Abstract}%
		\vspace{-1ex}%
		\begin{quote}%
		}
		{
			\par%
		\end{quote}%
		\vskip 0ex%
		\vspace{-2mm}
	}
	
	\title{\vskip -4ex \bf{Unintentional Unalignment: Likelihood Displacement in Direct Preference Optimization} \vskip -1ex}

	\newcommand{\aff}[1]{\textsuperscript{\normalfont #1}}
	\author{
		Noam Razin\aff{$\dagger$}, Sadhika Malladi\aff{$\dagger$}, Adithya Bhaskar\aff{$\dagger$}, \\[0.5mm] 
		Danqi Chen\aff{$\dagger$}, 	Sanjeev Arora\aff{$\dagger$}, Boris Hanin\aff{$\ddagger$} \\[3mm]
		{
			\fontsize{10.5}{11}\selectfont
			\textsuperscript{$\dagger$}\textit{Princeton Language and Intelligence, Princeton University}\;
			\textsuperscript{$\ddagger$}\textit{Princeton ORFE}
		}
	}
	\date{}
\fi
\ifdefined\ARXIVTWO
	\documentclass{article} %
	\usepackage[table]{xcolor}
	\usepackage{colm2024_conference}
 \usepackage{libertine} 
	\usepackage{pifont}
	\usepackage{footmisc}
	\usepackage{hyperref}
	\definecolor{mydarkblue}{rgb}{0,0.08,0.5}
	\hypersetup{ %
		pdftitle={},
		pdfauthor={},
		pdfsubject={},
		pdfkeywords={},
		colorlinks=true,
		linkcolor=mydarkblue,
		citecolor=mydarkblue,
		filecolor=mydarkblue,
		urlcolor=mydarkblue
	}
	\usepackage{url}
	
	\title{\vskip -2ex Unintentional Unalignment: Likelihood Displacement \\ in Direct Preference Optimization \vskip -0.5ex}
	
	\colmfinalcopy
	
	\newcommand{\aff}[1]{\textsuperscript{\normalfont #1}}
	\author{
		\vspace{-0.2mm}
		Noam Razin\aff{$\dagger$}, Sadhika Malladi\aff{$\dagger$}, Adithya Bhaskar\aff{$\dagger$}, Danqi Chen\aff{$\dagger$}, \\[0.5mm]  
		\textbf{~Sanjeev Arora\aff{$\dagger$}, Boris Hanin\aff{$\ddagger$}} \\[2mm]
		\textsuperscript{$\dagger$ }Princeton Language and Intelligence, Princeton University \; \textsuperscript{$\ddagger$ }Princeton ORFE
	}

\fi

\ifdefined\NEURIPS
	\PassOptionsToPackage{numbers}{natbib}
	\documentclass{article}
	\ifdefined\CAMREADY
		\usepackage[final]{neurips_2020}
	\else
		\usepackage{neurips_2020}
	\fi
	\usepackage{footmisc}
	\usepackage[utf8]{inputenc} 
	\usepackage[T1]{fontenc}    
	\usepackage{hyperref}       	
	\usepackage{url}            	
	\usepackage{amsfonts}       
	\usepackage{nicefrac}       	
	\usepackage{microtype}      
\fi

\ifdefined\CVPR
	\documentclass[10pt,twocolumn,letterpaper]{article}
	\usepackage{footmisc}
	\usepackage{cvpr}
	\usepackage{times}
	\usepackage{epsfig}
	\usepackage{graphicx}
	\usepackage{amsmath}
	\usepackage{amssymb}
	\usepackage[square,comma,numbers]{natbib}
\fi
\ifdefined\AISTATS
	\documentclass[twoside]{article}
	\ifdefined\CAMREADY
		\usepackage[accepted]{aistats2015}
	\else
		\usepackage{aistats2015}
	\fi
	\usepackage{url}
	\usepackage[round,authoryear]{natbib}
\fi
\ifdefined\ICML
	\documentclass{article}
	\usepackage{footmisc}
	\usepackage{microtype}
	\usepackage{graphicx}
	\usepackage{hyperref}
	
	\ifdefined\CAMREADY
		\usepackage[accepted]{icml2021}
	\else
		\usepackage{icml2021}
	\fi
\fi
\ifdefined\ICLR
	\documentclass{article}
	\usepackage{iclr2025_conference,times}
	\usepackage{footmisc}
	\usepackage{hyperref}
	\definecolor{mydarkblue}{rgb}{0,0.08,0.5}
	\hypersetup{ %
		pdftitle={},
		pdfauthor={},
		pdfsubject={},
		pdfkeywords={},
		colorlinks=true,
		linkcolor=mydarkblue,
		citecolor=mydarkblue,
		filecolor=mydarkblue,
		urlcolor=mydarkblue
	}
        \usepackage{url}

	\ifdefined\CAMREADY
		\iclrfinalcopy
	\fi
\fi
\ifdefined\COLT
	\ifdefined\CAMREADY
		\documentclass{colt2017}
	\else
		\documentclass[anon,12pt]{colt2017}
	\fi
    \usepackage{times}
\fi

\usepackage{booktabs,multirow}       
\usepackage{color}
\usepackage{xcolor}         
\usepackage{amsfonts}
\usepackage{algorithm}
\usepackage{algorithmic}
\usepackage{graphicx}
\usepackage{nicefrac}
\usepackage{bbm}
\usepackage{wrapfig}
\usepackage{tikz}
\usepackage[skins,theorems]{tcolorbox}
\usepackage[titletoc,title]{appendix}
\usepackage{stmaryrd}
\usepackage{comment}
\usepackage{endnotes}
\usepackage{hyperendnotes}
\usepackage{enumerate}
\usepackage{enumitem}
\usepackage{makecell}
\usepackage{wrapfig}
\usepackage[normalem]{ulem}

\usepackage{float}

\ifdefined\COLT
\else
	\usepackage{amsmath,amsthm,amssymb,mathtools}
\fi
\usepackage[small]{caption}
\usepackage[caption = false]{subfig}

\usepackage[extdef=true]{delimset}
\usepackage[capitalize,noabbrev]{cleveref}

\ifdefined\COLT
	\newtheorem{claim}[theorem]{Claim}
	\newtheorem{fact}[theorem]{Fact}
	\newtheorem{procedure}{Procedure}
	\newtheorem{conjecture}{Conjecture}	
	\newtheorem{hypothesis}{Hypothesis}	
	\newcommand{\qed}{\hfill\ensuremath{\blacksquare}}
\else
	\newtheorem{lemma}{Lemma}
	
	\newtheorem{theorem}{Theorem}
	\newtheorem{proposition}{Proposition}

	\theoremstyle{definition}
	\newtheorem{definition}{Definition}

\fi

\definecolor{green}{rgb}{0.0, 0.5, 0.0}

\definecolor{xcolor-gray}{gray}{0.95}

\definecolor{best}{rgb}{0.019, 0.627, 0.011}

\definecolor{second_best}{rgb}{0.101, 0.392, 0.635}

\definecolor{catastrophic_red}{rgb}{0.9058,0.0274,0.0274}
\newcommand\catastrophicred[1]{{\color{catastrophic_red}#1}}
\definecolor{benign_green}{rgb}{0.1529, 0.5882, 0.2705}
\newcommand\benigngreen[1]{{\color{benign_green}#1}}

\definecolor{bg_lightblue}{rgb}{0.9529,0.9725,1}
\definecolor{darkolive}{rgb}{0.1607,0.1803,0.1176}
\tcbset{
  takeawaybox/.style={
    top=10pt,
    bottom=5pt,
    colback=bg_lightblue,
    colframe=darkolive,
    colbacktitle=darkolive,
    boxrule=1pt,
    enhanced,
    center,
    attach boxed title to top left={yshift=-0.1in,xshift=0.08in},
    boxed title style={boxrule=0pt,colframe=white}
  }
}
\newtcolorbox[auto counter]{takeawaybox}[2][]{
	takeawaybox,
	title=Takeaway~\thetcbcounter: #2,#1
}

\newcommand{\whitecref}[1]{%
	\begingroup
	\hypersetup{linkcolor=white}%
	\cref{#1}%
	\endgroup
}

\include{notation}

\usepackage{stackengine}
\usepackage{scalerel}
\newcommand\dangersign[1][2ex]{%
	\renewcommand\stacktype{L}%
	\scaleto{\stackon[0.85pt]{\catastrophicred{$\triangle$}}{\tiny \hspace{-0.565mm} \catastrophicred{!}}}{#1}%
}

\renewcommand{\bibsection}{}

\DeclareFontFamily{U}{mathx}{\hyphenchar\font45}
\DeclareFontShape{U}{mathx}{m}{n}{<-> mathx10}{}
\DeclareSymbolFont{mathx}{U}{mathx}{m}{n}
\DeclareMathAccent{\widebar}{0}{mathx}{"73}

\definecolor{darkspringgreen}{rgb}{0.09, 0.45, 0.27}
\usetikzlibrary{decorations.pathreplacing,calligraphy}
\usetikzlibrary{patterns}

\ifdefined\ENABLEENDNOTES
	\renewcommand{\theendnote}{\alph{endnote}} 
\else
	\renewcommand{\endnote}[1]{\null} 
\fi

\ifdefined\CVPR
	\usepackage[breaklinks=true,bookmarks=false]{hyperref}
	\ifdefined\CAMREADY
		\cvprfinalcopy 
	\fi
	
\fi

\newcommand{\eg}{{\it e.g.}}
\newcommand{\ie}{{\it i.e.}}
\newcommand{\cf}{{\it cf.}}

\begin{document}
	
	
\ifdefined\ARXIV
		\maketitle
\fi

\ifdefined\ARXIVTWO
\maketitle
\fi

\ifdefined\NEURIPS
	\title{Paper Title}
		\author{
			Author 1 \\
			Author 1 Institution \\	
			\texttt{author1@email} \\
			\And
			Author 1 \\
			Author 1 Institution \\	
			\texttt{author1@email} \\
		}
		\maketitle
\fi

\ifdefined\CVPR
		\title{Paper Title}
		\author{
			Author 1 \\
			Author 1 Institution \\	
			\texttt{author1@email} \\
			\and
			Author 2 \\
			Author 2 Institution \\
			\texttt{author2@email} \\	
			\and
			Author 3 \\
			Author 3 Institution \\
			\texttt{author3@email} \\
		}
		\maketitle
	\fi
	\ifdefined\AISTATS
		\twocolumn[
		\aistatstitle{Paper Title}
		\ifdefined\CAMREADY
			\aistatsauthor{Author 1 \And Author 2 \And Author 3}
			\aistatsaddress{Author 1 Institution \And Author 2 Institution \And Author 3 Institution}
		\else
			\aistatsauthor{Anonymous Author 1 \And Anonymous Author 2 \And Anonymous Author 3}
			\aistatsaddress{Unknown Institution 1 \And Unknown Institution 2 \And Unknown Institution 3}
		\fi
		]	
	\fi
	\ifdefined\ICML
		\icmltitlerunning{Paper Title}
		\twocolumn[
		\icmltitle{Paper Title} 
		\icmlsetsymbol{equal}{*}
		\begin{icmlauthorlist}
			\icmlauthor{Author 1}{inst} 
			\icmlauthor{Author 2}{inst}
		\end{icmlauthorlist}
		\icmlaffiliation{inst}{Some Institute}
		\icmlcorrespondingauthor{Author 1}{author1@email}
		\icmlkeywords{}
		\vskip 0.3in
		]
		\printAffiliationsAndNotice{} 
	\fi
	\ifdefined\ICLR
        \title{\fontsize{16.6}{19}\selectfont Unintentional Unalignment: Likelihood \\ Displacement in Direct Preference Optimization}
            \newcommand{\aff}[1]{\textsuperscript{\normalfont #1}}
		\author{
		Noam Razin\aff{$\dagger$}, Sadhika Malladi\aff{$\dagger$}, Adithya Bhaskar\aff{$\dagger$}, Danqi Chen\aff{$\dagger$}, \\[0.5mm]  
		\textbf{~Sanjeev Arora\aff{$\dagger$}, Boris Hanin\aff{$\ddagger$}} \\[2mm]
       \textsuperscript{$\dagger$ }Princeton Language and Intelligence, Princeton University \;  \textsuperscript{$\ddagger$ }Princeton ORFE
		}
	
		\maketitle
	\fi
	\ifdefined\COLT
		\title{Paper Title}
		\coltauthor{
			\Name{Author 1} \Email{author1@email} \\
			\addr Author 1 Institution
			\And
			\Name{Author 2} \Email{author2@email} \\
			\addr Author 2 Institution
			\And
			\Name{Author 3} \Email{author3@email} \\
			\addr Author 3 Institution}
		\maketitle
	\fi

	\input{abstract}

	\ifdefined\COLT
		\medskip
	\fi

	
	\input{sec_intro}

	\input{sec_prelim}

    \input{sec_likelihood_displacement}

	\input{sec_likelihood_dynamics}

    \input{sec_ches}

	\input{sec_unintentional_unalignment}

 	\input{sec_related}

    \input{sec_conclusion}

	\ifdefined\NEURIPS
		\begin{ack}
			\input{ack}
		\end{ack}
	\else
		\newcommand{\ack}{\input{ack}}
	\fi
	\ifdefined\ARXIV
		\section*{Acknowledgements}
		\input{ack}
	\else
		\ifdefined\COLT
			\acks{\input{ack}}
		\else
			\ifdefined\CAMREADY
				\ifdefined\ICLR
					\newcommand*{\subsuback}{}
				\fi
				\ifdefined\NEURIPS
				\else
					\section*{Acknowledgements}
					\input{ack}
				\fi
			\fi
		\fi
	\fi

	\section*{References}
	{\small
		\ifdefined\ICML
			\bibliographystyle{icml2021}
		\else
			\ifdefined\COLM
				\bibliographystyle{colm2024_conference}
			\else
				\bibliographystyle{plainnat}
			\fi
		\fi
		\bibliography{refs}
	}

	\clearpage
	\appendix
	
	
	\ifdefined\ENABLEENDNOTES
		\theendnotes
	\fi
	


		\input{app_ln_ches}
		
		\input{app_pref_losses}

		\input{app_further_comparison}

        \input{app_formal_analysis}

        \input{app_additional_sft}

        \input{app_linear_no_displacement}
        

		\input{app_proofs}
 
		\input{app_experiments}

\end{document}

%% file: notation.tex
\newcommand{\vocab}{{\mathcal{V}}}
\newcommand{\rep}{{\hbf}}
\newcommand{\unembed}{\Wbf}
\newcommand{\piref}{\pi_{\mathrm{ref}}}
\newcommand{\thetainit}{\theta_{\mathrm{init}}}
\newcommand{\thetafin}{\theta_{\mathrm{fin}}}
\newcommand{\thetasft}{\theta_{\mathrm{S}}}
\newcommand{\thetaw}{\theta_{\mathrm{w}}}

\newcommand{\proj}{\mathrm{proj}}
\newcommand{\prob}{\pi}
\newcommand{\yw}{\ybf^{+}}
\newcommand{\yl}{\ybf^{-}}
\newcommand{\ywsmash}{\smash{\ybf^{+}}}
\newcommand{\ylsmash}{\smash{\ybf^{-}}}

\newcommand{\linkf}{\ell_{\xbf, \yw, \yl}}
\newcommand{\linkfother}{\ell_{\tilde{\xbf}, \tilde{\ybf}^{+}, \tilde{\ybf}^{-}}}
\newcommand{\smax}{\mathrm{softmax}}
\newcommand{\dataset}{\mathcal{D}}
\newcommand{\E}{\mathop{\mathbb{E}}}
\newcommand{\ches}{\mathrm{CHES}}
\newcommand{\chesln}{\overline{\mathrm{CHES}}}

\newcommand{\singledyn}{S}

\def\be{\begin{equation}}
	\def\ee{\end{equation}}
\def\beas{\begin{eqnarray*}}
	\def\eeas{\end{eqnarray*}}
\def\bea{\begin{eqnarray}}
	\def\eea{\end{eqnarray}}
\newcommand{\hbf}{{\mathbf h}}
\newcommand{\xbf}{{\mathbf x}}
\newcommand{\ybf}{{\mathbf y}}
\newcommand{\zbf}{{\mathbf z}}

\newcommand{\vbf}{{\mathbf v}}

\newcommand{\ebf}{{\mathbf e}}

\newcommand{\Wbf}{{\mathbf W}}

\newcommand{\D}{{\mathcal D}}

\newcommand{\V}{{\mathcal V}}

\renewcommand{\L}{\mathcal{L}}

\newcommand{\R}{{\mathbb R}}

\newcommand{\indc}[1]{\mathbbm{1}\left[#1\right]}

\newcommand{\inprod}[2]  {\left\langle{#1},{#2}\right\rangle}
\newcommand{\inprodbig}[2]  {\big \langle{#1},{#2} \big \rangle}
\newcommand{\inprodBig}[2]  {\Big \langle{#1},{#2} \Big \rangle}

\newcommand{\inprodnoflex}[2]{\langle{#1},{#2}\rangle}


%% file: abstract.tex
\begin{abstract}
Direct Preference Optimization (DPO) and its variants are increasingly used for aligning language models with human preferences. 
Although these methods are designed to teach a model to generate preferred responses more frequently relative to dispreferred responses, prior work has observed that the likelihood of preferred responses often decreases during training.
The current work sheds light on the causes and implications of this counterintuitive phenomenon, which we term \emph{likelihood displacement}.
We demonstrate that likelihood displacement can be \emph{catastrophic}, shifting probability mass from preferred responses to responses with an opposite meaning.
As a simple example, training a model to prefer \texttt{No} over \texttt{Never} can sharply increase the probability of \texttt{Yes}.
Moreover, when aligning the model to refuse unsafe prompts, we show that such displacement can \emph{unintentionally lead to unalignment}, by shifting probability mass from preferred refusal responses to harmful responses (\eg, reducing the refusal rate of Llama-3-8B-Instruct from 74.4\% to 33.4\%).
We theoretically characterize that likelihood displacement is driven by preferences that induce similar embeddings, as measured by a \emph{centered hidden embedding similarity (CHES)} score.
Empirically, the CHES score enables identifying which training samples contribute most to likelihood displacement in a given dataset.
Filtering out these samples effectively mitigated unintentional unalignment in our experiments.
More broadly, our results highlight the importance of curating data with sufficiently distinct preferences, for which we believe the CHES score may prove valuable.\footnote{
Our code is available at \scriptsize{\url{https://github.com/princeton-nlp/unintentional-unalignment}}.
}
\end{abstract}

%% file: sec_intro.tex
\section{Introduction}
\label{sec:intro}

To ensure that language models generate safe and helpful content, they are typically aligned based on pairwise preference data. 
One prominent alignment method, known as \emph{Reinforcement Learning from Human Feedback (RLHF)} \citep{ouyang2022training}, requires fitting a reward model to a dataset of human preferences, and then training the language model to maximize the reward via RL.
While often effective for improving the quality of generated responses \citep{bai2022training,achiam2023gpt,touvron2023llama}, the complexity and computational costs of RLHF motivated the rise of \emph{direct preference learning} methods such as DPO \citep{rafailov2023direct}.

Given a prompt $\xbf$, DPO and its variants (\eg, \citet{azar2024general,tang2024generalized,xu2024contrastive,meng2024simpo}) eschew the need for RL by directly teaching a model $\prob_\theta$ to increase the margin between the log probabilities of a preferred response $\yw$ and a dispreferred response $\yl$.
While intuitively these methods should increase the probability of $\yw$ while decreasing that of $\yl$,  several recent works observed that the probabilities of both $\yw$ and $\yl$ tend to \emph{decrease} over the course of training \citep{pal2024smaug,yuan2024advancing,rafailov2024r,tajwar2024preference,pang2024iterative,liu2024provably}.
We term this phenomenon \emph{likelihood displacement}~---~see \cref{fig:main}.

\begin{figure*}[t]
	\vspace{-4mm}
	\begin{center}
		\includegraphics[width=1\textwidth]{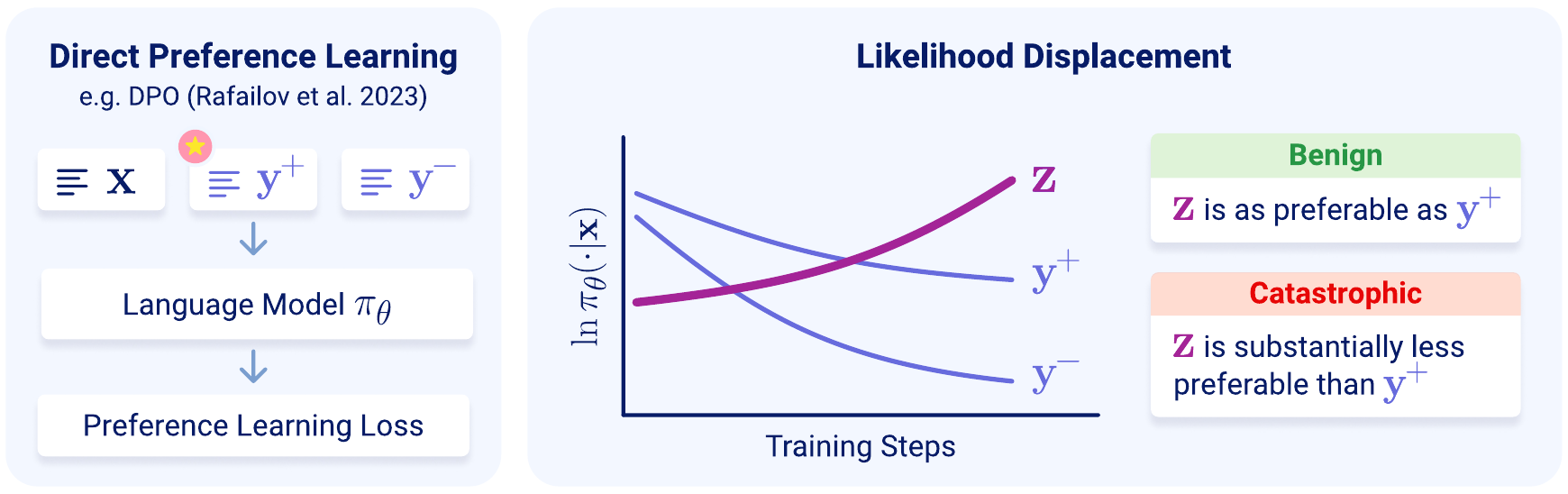}
	\end{center}
	\vspace{-1.5mm}
	\caption{
		\textbf{Illustration of likelihood displacement in direct preference learning.}
		For a prompt $\xbf$, direct preference learning aims to increase the probability that a model $\prob_\theta$ assigns to a preferred response~$\ywsmash$ relative to a dispreferred response~$\ylsmash$.
		\emph{Likelihood displacement} refers to the counterintuitive phenomenon where, while the gap between $\ln \prob_\theta ( \ywsmash | \xbf )$ and $\ln \prob_\theta (\ylsmash | \xbf)$ increases, they both decrease.
		If the responses increasing instead in probability (depicted by $\zbf$) are as preferable as $\ywsmash$ (\eg, $\zbf$ is similar in meaning to~$\ywsmash$), then the likelihood displacement is \emph{benign}.
		However, if the probability mass goes to responses that are substantially less preferable than $\ywsmash$ (\eg, $\zbf$ is opposite in meaning to~$\ywsmash$), then we say that it is \emph{catastrophic}.
	}
	\label{fig:main}
\end{figure*}

When the probability of $\yw$ decreases, the probability of other, possibly undesirable, responses must increase.
However, despite the prevalence of likelihood displacement, its causes and implications remain poorly understood.
The purpose of this work is to address these gaps.
Through theory and experiments, we characterize mechanisms driving likelihood displacement, demonstrate that it can lead to surprising failures in alignment, and provide preventative guidelines.
Our experiments cover models of different families and scales, including OLMo-1B \citep{groeneveld2024olmo}, Gemma-2B \citep{team2024gemma}, and Llama-3-8B \citep{dubey2024llama}.
The main contributions are listed below.

\begin{itemize}[leftmargin=6mm]

\item \textbf{Likelihood displacement can be catastrophic even in simple settings.}
We demonstrate that, even when training on just a single prompt whose preferences $\yw$ and $\yl$ consist of a single token each, likelihood displacement is pervasive (\cref{sec:catastrophic}).
Moreover, the tokens increasing in probability at the expense of $\yw$ can have a meaning opposite to it.
For example, training a model to prefer $\yw \! = \texttt{No}$ over $\yl \! = \texttt{Never}$ often sharply increases the probability of \texttt{Yes}.
This stands in stark contrast to prior work attributing likelihood displacement to different complexities in the preference learning pipeline \citep{tajwar2024preference,pal2024smaug,rafailov2024r}, and emphasizes the need to formally characterize its underlying causes.

\vspace{0.75mm}

\item \textbf{Theory: likelihood displacement is determined by the model's embedding geometry.}
We analyze the evolution of $\ln \prob_\theta (\yw | \xbf)$ during gradient-based training (\cref{sec:likelihood_dynamics}).
Our theory reveals that likelihood displacement is governed by the (static) token unembeddings and (contextual) hidden embeddings of $\yw$ and $\yl$.
In particular, it formalizes intuition by which the more similar $\yw$ and $\yl$ are the more $\ln \prob_\theta (\yw | \xbf)$ tends to decrease.

\vspace{0.75mm}

\item \textbf{Identifying sources of likelihood displacement.}
Based on our analysis, we derive a (model-aware) measure of similarity between preferences, called the \emph{centered hidden embedding similarity (CHES)} score (\cref{def:ches}).
We demonstrate that the CHES score accurately identifies which training samples contribute most to likelihood displacement in a given dataset (\eg, UltraFeedback \citep{cui2024ultrafeedback} and AlpacaFarm \citep{dubois2024alpacafarm}), whereas other similarity measures relying on hidden embeddings or token-level cues do not (\cref{sec:ches}).

\vspace{0.75mm}

\item \textbf{Unintentional unalignment due to likelihood displacement.}
To demonstrate the potential uses of the CHES score, we consider training a language model to refuse unsafe prompts via direct preference learning (\cref{sec:unalignment}).
We find that likelihood displacement can \emph{unintentionally unalign} the model, by causing probability mass to shift from preferred refusal responses to responses that comply with unsafe prompts! 
For example, the refusal rate of Llama-3-8B-Instruct drops from 74.4\% to 33.4\% over the SORRY-Bench dataset  \citep{xie2024sorry}.
We then show that filtering out samples with a high CHES score prevents such unintentional unalignment, and does so more effectively than adding a supervised finetuning term to the loss (\eg, as done in \citet{pal2024smaug,xu2024contrastive,pang2024iterative,liu2024provably}).

\end{itemize}

Overall, our results highlight the importance of curating data with sufficiently distinct preferences.
We believe the CHES score introduced by our theory may prove valuable in achieving this goal.

%% file: sec_prelim.tex
\section{Preliminaries}
\label{sec:prelim}

Let $\vocab$ be a vocabulary of tokens.
Modern language models consist of two parts: \emph{(i)} a neural network (\eg, Transformer~\citep{vaswani2017attention}) that intakes a sequence of tokens $\xbf \in \vocab^*$ and produces a \emph{hidden embedding} $\rep_\xbf \in \R^d$; and \emph{(ii)} a \emph{token unembedding matrix} $\Wbf \in \R^{\abs{\vocab} \times d}$ that converts the hidden embedding into logits.
The logits are then passed through a softmax to compute a distribution over tokens that can follow $\xbf$.
For assigning probabilities to sequences $\ybf \in \vocab^*$, a language model~$\prob_\theta$ operates autoregressively, \ie:
\be
\prob_\theta (\ybf | \xbf) = \prod\nolimits_{k = 1}^{\abs{\ybf}} \prob_\theta (\ybf_k | \xbf, \ybf_{< k }) = \prod\nolimits_{k = 1}^{ 
\abs{\ybf} } \smax \brk*{ \unembed \rep_{\xbf, \ybf_{< k} }  }_{\ybf_k}
\text{\,,}
\label{eq:lm_prob}
\ee
where $\theta$ stands for the model's parameters (\ie~the parameters of the neural network and the unembedding matrix $\unembed$), and $\ybf_{< k}$ denotes the first $k - 1$ tokens of $\ybf$.

\subsection{Direct Preference Learning}
\label{sec:prelim:direct_preference_learning}

\textbf{Preference data.} We consider the widely adopted direct preference learning pipeline, which relies on pairwise comparisons (\cf~\citet{rafailov2023direct}).
Specifically, we assume access to a preference dataset $\dataset$ containing samples $(\xbf, \yw, \yl)$, where $\xbf$ is a prompt, $\yw$ is a preferred response to $\xbf$, and $\yl$ is a dispreferred response to $\xbf$.
The preferred and dispreferred responses can be obtained by generating two candidate responses from the model (\ie~on-policy), and labeling them via human or AI raters (\cf~\citet{ouyang2022training,bai2022constitutional}).
Alternatively, they can be taken from some static dataset (\ie~off-policy).
Our analysis and experiments capture both cases.

\textbf{Supervised finetuning (SFT).}
Preference learning typically includes an initial SFT phase, in which the model is finetuned via the standard cross-entropy loss.
The sequences used for SFT can either be independent of the preference dataset $\dataset$ \citep{touvron2023llama} or consist of prompts and preferred responses from $\dataset$, \ie~of $\brk[c]{ (\xbf, \yw) : (\xbf, \yw, \yl) \in \dataset }$ \citep{rafailov2023direct}.

\textbf{Preference learning loss.}
Aligning language models based on pairwise preferences is usually done by minimizing a loss of the following form:
\be
    \L (\theta) := \E\nolimits_{(\xbf, \yw, \yl)\sim\dataset} \brk[s]*{ \linkf \brk2{ \ln \prob_\theta(\yw|\xbf) - \ln \prob_\theta(\yl|\xbf)} }
    \text{\,,}
    \label{eq:pref_loss}
\ee
where $\linkf : \R \to \R_{\geq 0}$ is convex and differentiable, for every $(\xbf, \yw, \yl) \in \dataset$.
Denote by $\thetainit$ the parameters of the model prior to minimizing the loss $\L$.
To guarantee that minimizing~$\L$ entails increasing the difference between $\ln \prob_{\theta} (\yw|\xbf)$ and $\ln \prob_{\theta} (\yl|\xbf)$, as expected from a reasonable preference learning loss, we make the mild assumption that $\linkf$ is monotonically decreasing in a neighborhood of $\ln \prob_{\thetainit} (\yw | \xbf) - \ln \prob_{\thetainit} (\yl | \xbf)$.

The loss $\L$ generalizes many existing losses, including: DPO \citep{rafailov2023direct}, IPO \citep{azar2024general}, SLiC \citep{zhao2023calibrating}, REBEL \citep{gao2024rebel}, and GPO \citep{tang2024generalized}~---~see \cref{app:pref_losses} for details on the choice of $\linkf$ corresponding to each loss.\footnote{
For SLiC and GPO, the corresponding $\linkf$ is differentiable almost everywhere, as opposed to differentiable.
Our analysis applies to such losses up to minor adaptations excluding non-differentiable points.
}
Notably, the common dependence on a reference model is abstracted through $\linkf$.
Other loss variants apply different weightings to the log probabilities of preferred and dispreferred responses or incorporate an additional SFT regularization term (\eg, DPOP \citep{pal2024smaug}, CPO \citep{xu2024contrastive}, RPO \citep{liu2024provably}, BoNBoN \citep{gui2024bonbon}, and SimPO \citep{meng2024simpo}).
For conciseness, we defer an extension of our analysis for these variants to \cref{app:additional_sft}.

\subsection{Likelihood Displacement}
\label{sec:prelim:likelihood_displacement}

We define likelihood displacement as the phenomenon where, although the preference learning loss is steadily minimized, the log probabilities of preferred responses decrease.

\begin{definition}
	\label{def:likelihood_displacement}
	Let $\prob_{\thetainit}$ and $\prob_{\thetafin}$ denote a language model before and after training with a preference learning loss $\L$ over the dataset $\dataset$ (\cref{eq:pref_loss}), respectively, and suppose that the loss was successfully reduced, \ie~$\L (\thetafin) < \L (\thetainit)$.
	We say that \emph{likelihood displacement occurred} if:\footnote{
		Note that $\ln \prob_{\theta} (\yw | \xbf)$ can decrease even as the loss $\L$ is minimized, since minimizing $\L$ only requires increasing the gap between $\ln \prob_{\theta} (\yw | \xbf)$ and $\ln \prob_{\theta} (\yl | \xbf)$.
	}
	\[
	\frac{1}{\abs{\dataset}} \sum\nolimits_{(\xbf, \yw, \yl) \in \dataset} \ln \prob_{\thetafin} (\yw | \xbf) < \frac{1}{\abs{\dataset}} \sum\nolimits_{(\xbf, \yw, \yl) \in \dataset} \ln \prob_{\thetainit} (\yw | \xbf)
	\text{\,;}
	\]
	and that \emph{likelihood displacement occurred for} $(\xbf, \yw, \yl) \in \dataset$ if $\ln \prob_{\thetafin} (\yw | \xbf) < \ln \prob_{\thetainit} (\yw | \xbf)$.
\end{definition}

Likelihood displacement is not necessarily problematic.
For $(\xbf, \yw, \yl) \in \dataset$, we refer to it as \emph{benign} if the responses increasing in probability are as preferable as $\yw$ (\eg, they are similar in meaning to $\yw$).
However, as \cref{sec:catastrophic} demonstrates, the probability mass can go to responses that are substantially less preferable than $\yw$ (\eg, they are opposite in meaning to $\yw$), in which case we say it is \emph{catastrophic}.

%% file: sec_likelihood_displacement.tex
\section{Catastrophic Likelihood Displacement in Simple Settings}
\label{sec:catastrophic}

Despite the prevalence of likelihood displacement \citep{pal2024smaug,yuan2024advancing,pang2024iterative,rafailov2024scaling,liu2024provably}, there is limited understanding as to why it occurs and where the probability mass goes.
Prior work attributed this phenomenon to limitations in model capacity \citep{tajwar2024preference}, the presence of multiple training samples or output tokens \citep{tajwar2024preference,pal2024smaug}, and the initial SFT phase \citep{rafailov2024r}.
In contrast, we demonstrate that likelihood displacement can occur and be catastrophic independently of these factors, even when training over just a single prompt whose responses contain a single token each.
The potential adverse effects of such displacement raise the need to formally characterize its underlying~causes.

\textbf{Setting.}
The experiments are based on the Persona dataset \citep{perez2022discovering}, in which every prompt contains a statement, and the model needs to respond whether it agrees with the statement using a single token.
We assign to each prompt a pair of preferred and dispreferred tokens $(\yw, \yl)$ from a predetermined set containing, \eg, \texttt{Yes}, \texttt{Sure}, \texttt{No}, and \texttt{Never}.
Then, for the OLMo-1B, Gemma-2B, and Llama-3-8B models, we perform one epoch of SFT using the preferred tokens as labels, in line with common practices, and train each model via DPO on a single randomly selected prompt.
See \cref{app:experiments:details:catastrophic} for additional details.

\begin{table*}[t]
	\vspace{-4mm}
	\begin{center}
		\small
		\begin{tabular}{lccccc}
			\toprule
			& & & &  \multicolumn{2}{c}{\textbf{Tokens Increasing Most in Probability}} \\
			\cmidrule(lr){5-6}
			\textbf{Model} & $\yw$ & $\yl$ & $\prob_\theta (\yw | \xbf)$ \textbf{Decrease} & \textbf{\benigngreen{Benign}} & \textbf{\catastrophicred{Catastrophic}} \\
			\midrule\\[-0.95em]
			\multirow{2}{*}{OLMo-1B} & Yes & No &  $0.69$ \, ($0.96 \to 0.27$) & \_Yes, \_yes & --- \\[0.15em]
			& No & Never & $0.84$ \, ($0.85 \to 0.01$) & \_No & Yes, \_Yes, \_yes \\[0.05em]
			\midrule\\[-0.95em]
			\multirow{2}{*}{Gemma-2B} & Yes & No & $0.22$ \, ($0.99 \to 0.77$) & \_Yes, \_yes & --- \\[0.15em]
			& No & Never & $0.21$ \, ($0.65 \to 0.44$) & no, \_No & yes, Yeah \\[0.05em]
			\midrule\\[-0.95em]
			\multirow{2}{*}{Llama-3-8B} & Yes & No & $0.96$ \, ($0.99 \to 0.03$) & yes, \_yes, \_Yes & --- \\[0.15em]
			& Sure & Yes & $0.59$ \, ($0.98 \to 0.39$)  & sure, \_Sure & Maybe, No, Never\\[0.05em]             
			\bottomrule
		\end{tabular}
	\end{center}
        \vspace{-1mm}
	\caption{
	\textbf{Likelihood displacement can be catastrophic, even  when training on a single prompt with single token responses.}
        Each model was trained via DPO on a randomly chosen prompt from the Persona dataset \citep{perez2022discovering}, using different pairs of preferred and dispreferred tokens $(\yw, \yl)$ (as detailed in \cref{sec:catastrophic}).
        We report the largest decrease in the preferred token probability $\prob_\theta (\yw | \xbf)$ during training for representative $(\yw, \yl)$ pairs, averaged across ten runs differing in the chosen prompt.
        On the right are notable tokens whose probabilities increase at the expense of $\yw$, categorized into benign or catastrophic according to whether they have a meaning similar to or distinct from $\yw$, respectively (a preceding “\_'' stands for a whitespace; see \cref{app:experiments:further:catastrophic} for the full list and extents of increase).
        Remarkably, when $\yw$ and $\yl$ are similar in meaning, the \textbf{tokens increasing most in probability are often opposite in meaning to $\yw$}.
	}
	\label{table:persona_single_example_post_sft}
\end{table*}

\textbf{Likelihood displacement is pervasive and can be catastrophic.}
\cref{table:persona_single_example_post_sft} reports the decrease in preferred token probability, and notable tokens whose probabilities increase at the expense of~$\yw$.
The probability of $\yw$ dropped by at least $0.21$ and up to $0.96$ absolute value in all runs.
Remarkably, when $\yw$ and $\yl$ are similar in meaning, the probability mass often shifts to tokens with meanings opposite to $\yw$.
\cref{app:experiments:further:catastrophic} reports similar findings for experiments using: \emph{(i)} base models that did not undergo an initial SFT phase (\cref{table:persona_single_example_base}); or \emph{(ii)} IPO instead of DPO (\cref{table:persona_single_example_post_sft_ipo}).

%% file: sec_likelihood_dynamics.tex
\section{Theoretical Analysis of Likelihood Displacement}
\label{sec:likelihood_dynamics}

To uncover what causes likelihood displacement when minimizing a preference learning loss, we characterize how the log probabilities of responses evolve during gradient-based training. 
For a preference sample $(\xbf, \yw, \yl) \in \dataset$, we identify the factors pushing $\ln \prob_{\theta} (\yw | \xbf)$ downwards and those determining which responses increase most in log probability instead.
\cref{sec:likelihood_dynamics:approach} lays out the technical approach, after which \cref{sec:likelihood_dynamics:overview} gives an overview of the main results.
The full analysis is deferred to \cref{app:formal_analysis}.
For the convenience of the reader, we provide the main takeaways below.

 \medskip 

 \begin{takeawaybox}[label=takeaway:single_token]{Role of the Token Unembedding Geometry (\whitecref{sec:likelihood_dynamics:overview:single_token})
 	}
 	Even when training over just a single prompt whose responses $\yw$ and $\yl$ contain a single token, likelihood displacement can occur due to the token unembedding geometry.
 	The underlying causes are: \emph{(i)} an alignment between the preferred and dispreferred token unembeddings, measured as $\langle \unembed_{\yw}, \unembed_{\yl} \rangle$; and \emph{(ii)} tokens whose unembeddings align with $\unembed_{\yw} - \unembed_{\yl}$, which increase in log probability at the expense of \smash{$\yw$}.
 	Since $\unembed_{\yw} - \unembed_{\yl}$ can have a large component orthogonal to $\unembed_{\yw}$ (introduced by $\unembed_{\yl}$), and unembeddings often linearly encode semantics, this provides an explanation for why probability mass can go to tokens unrelated or opposite in meaning to $\yw$ (as observed empirically in \cref{sec:catastrophic}).
 \end{takeawaybox}

 \medskip

 \begin{takeawaybox}[label=takeaway:multiple_tokens]{Role of the Hidden Embedding Geometry (\whitecref{sec:likelihood_dynamics:overview:multiple_tokens})}
 	Besides the impact of the token unembedding geometry (Takeaway~\ref{takeaway:single_token}), likelihood displacement occurs when the preferred and dispreferred responses are similar according to the following measure, which is based on their hidden embeddings.
	
 	\begin{definition}
 		\label{def:ches}
		 For a preference sample $(\xbf, \yw, \yl)\in\dataset$, we define the \emph{centered hidden embedding similarity (CHES)} score of $\yw$ and $\yl$ with respect to a model $\prob_{\theta}$ by:
		 \[
		 \ches_\xbf (\yw, \yl) := \inprodBig{ \!\! \underbrace{ \sum\nolimits_{k = 1}^{\abs{\yw}}  \rep_{\xbf, \yw_{< k} } }_{ \text{ $\yw$ hidden embeddings } } }{ \underbrace{ \sum\nolimits_{k' = 1}^{\abs{\yl}} \rep_{\xbf, \yl_{< k'} } }_{ \text{$\yl$ hidden embeddings} } }  - \norm2{ \sum\nolimits_{k = 1}^{\abs{\yw}} \rep_{\xbf, \yw_{< k} }  }^2
		 \text{\,,}
		 \]
		 where $\rep_{\xbf, \zbf_{< k}}$ denotes the hidden embedding that the model produces given $\xbf$ and the first $k - 1$ tokens of $\zbf \in \V^*$.
		 A higher CHES score stands for more similar preferences.
		 We omit the dependence of CHES on $\pi_\theta$ in our notation as it will be clear from context.
   	\end{definition}
 \end{takeawaybox}

\medskip

\textbf{Losses with SFT regularization.}
\cref{app:additional_sft} extends our analysis to losses incorporating an SFT regularization term.
In particular, it formalizes how this modification helps mitigate likelihood displacement, as proposed in \citet{pal2024smaug,liu2024provably,pang2024iterative,gui2024bonbon}.
We note, however, that our experiments in \cref{sec:unalignment} reveal a limitation of this approach for mitigating the adverse effects of likelihood displacement, compared to improving the data curation pipeline.

\subsection{Technical Approach}
\label{sec:likelihood_dynamics:approach}

Given a prompt $\xbf$, the probability that the model $\prob_\theta$ assigns to a response $\zbf$ is determined by the hidden embeddings $\rep_\xbf, \rep_{\xbf, \zbf_{< 2}}, \ldots, \rep_{\xbf, \zbf_{< \abs{\zbf}}}$ and the token unembeddings~$\unembed$ (\cref{eq:lm_prob}).
Our analysis relies on tracking their evolution when minimizing the loss $\L$ (\cref{eq:pref_loss}).
To do so, we adopt the \emph{unconstrained features model} \citep{mixon2022neural}, which amounts to treating hidden embeddings as directly trainable parameters.
Formally, the trainable parameters are taken to be $\theta = \brk[c]{ \rep_\zbf : \zbf \in \vocab^* } \cup \brk[c]{ \unembed }$.
This simplification has proven useful for studying various deep learning phenomena, including neural collapse (\eg, \citet{zhu2021geometric,ji2022unconstrained,tirer2023perturbation}) and the benefits of language model pretraining for downstream tasks \citep{saunshi2021a}.
As verified in \cref{sec:ches,sec:unalignment}, it also allows extracting salient sources of likelihood displacement in practice.\footnote{
	In contrast to prior theoretical analyses of likelihood displacement, which consider simpler settings, such as linear models and cases where the preferred and dispreferrred responses differ only by a single token~\citep{pal2024smaug,fisch2024robust,song2024understanding,ren2024learning}.
}

Language model finetuning is typically done with small learning rates.
Accordingly, we analyze the training dynamics of (stochastic) gradient descent at the small learning rate limit, \ie~\emph{gradient~flow}:
\[
\frac{d}{dt} \theta (t) = - \nabla \L \brk*{ \theta (t) } \quad , ~t \geq 0
\text{\,,}
\]
where $\theta (t)$ denotes the parameters at time $t \geq 0$ of training.
Note that under gradient flow the loss is monotonically decreasing.\footnote{
Except for the trivial case where $\theta (0)$ is a critical point of $\L$, in which $\L (\theta (t)) = \L (\theta (0))$ for all $t \geq 0$.
} 
Thus, any reduction in the log probabilities of preferred responses constitutes likelihood displacement (\cf~\cref{def:likelihood_displacement}).

\subsection{Overview of the Main Results}
\label{sec:likelihood_dynamics:overview}

\subsubsection{Single Training Sample and Output Token}
\label{sec:likelihood_dynamics:overview:single_token}

It is instructive to first consider the case of training on a single sample $(\xbf, \yw, \yl)$, whose responses $\yw \in \V$ and $\yl \in \V$ contain a single token.
\cref{thm:gf_single_token_preferred_logprob_informal} characterizes how the token unembedding geometry determines when $\frac{d}{dt} \ln \prob_{\theta (t)} (\yw | \xbf)$ is negative, \ie~when likelihood displacement occurs.

\begin{theorem}[Informal version of \cref{thm:gf_single_token_preferred_logprob}]
\label{thm:gf_single_token_preferred_logprob_informal}
Suppose that the dataset $\D$ contains a single sample $(\xbf, \yw, \yl)$, with $\yw \in \vocab$ and $\yl \in \vocab$ each being a single token.
At any time $t \geq 0$ of training, $\frac{d}{dt} \ln \prob_{\theta (t)} (\yw | \xbf)$ is more negative the larger the following term is:
\[
\underbrace{ \inprod{ \unembed_{\yw} (t) }{ \unembed_{\yl} (t) } }_{ \text{preferences unembedding alignment} } \! + ~~ \sum\nolimits_{z \in \vocab \setminus \{ \yw, \yl \} } \prob_{\theta (t)} (z | \xbf) \cdot \!\!\!\!\! \underbrace{ \inprod{ \unembed_z (t) }{ \unembed_{\yw} (t) - \unembed_{\yl} (t) } }_{ \text{alignment of other token with $\unembed_{\yw} (t) - \unembed_{\yl} (t)$} }
\text{\,\,,}
\]
where $\unembed_z (t)$ denotes the token unembedding of $z \in \V$ at time~$t$.
\end{theorem}

Two terms govern the extent of likelihood displacement in the case of single token responses.
First, $\langle \unembed_{\yw} (t) , \unembed_{\yl} (t) \rangle$ formalizes the intuition that likelihood displacement occurs when the preferred and dispreferred responses are similar.
A higher inner product in unembedding space translates to a more substantial (instantaneous) decrease in $\ln \prob_{\theta (t)} (\yw | \xbf)$.
Second, is a term which measures the alignment of other token unembeddings with $\unembed_{\yw} (t) - \unembed_{\yl} (t)$, where tokens deemed more likely by the model have a larger weight.
The alignment of token unembeddings with $\unembed_{\yw} (t) - \unembed_{\yl} (t)$ also determines which tokens increase most in log probability.

\begin{theorem}[Informal version of \cref{thm:gf_single_token_where_mass_goes}]
\label{thm:gf_single_token_where_mass_goes_informal}
Under the setting of \cref{thm:gf_single_token_preferred_logprob_informal}, for any time $t \geq 0$ of training and token $z \in \V \setminus \{ \yw, \yl \}$ it holds that $\frac{d}{dt} \ln \prob_{\theta (t)} (z | \xbf) \propto \inprod{ \unembed_z (t) }{ \unembed_{\yw} (t) - \unembed_{\yl} (t) }$, up to an additive term independent of $z$.
\end{theorem}

The direction $\unembed_{\yw} (t) - \unembed_{\yl} (t)$ can be decomposed into its projection onto $\unembed_{\yw} (t)$ and a component orthogonal to $\unembed_{\yw} (t)$, introduced by $\unembed_{\yl} (t)$.
Thus, tokens increasing in log probability can have unembeddings that mostly align with directions orthogonal to $\unembed_{\yw} (t)$, especially when the component orthogonal to $\unembed_{\yw} (t)$ of $\unembed_{\yw} (t) - \unembed_{\yl} (t)$ is relatively large (which we often find to be the case empirically; see \cref{table:persona_pref_vs_proj_direction_norm} in \cref{app:experiments:further:catastrophic}).
Given that token unembeddings are known to linearly encode semantics \citep{mikolov2013distributed,arora2016latent,park2024linear}, this provides an explanation for why the probability mass can shift to tokens that are unrelated or opposite in meaning to the preferred token, \ie~why likelihood displacement can be catastrophic even in simple settings (as observed in \cref{sec:catastrophic}).

\subsubsection{Responses with Multiple Tokens}
\label{sec:likelihood_dynamics:overview:multiple_tokens}

We now extend our analysis to the typical case where responses are sequences of tokens.
As shown below, the existence of multiple tokens in each response introduces a dependence on their (contextual) hidden embeddings.

\begin{theorem}[Informal version of \cref{thm:gf_multiple_tokens_preferred_logprob}]
\label{thm:gf_multiple_tokens_preferred_logprob_informal}
Suppose that the dataset $\dataset$ contains a single sample $(\xbf, \yw, \yl)$, with $\yw \in \V^*$ and $\yl \in \V^*$.
At any time $t \geq 0$ of training, in addition to the dependence on token unembeddings identified in \cref{thm:gf_single_token_preferred_logprob_informal}, $\frac{d}{dt} \ln \prob_{\theta (t)} (\yw | \xbf)$ is more negative the larger the following term is:
 \[
  \sum_{k = 1}^{\abs{\yw}} \sum_{k' = 1}^{\abs{\yl}} \alpha^{-}_{k, k'} (t) \cdot \underbrace{ \inprod{ \rep_{\xbf, \yw_{< k} } (t) } { \rep_{\xbf, \yl_{< k'} } (t) } }_{\text{preferred-dispreferred alignment}} - \sum_{k = 1}^{\abs{\yw}} \sum_{k' = 1}^{\abs{\yw}} \alpha^{+}_{k, k'} (t) \cdot \underbrace{ \inprod{ \rep_{\xbf, \yw_{< k} } (t) }{ \rep_{\xbf, \yw_{< k'} } (t) } }_{\text{preferred-preferred alignment}}
 \text{\,,}
 \]
where $\rep_{\zbf} (t)$ denotes the hidden embedding of $\zbf \in \V^*$ at time $t$, and $\alpha^{-}_{k, k'} (t), \alpha^{+}_{k, k'} (t) \in [-2, 2]$ are coefficients determined by the model's next-token distributions for prefixes of $\yw$ and $\yl$ (see \cref{thm:gf_multiple_tokens_preferred_logprob} in \cref{app:formal_analysis:multiple_tokens} for their definition).
\end{theorem}

\cref{thm:gf_multiple_tokens_preferred_logprob_informal} establishes that the inner products between hidden embeddings, of both the “preferred-dispreferred" and “preferred-preferred" types, affect likelihood displacement.
A larger inner product leads to an upwards or downwards push on \smash{$\ln \prob_{\theta (t)} (\yw | \xbf)$}, depending on the sign of the corresponding \smash{$\alpha^{-}_{k, k'} (t)$} or \smash{$\alpha^{+}_{k, k'} (t)$} coefficient.
Empirically, we find that these coefficients are mostly positive across models and datasets; \eg, the OLMo-1B, Gemma-2B, and Llama-3-8B models and the UltraFeedback and AlpacaFarm datasets (see \cref{app:experiments:further:coeff} for details).
By accordingly setting all coefficients in \cref{thm:gf_multiple_tokens_preferred_logprob_informal} to one, we derive the \emph{centered hidden embedding similarity (CHES)} score between preferred and dispreferred responses (\cref{def:ches}).
Our analysis indicates that a higher CHES score implies more severe likelihood displacement.
\cref{sec:ches} empirically verifies this relation, and demonstrates that the CHES score is significantly more predictive of likelihood displacement than other plausible similarity measures.

Our analysis also provides insight into which responses increase most in probability at the expense of $\yw$.
\cref{thm:gf_multiple_tokens_where_mass_goes} in \cref{app:formal_analysis:multiple_tokens} derives the dependence of \smash{$\frac{d}{dt} \ln \prob_{\theta (t)} (\zbf | \xbf)$}, for any response $\zbf \in \V^*$, on the alignment of its hidden embeddings with those of $\yw$ and $\yl$.
However, in general settings, it is difficult to qualitatively describe the types of responses increasing in probability, and whether they constitute benign or catastrophic likelihood displacement.
\cref{sec:unalignment} thus demonstrates the (harmful) implications of likelihood displacement in settings where responses can be easily categorized into benign or catastrophic.
We regard studying the question of where the probability mass goes in additional settings as a promising direction for future work.

\subsubsection{Multiple Training Samples}
\label{sec:likelihood_dynamics:overview:multiple_examples}

\cref{sec:likelihood_dynamics:overview:single_token,sec:likelihood_dynamics:overview:multiple_tokens} showed that likelihood displacement may occur regardless of the dataset size.
Nonetheless, increasing the number of training samples was empirically observed to exacerbate it \citep{tajwar2024preference}.
\cref{app:formal_analysis:multiple_examples} sheds light on this observation by characterizing, for any $(\xbf, \ywsmash, \ylsmash) \in \dataset$, when additional training samples lead to a larger decrease in \smash{$\ln \prob_{\theta (t)} (\yw | \xbf)$}.
This unsurprisingly occurs when $\ywsmash$ appears as the dispreferred response of other prompts, \ie~there are contradicting samples.
We further establish that additional training samples can contribute negatively to \smash{$\frac{d}{dt} \ln \prob_{\theta (t)} (\yw | \xbf)$} even when their preferences are distinct from those of $\xbf$.

%% file: sec_ches.tex
\section{Identifying Sources of Likelihood Displacement}
\label{sec:ches}

In \cref{sec:likelihood_dynamics} we derived the CHES score (\cref{def:ches}), which for a given model and preference sample $(\xbf, \yw, \yl)$, measures the similarity of $\yw$ and $\yl$ based on their hidden embeddings.
Our theory indicates that samples with a higher CHES score lead to more likelihood displacement.
Below, we affirm this prediction and show that the CHES score enables identifying which training samples in a dataset contribute most to likelihood displacement, whereas alternative similarity measures fail to do so.
The following \cref{sec:unalignment} then demonstrates that filtering out samples with a high CHES score can mitigate undesirable implications of likelihood displacement.

\textbf{Setting.}
We use the UltraFeedback and AlpacaFarm datasets and the OLMo-1B, Gemma-2B, and Llama-3-8B models.
For every preference dataset and model, we compute the CHES scores of all samples.
This requires performing a single forward pass over the dataset.
Then, for each of the 0th, 25th, 50th, 75th, and 100th score percentiles, we take a subset of 512 samples centered around it.\footnote{
The 0th and 100th percentile subsets include the 512 samples with lowest and highest scores, respectively.
}
Lastly, we train the model via DPO on each subset separately, and track the change in log probability for preferred responses in the subset~---~the more the log probabilities decrease, the more severe the likelihood displacement is.
See \cref{app:experiments:details:hidden_embedding} for further details.

\textbf{Baselines.} 
Preferences with low (normalized) edit distance were suggested in \citet{pal2024smaug} as a cause for likelihood displacement.
Thus, we repeat the process outlined above while ranking the similarity of preferences using the (normalized) edit distance, where a lower edit distance between $\yw$ and $\yl$ corresponds to a higher similarity.
To the best of our knowledge, no other property of a preference sample was linked with likelihood displacement in the literature.
So we additionally compare to a natural candidate: using the inner product between the last hidden embeddings of $\yw$ and $\yl$, \ie~$\inprodnoflex{\rep_{ \xbf, \yw }}{ \rep_{\xbf, \yl}}$, as the similarity score.

\begin{figure*}[t]
	\centering
	\vspace{-5mm}
	\hspace*{-2mm}
	\includegraphics[width=\linewidth]{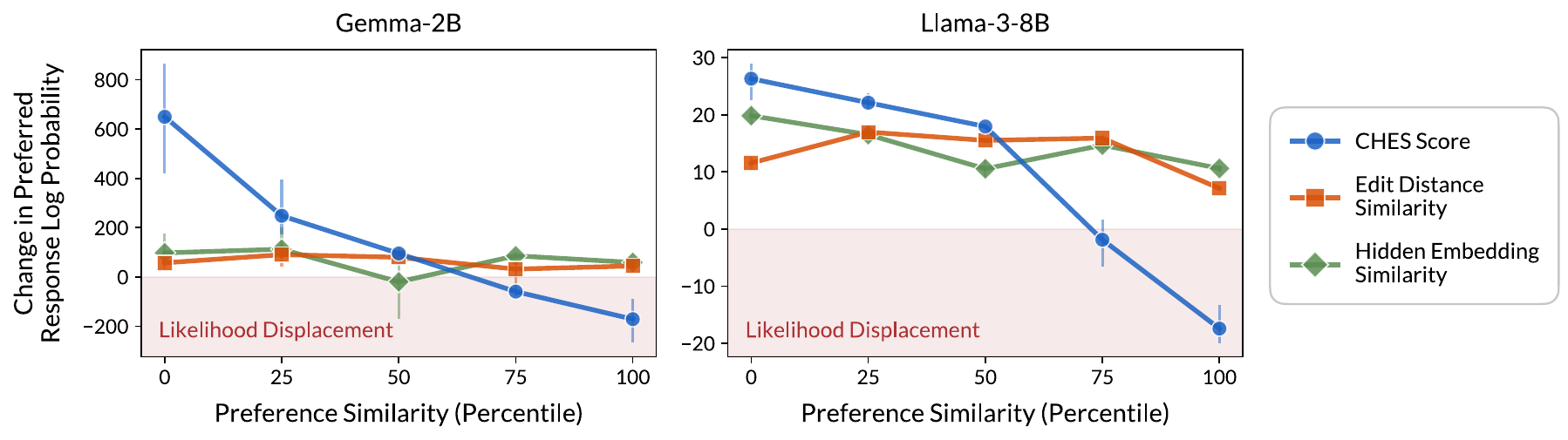}
	\vspace{-1.25mm}
	\caption{
		\textbf{CHES score (\cref{def:ches}) identifies which training samples contribute to likelihood displacement, whereas alternative similarity measures do not.}
		Each model was trained via DPO on subsets of 512 samples from the UltraFeedback dataset.
		The subsets are centered around different preference similarity percentiles, according to the following measures: \emph{(i)} the CHES score; \emph{(ii)} (normalized) edit distance, which was suggested in \citet{pal2024smaug} as indicative of likelihood displacement; and \emph{(iii)} the inner product between the last hidden embeddings of the preferred and dispreferred responses (see \cref{sec:ches} for further details).
		We report for each subset the change in mean preferred response log probability, averaged across three runs (error bars denote minimal and maximal values).
		The CHES score ranking perfectly matches with the degree of likelihood displacement~---~subsets with a higher score percentile induce a larger log probability decrease.
		On the other hand, the alternative measures are not indicative of likelihood displacement.
	}
	\label{fig:ultrafeedback_displacement_to_pref_sim_dpo}
\end{figure*}

\textbf{CHES score effectively identifies samples leading to likelihood displacement.}
For the UltraFeedback dataset, \cref{fig:ultrafeedback_displacement_to_pref_sim_dpo} shows the change in mean preferred response log probability against the similarity percentile of samples.
Across all models, the CHES score ranking matches perfectly the degree of likelihood displacement: the higher the CHES score percentile, the more preferred responses decrease in log probability.
Moreover, training on samples with high CHES scores leads to severe likelihood displacement, whereas training on samples with low CHES scores leads the preferred responses to increase in log probability.

\textbf{CHES score is more indicative of likelihood displacement than alternative measures.}
In contrast to the CHES score, the edit distance of preferences and the inner product between their last hidden embeddings are not indicative of likelihood displacement.
Moreover, these measures failed to identify samples leading to likelihood displacement: for almost all similarity percentiles, the mean preferred response log probability increased, with the few exceptional decreases being minor.

\textbf{Additional experiments.}
\cref{app:experiments:further:hidden_embedding} reports similar findings for experiments using: \emph{(i)} the AlpacaFarm dataset instead of UltraFeedback (\cref{fig:alpacafarm_displacement_to_pref_sim_dpo}); \emph{(ii)} IPO instead of DPO (\cref{fig:ultrafeedback_displacement_to_pref_sim_ipo}); or \emph{(iii)} the OLMo-1B model (\cref{fig:ultrafeedback_displacement_to_pref_sim_olmo}).

\textbf{Qualitative analysis.}
\cref{app:experiments:further:hidden_embedding} further includes representative samples with high and low CHES scores (\cref{tab:ultrafeedback_highest_examples,tab:ultrafeedback_least_examples}, respectively).
A noticeable trait is that, in samples with a high CHES score, the dispreferred response is often longer than the preferred response, whereas for samples with a low CHES score the trend is reversed (\ie~preferred responses are longer).
We find that this stems from a tendency of current models to produce, for different responses, hidden embeddings with a positive inner product (\eg, over 99\% of such inner products are positive for the Llama-3-8B model and UltraFeedback dataset).
As a result, for samples with longer dispreferred responses the CHES score comprises more positive terms than negative terms.

%% file: sec_unintentional_unalignment.tex
\section{Unintentional Unalignment in Direct Preference Learning}
\label{sec:unalignment}

Direct preference learning has been successfully applied for improving general instruction following and performance on downstream benchmarks (\eg, \citet{tunstall2023zephyr,ivison2023camels,jiang2024mixtral,dubey2024llama}).
This suggests that likelihood displacement may often be benign in such settings, and so does not require mitigation.
However, in this section, we reveal that it can undermine the efficacy of safety alignment.
When training a language model to refuse unsafe prompts, we find that likelihood displacement can \emph{unintentionally unalign} the model, by causing probability mass to shift from preferred refusal responses to harmful responses.
We then demonstrate that this undesirable outcome can be prevented by discarding samples with a high (length-normalized) CHES score (\cref{def:ches}), showcasing the potential of the CHES score for mitigating adverse effects of likelihood displacement.

\subsection{Setting}
\label{sec:unalignment:setting}

We train a language model to refuse unsafe prompts via the (on-policy) direct preference learning pipeline outlined in \citet{rafailov2023direct}, as specified below.
To account for the common scenario whereby one wishes to further align an existing (moderately) aligned model, we use the Gemma-2B-IT and Llama-3-8B-Instruct models.\footnote{
	The scenario of further aligning an existing moderately aligned model also arises in iterative direct preference learning pipelines \citep{yuan2024advancing,xiong2024iterative,pang2024iterative}.
}
Then, for each model separately, we create a preference dataset based on unsafe prompts from SORRY-Bench \citep{xie2024sorry}.
Specifically, for every prompt, we generate two candidate responses from the model and label them as refusals or non-refusals using the judge model from \citet{xie2024sorry}.
Refusals are deemed more preferable compared to non-refusals, and ties are broken by the PairRM reward model \citep{jiang2023llm}.\footnote{
Breaking ties randomly between responses of the same type led to similar results.
}
Lastly, we partition the datasets into training and test sets according to a 85\%/15\% split, and train the language models via DPO over their respective training sets.
For brevity, we defer to \cref{app:experiments:further,app:experiments:details} some experiments using IPO and implementation details, respectively.

\subsection{Catastrophic Likelihood Displacement Causes Unintentional Unalignment}
\label{sec:unalignment:unalignment}

Since the initial models are moderately aligned, we find that they often generate two refusal responses for a given prompt.
Specifically, for over 70\% of the prompts in the generated datasets, both the preferred and dispreferred responses are refusals.
This situation resembles the experiments of \cref{sec:catastrophic}, where training on similar preferences led to catastrophic likelihood displacement (\eg, when $\yw$ was \texttt{No} and $\yl$ was \texttt{Never}, the probability of \texttt{Yes} sharply increased).

Analogously, we observe that as the DPO loss is minimized, likelihood displacement causes probability mass to shift away from preferred refusal responses (\cref{tab:sorrybench_dpo_logprobs} in \cref{app:experiments:further:unalignment} reports the log probability decrease of preferred responses).
This leads to a significant drop in refusal rates.
Specifically, over the training sets, DPO makes the refusal rates of Gemma-2B-IT and Llama-3-8B-Instruct drop from 80.5\% to 54.8\% and 74.4\% to 33.4\%, respectively (similar drops occur over the test sets).
In other words, instead of further aligning the model, preference learning unintentionally unaligns it.
See \cref{app:experiments:further:unalignment} for examples of unsafe prompts from the training sets, for which initially the models generated two refusals, yet after DPO they comply with the prompts (\cref{tab:sorrybench_examples}).

We note that alignment usually involves a tradeoff between safety and helpfulness.
The drop in refusal rates is particularly striking since the models are trained with the sole purpose of refusing prompts, without any attempt to maintain their helpfulness.

\subsection{Filtering Data via CHES Score Mitigates Unintentional Unalignment}
\label{sec:unalignment:mitigating}

\cref{sec:ches} showed that samples with a high CHES score (\cref{def:ches}) contribute most to likelihood displacement.
Motivated by this, we explore whether filtering data via the CHES score can mitigate unintentional unalignment, and which types of samples it marks as problematic.

As discussed in \cref{sec:ches}, due to the embedding geometry of current models, CHES scores can correlate with the lengths of responses.
To avoid introducing a length bias when filtering data, we apply a length-normalized variant of CHES (see \cref{def:ln_ches} in \cref{app:ln_ches}).
For comparison, we also consider adding an SFT term to the DPO loss, as suggested in \citet{pal2024smaug,xu2024contrastive,pang2024iterative,liu2024provably}, and training over “gold" responses from SORRY-Bench, which were generated from a diverse set of base and safety aligned models and labeled by human raters.

\begin{figure*}[t]
	\centering
	\vspace{-4mm}
	\begin{minipage}[b]{0.53\linewidth}
		\hspace*{0.7mm}
		\includegraphics[width=0.92\linewidth]{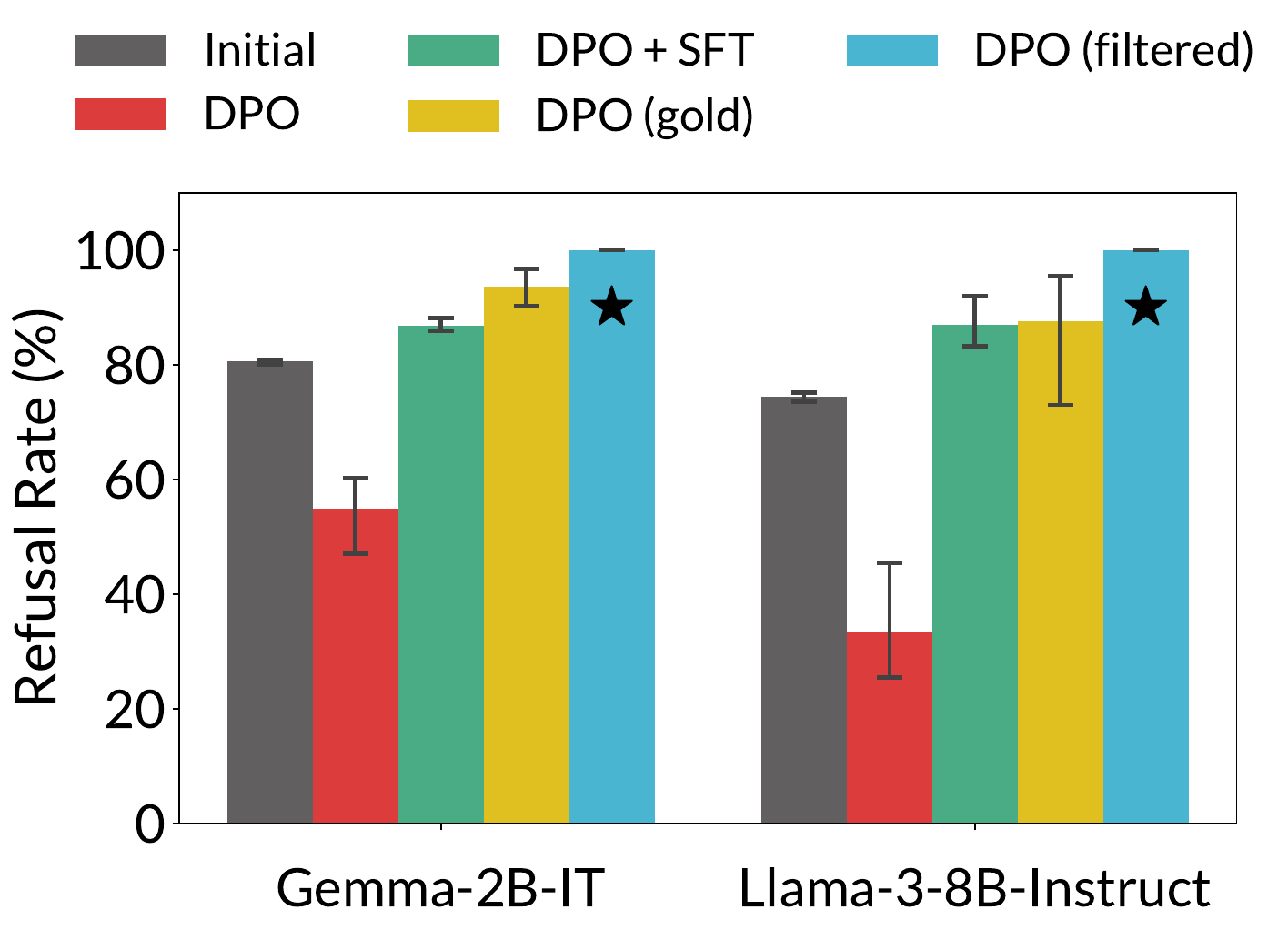}
		\vspace{-0.6mm}
		\caption{
			\textbf{Likelihood displacement can cause unintentional unalignment, which is mitigated by data filtering.}
			Training a model to refuse unsafe prompts from SORRY-Bench via DPO unintentionally leads to a substantial decrease in refusal rates due to likelihood displacement.
			Filtering out samples with a high length-normalized CHES score ($\star$) or using “gold" preference data, generated from a diverse set of models, successfully mitigates the problem, and goes beyond the improvement achieved when adding an SFT term to the DPO loss.
			Reported are the refusal rates over the training sets, averaged across three runs (error bars denote minimal and maximal values).
			Results over the test sets were similar.
			See \cref{sec:unalignment} for further details.
		}
		\label{fig:sorrybench_refusal_rate_dpo}
	\end{minipage}
	\hspace{0.02\linewidth}
	\begin{minipage}[b]{0.43\linewidth}
		\hspace{1mm}
		\includegraphics[width=0.87\linewidth]{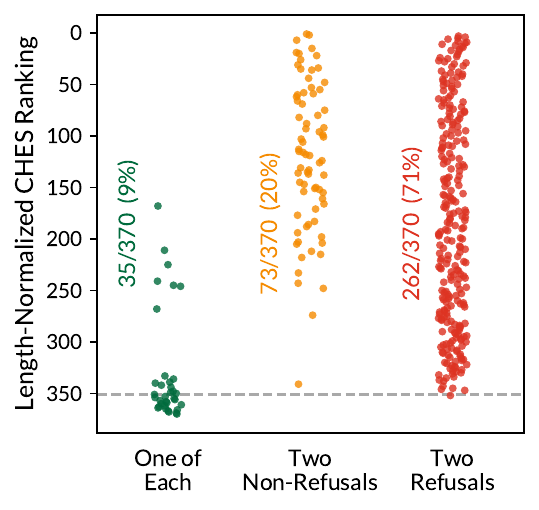}
		\vspace{-1.6mm}
		\caption{
			\textbf{Length-normalized CHES score identifies samples with two responses of the same type as responsible for likelihood displacement.}
			For Llama-3-8B-Instruct, we take the corresponding SORRY-Bench training preference dataset (see \cref{sec:unalignment:setting} for details on the dataset creation process), and plot the ranking of samples according to their length-normalized CHES scores.
			Gray line marks the 5\% samples included in the filtered dataset of \cref{fig:sorrybench_refusal_rate_dpo}.
			Agreeing with intuition, samples with two refusal or two non-refusal responses tend to have a higher score than samples with one of each.
		}	
		\label{fig:sorrybench_llama_category_ches}
	\end{minipage}
\end{figure*}

\textbf{Filtering data via CHES score mitigates unintentional unalignment.}
\cref{fig:sorrybench_refusal_rate_dpo} reports the refusal rates before and after training via DPO: \emph{(i)} on the original dataset, which as stated in \cref{sec:unalignment:unalignment} leads to a substantial drop in refusal rates; \emph{(ii)} with an additional SFT term on the original dataset; \emph{(iii)} on the gold dataset; and \emph{(iv)} on a filtered version of the original dataset that contains the 5\% samples with lowest length-normalized CHES scores.\footnote{
Keeping up to 15\% of the original samples led to analogous results.
Beyond that, as when training on the full dataset, likelihood displacement caused refusal rates to drop.
}
Filtering data via the CHES score successfully mitigates unintentional unalignment.
Moreover, while adding an SFT term to the loss also prevents the drop in refusal rates, data filtering boosts the refusal rates more substantially.
We further find that DPO on gold preferences does not suffer from likelihood displacement or unintentional unalignment (\ie~the preferred responses increase in log probability; see \cref{tab:sorrybench_dpo_logprobs}).
Overall, these results highlight the importance of curating data with sufficiently distinct preferences for effective preference learning.

\textbf{Which samples have a high CHES score?}
\cref{fig:sorrybench_llama_category_ches} reveals that the length-normalized CHES score ranking falls in line with intuition~---~samples with two refusal or two non-refusal responses tend to have a higher score than samples with one of each, and so are more likely to cause likelihood displacement.
To confirm that both samples with two refusal responses and samples with two non-refusals are responsible for the drop in refusal rates (shown in \cref{fig:sorrybench_refusal_rate_dpo}), we trained the Gemma-2B-IT and Llama-3-8B-Instruct models via DPO on each type of samples separately. 
In both cases, likelihood displacement occurred and the refusal rates dropped as when training on the full dataset.

%% file: sec_related.tex
\section{Related Work}
\label{sec:related}

\textbf{Preference learning for language model alignment.}
There are two main approaches for aligning language models based on preference data.
First, RLHF (or RLAIF)~\citep{ziegler2019fine,stiennon2020summarize,ouyang2022training,bai2022constitutional}, which requires fitting a reward model to a dataset of human (or AI) preferences, and then training the language model to maximize the reward.
While often effective for improving the quality of generated responses, RLHF is computationally expensive and can suffer from instabilities \citep{zheng2023secrets,ramamurthy2023reinforcement,razin2024vanishing}.
This has led to the rise of \emph{direct preference learning}, as popularized by DPO \citep{rafailov2023direct}.
Our analysis supports methods that maximize the log probability ratio of preferred and dispreferred responses (\cf~\cref{sec:prelim:direct_preference_learning}), including DPO and many of its variants (\eg, \citet{zhao2023calibrating,azar2024general,gao2024rebel,tang2024generalized,pal2024smaug,xu2024contrastive,liu2024provably,gui2024bonbon,meng2024simpo}).
Investigating whether other variants, \eg, those proposed in \citet{ethayarajh2024kto,hong2024reference,song2024preference,wu2024self}, suffer from likelihood displacement is a potential avenue for future work.

\textbf{Analyses of direct preference learning.}
Prior work mostly established sample complexity guarantees for DPO (or a variant of it) when the training data either obeys a certain stringent structure~\citep{im2024generalization}, provides sufficient coverage~\citep{liu2024provably,song2024understanding,cen2024value,huang2024correcting}, or is obtained in an online manner \citep{cen2024value,zhang2024self,xie2024exploratory}.
Additionally, \citet{im2024understanding,feng2024towards} studied the optimization rate of DPO.
More relevant to our work is \citet{chen2024preference}, which demonstrated that DPO can fail to correct how a model ranks preferred and dispreferred responses.
Although related, this phenomenon is distinct from likelihood displacement.
In particular, when likelihood displacement occurs the probability of preferred responses is often higher than the probability of dispreferred responses (as illustrated in \cref{fig:main} and was the case in the experiments of \cref{sec:catastrophic,sec:ches,sec:unalignment}).

\textbf{Likelihood displacement.}
The relation of our results to existing claims regarding likelihood displacement was discussed throughout the paper.
We provide in \cref{app:further_comparison} an extended account.

\textbf{Jailbreaking and unalignment.}
Aligned language models are vulnerable to jailbreaking through carefully designed adversarial prompts~\citep{xu2024comprehensive}.
However, even when one does not intend to unalign a given model, \citet{pelrine2023exploiting,qi2024fine,he2024safe,zhan2023removing,lyu2024keeping} showed that performing SFT over seemingly benign data can result in unalignment.
The experiments in \cref{sec:unalignment} provide a more extreme case of unintentional unalignment.
Specifically, although the models are trained with the sole purpose of refusing unsafe prompts, likelihood displacement causes the refusal rates to drop, instead of increase.

%% file: sec_conclusion.tex
\section{Conclusion}
\label{sec:conclusion}

While direct preference learning has been widely adopted, there is considerable uncertainty around how it affects the model (\cf~\citet{xu2024dpo,chen2024preference}).
Our theory and experiments shed light on the causes and implications of one counterintuitive phenomenon~---~\emph{likelihood displacement}.
We demonstrated that likelihood displacement can be catastrophic, shifting probability mass from preferred responses to responses with an opposite meaning, which can result in \emph{unintentional unalignment} when training a language model to refuse unsafe prompts.
Intuitively, these failures arise when the preferred and dispreferred responses are similar.
We formalized this intuition and derived the \emph{centered hidden embedding similarity (CHES)} score (\cref{def:ches}), which effectively identifies samples contributing to likelihood displacement in a given dataset.
As an example of its potential uses, we showed that filtering out samples with a high (length-normalized) CHES score can prevent unintentional unalignment.
More broadly, our work highlights the importance of curating data with sufficiently distinct preferences.
We believe the CHES score introduced by our theory may prove valuable in achieving this goal.

\subsection{Limitations and Future Work}
\label{sec:conclusion:limitations_and_future}

\textbf{Theoretical analysis.}
Our theory focuses on the instantaneous change of log probabilities, and abstracts away which neural network architecture is used for computing hidden embeddings.
Future work can extend it by studying the evolution of log probabilities throughout training and accounting for how the architecture choice influences likelihood displacement.

\textbf{Occurrences of catastrophic likelihood displacement.}
While our findings reveal that likelihood displacement can make well-intentioned training result in undesirable outcomes, we do not claim that this occurs universally. 
Indeed, direct preference learning methods have been successfully applied for aligning language models \citep{tunstall2023zephyr,ivison2023camels,jiang2024mixtral,dubey2024llama}.
Nonetheless, in light of the growing prominence of these methods, we believe it is crucial to identify additional settings in which likelihood displacement is catastrophic.

\textbf{Utility of the CHES score.}
We demonstrated the potential of the (length-normalized) CHES score for filtering out samples that cause likelihood displacement.
However, further investigation is necessary to assess its utility more broadly.
For example, exploring whether data filtering via CHES scores improves performance in general instruction following settings, or whether CHES scores can be useful in more complex data curation pipelines for selecting distinct preferences based on a pool of candidate responses, possibly generated from a diverse set of models.

%% file: ack.tex
We thank Eshbal Hezroni for aid in preparing illustrative figures, Angelica Chen, Tianyu Gao, and Mengzhou Xia for providing feedback on this paper, and Yi Ren for insightful discussions.
NR is supported in part by the Zuckerman STEM Leadership Program.
SM and SA acknowledge funding from NSF, ONR, Simons Foundation, and DARPA.
AB gratefully acknowledges the support of a Hisashi and Masae Kobayashi*67 Fellowship. DC is supported by the National Science Foundation (IIS-2211779) and a Sloan Research Fellowship.
BH is supported by a 2024 Sloan Fellowship in Mathematics, NSF CAREER grant DMS-2143754, and NSF grants DMS-1855684, DMS-2133806.

%% file: app_ln_ches.tex
\section{Length-Normalized CHES Score}
\label{app:ln_ches}

In \cref{sec:likelihood_dynamics} we derived the CHES score (\cref{def:ches}), which for a given model and preference sample $(\xbf, \yw, \yl)$, measures the similarity of $\yw$ and $\yl$ based on their hidden embeddings.
\cref{sec:ches} then demonstrated on standard preference learning datasets (UltraFeedback and AlpacaFarm) that samples with high CHES scores contribute most to likelihood displacement.
However, as discussed in \cref{sec:ches}, due to the embedding geometry of current models, CHES scores often correlate with the lengths of responses.
Thus, to avoid introducing a length bias when filtering data in \cref{sec:unalignment:mitigating}, we apply the following length-normalized variant of CHES.

\begin{definition}
	\label{def:ln_ches}
	For a preference sample $(\xbf, \yw, \yl)\in\dataset$, we define the \emph{length-normalized CHES} score of $\yw$ and $\yl$ with respect to a model $\prob_{\theta}$ by:
	\[
		\chesln_\xbf (\yw, \yl) := \frac{1}{ \abs{\yw} \abs{\yl} } \inprodBig{\!\!  \underbrace{ \sum\nolimits_{k = 1}^{\abs{\yw}}  \rep_{\xbf, \yw_{< k} } }_{ \text{ $\yw$ hidden embeddings } } }{ \underbrace{ \sum\nolimits_{k' = 1}^{\abs{\yl}} \rep_{\xbf, \yl_{< k'} } }_{ \text{$\yl$ hidden embeddings} } }  - \frac{1}{ \abs{\yw}^2 } \norm2{ \sum\nolimits_{k = 1}^{\abs{\yw}} \rep_{\xbf, \yw_{< k} }  }^2
	\text{\,,}
	\]
	where $\rep_{\xbf, \zbf_{< k}}$ denotes the hidden embedding that the model produces given $\xbf$ and the first $k - 1$ tokens of $\zbf \in \V^*$.
	We omit the dependence on $\pi_\theta$ in our notation as it will be clear from context.
\end{definition}

%% file: app_pref_losses.tex
\section{Common Instances of the Analyzed Preference Learning Loss}
\label{app:pref_losses}

Let $(\xbf, \yw, \yl) \in \dataset$ be a preference sample.
As discussed in \cref{sec:prelim:direct_preference_learning}, the preference learning loss $\L$ (\cref{eq:pref_loss}) considered in our analysis generalizes many existing losses, which are realized by different choices of $\linkf$.
The choice of $\linkf$ corresponding to each loss is given below.

\textbf{DPO~\citep{rafailov2023direct}.}
The DPO loss can be written as:
\[
    \linkf \brk*{ \ln \frac{ \prob_{\theta} (\yw | \xbf) }{ \prob_{\theta} (\yl | \xbf) } } := - \ln \sigma \brk*{ \beta \brk*{ \ln \frac{\ \prob_\theta(\yw|\xbf) }{ \prob_\theta(\yl|\xbf) } - \ln \frac{ \piref(\yw|\xbf) }{ \piref(\yl|\xbf) } } }
    \text{\,,}
\]
where $\piref$ is some reference model, $\beta > 0$ is a regularization hyperparameter, and $\sigma : \R \to [0, 1]$ denotes the sigmoid function.

\textbf{IPO~\citep{azar2024general}.}
The IPO loss can be written as:
\[
    \linkf \brk*{ \ln \frac{ \prob_{\theta} (\yw | \xbf) }{ \prob_{\theta} (\yl | \xbf) } } := \brk*{ \ln \frac{\ \prob_\theta(\yw|\xbf) }{ \prob_\theta(\yl|\xbf) } - \ln \frac{ \piref(\yw|\xbf) }{ \piref(\yl|\xbf) } - \frac{1}{2\tau} }^2
    \text{\,,}
\]
where $\piref$ is some reference model and $\tau > 0$ is a hyperparameter controlling the target log probability margin. 

\textbf{SLiC~\citep{zhao2023calibrating}.}
The SLiC loss can be written as: 
\[
    \linkf \brk*{ \ln \frac{ \prob_{\theta} (\yw | \xbf) }{ \prob_{\theta} (\yl | \xbf) } } := \max \brk[c]*{ 0, \delta - \ln \frac{\prob_\theta(\yw|\xbf)}{\prob_\theta(\yl|\xbf)} }
    \text{\,,}
\]
where $\delta > 0$ is a hyperparameter controlling the target log probability margin.
We note that our assumption on $\linkf$ being monotonically decreasing in a neighborhood of $\ln \prob_{\thetainit} (\yw | \xbf) - \ln \prob_{\thetainit} (\yl | \xbf)$ holds, except for the case where the loss for $(\xbf, \yw, \yl)$ is already zero at initialization (recall $\thetainit$ stands for the initial parameters of the model).

\textbf{REBEL~\citep{gao2024rebel}.}
The REBEL loss can be written as:
\[
\linkf \brk*{ \ln \frac{ \prob_{\theta} (\yw | \xbf) }{ \prob_{\theta} (\yl | \xbf) } } :=
\brk*{ \frac{1}{\eta} \brk*{  \ln \frac{\ \prob_\theta(\yw|\xbf) }{ \prob_\theta(\yl|\xbf) } - \ln \frac{ \piref(\yw|\xbf) }{ \piref(\yl|\xbf) } }  - r(\xbf, \yw) + r(\xbf, \yl) }^2
\text{\,,}
\]
where $\piref$ is some reference model, $\eta > 0$ is a regularization parameter, and $r$ is a reward model.

\textbf{GPO~\citep{tang2024generalized}.}
GPO describes a family of losses, which can be written as: 
\[
    \linkf \brk*{ \ln \frac{ \prob_{\theta} (\yw | \xbf) }{ \prob_{\theta} (\yl | \xbf) } } := f \brk*{ \beta \brk*{ \ln \frac{\ \prob_\theta(\yw|\xbf) }{ \prob_\theta(\yl|\xbf) } - \ln \frac{ \piref(\yw|\xbf) }{ \piref(\yl|\xbf) } } }
    \text{\,,}
\]
where $\piref$ is some reference model and $f : \R \to \R$ is convex and monotonically decreasing in a neighborhood of $\ln \prob_{\thetainit} (\yw | \xbf) - \ln \prob_{\thetainit} (\yl | \xbf)$ (recall $\thetainit$ stands for the initial parameters of the model).

%% file: app_further_comparison.tex
\section{Relation to Existing Claims on Likelihood Displacement}
\label{app:further_comparison}

Throughout the paper, we specified how our results relate to existing claims regarding likelihood displacement.
This appendix provides a concentrated and extended account.

\textbf{Similarity of preferences.} 
\citet{tajwar2024preference} and \citet{pal2024smaug} informally claimed that samples with similar preferences are responsible for likelihood displacement.
Our theoretical analysis (\cref{sec:likelihood_dynamics}) formalizes this intuition, by proving that similarities between the token unembeddings and hidden embeddings of preferred and dispreferred responses drive likelihood displacement.

\textbf{Dataset size and model capacity.}
\citet{tajwar2024preference} also attributed likelihood displacement to the presence of multiple training samples in a dataset or a limited model capacity.
\cref{sec:catastrophic} demonstrates that likelihood displacement can occur independently of these factors, even when training an 8B model on a single sample.
Nonetheless, as we characterize in \cref{sec:likelihood_dynamics:overview:multiple_examples}, having multiple training samples can contribute to the severity of likelihood displacement.

\textbf{Preferences with small edit distance.}
\citet{pal2024smaug} showed in controlled settings that preferences with a small edit distance can lead to likelihood displacement.
However, as the experiments in \cref{sec:ches} demonstrate, in more general settings edit distance is not indicative of likelihood displacement, in contrast to the CHES score (\cref{def:ches}), which measures similarity based on hidden embeddings.

\textbf{Initial SFT Phase.}
\citet{rafailov2024r} suggested that likelihood displacement occurs due to the initial SFT phase in the direct preference learning pipeline (see \cref{sec:prelim}).
Our experiments and theory (\cref{sec:catastrophic,sec:likelihood_dynamics}) refine this claim by showing that likelihood displacement can occur regardless of whether a model undergoes an initial SFT phase or not.

\textbf{Squeezing effect.}
\citet{ren2024learning} analyzed the impact of doing a gradient update to decrease the log probability of a dispreferred response $\yl$.
Focusing on a linear model with single token responses, \ie, a setting identical to that analyzed in \cref{sec:likelihood_dynamics:overview:single_token}, but with the hidden embedding of a prompt $\xbf$ being fixed during training, they identified a \emph{squeezing effect}, whereby the downwards push on $\yl$ predominantly shifts probability mass to tokens that already have a high probability.
\citet{ren2024learning} hypothesized that this squeezing effect is responsible for likelihood displacement.
However, as proven in \cref{app:linear_no_displacement}, in the linear setting that they consider likelihood displacement cannot occur~---~the preferred response never decreases in probability.
In particular, the squeezing effect due to the negative gradient on $\yl$, \ie~$- \nabla \ln \prob_\theta (\yl | x)$, is counteracted by the positive gradient on $\yw$, \ie~$\nabla \ln \prob_\theta (\yw | x)$.
This implies that the squeezing effect does not fully capture why likelihood displacement occurs, and emphasizes the need for taking into account how the hidden embeddings evolve during training, as done in \cref{sec:likelihood_dynamics}.

\textbf{Past sightings of catastrophic likelihood displacement.}
Prior work observed that DPO tends to degrade the performance on math and reasoning benchmarks \citep{pal2024smaug,yuan2024advancing,pang2024iterative,meng2024simpo}.
This can be considered as an instance of catastrophic likelihood displacement.
We note that, because in those settings only a few responses are correct, almost any likelihood displacement is catastrophic.
In contrast, our work demonstrates that likelihood displacement can be catastrophic even in settings where there are many acceptable responses, and reveals its adverse effects for safety alignment.

%% file: app_formal_analysis.tex
\section{Formal Analysis of Likelihood Displacement}
\label{app:formal_analysis}

This appendix delivers the formal analysis overviewed in \cref{sec:likelihood_dynamics:overview}.
\cref{app:formal_analysis:single_token,app:formal_analysis:multiple_tokens,app:formal_analysis:multiple_examples} cover the results discussed in \cref{sec:likelihood_dynamics:overview:single_token,sec:likelihood_dynamics:overview:multiple_tokens,sec:likelihood_dynamics:overview:multiple_examples}, respectively.
We refer the reader to \cref{sec:likelihood_dynamics:approach} for the technical setting of the analysis.

\textbf{Notation.}
For any time $t \geq 0$, we use $\unembed (t), \unembed_{z} (t)$, and $\rep_{\zbf} (t)$ to denote the token unembedding matrix, unembedding of a token $z \in \V$, and hidden embedding of $\zbf \in \V^*$, respectively.
We let $\zbf_k$ be the $k$th token in $\zbf$ and $\zbf_{< k}$ be the first $k - 1$ tokens in $\zbf$.
With slight abuse of notation, we shorthand $\linkf' (t) := \linkf' \brk{ \ln \prob_{\theta (t)} (\yw|\xbf) - \ln \prob_{\theta (t)} (\yl|\xbf)}$ for a preference sample $(\xbf, \yw, \yl) \in \dataset$, where $\linkf'$ stands for the derivative of $\linkf$.
Lastly, we denote by $\ebf_z \in \R^{\abs{\vocab}}$ the standard basis vector corresponding to $z \in \vocab$.

\subsection{Single Training Sample and Output Token (Overview in \cref{sec:likelihood_dynamics:overview:single_token})}
\label{app:formal_analysis:single_token}

We first consider the case of training on a single sample $(\xbf, \yw, \yl) \in \dataset$, whose responses $\yw \in \vocab$ and $\yl \in \vocab$ contain a single token.
\cref{thm:gf_single_token_preferred_logprob} characterizes the dependence of $\frac{d}{dt} \ln \prob_{\theta (t)} (\yw | \xbf)$ on the token unembedding geometry (proof deferred to \cref{app:proofs:gf_single_token_preferred_logprob}).

\begin{theorem}
\label{thm:gf_single_token_preferred_logprob}
Suppose that the dataset $\D$ contains a single sample $(\xbf, \yw, \yl)$, with $\yw \in \vocab$ and $\yl \in \vocab$ each being a single token.
At any time $t \geq 0$ of training:
\[
\begin{split}
    & \frac{d}{dt} \ln \prob_{\theta (t)} (\yw | \xbf) \\
    & \hspace{3mm} = - \linkf' (t) \bigg [ m (t) - \brk*{ 1 - \prob_{\theta (t)} (\yw | \xbf) + \prob_{\theta (t)} (\yl | \xbf) } \cdot \!\!\!\! \underbrace{ \inprod{ \unembed_{\yw} (t) }{ \unembed_{\yl} (t) } }_{ \text{preferences unembedding alignment} } \\
    & \hspace{37mm} - \sum\nolimits_{z \in \vocab \setminus \{ \yw, \yl \} } \prob_{\theta(t)} ( z | \xbf) \cdot \!\!\!\!\!\!\underbrace{ \inprod{ \unembed_z (t) }{ \unembed_{\yw} (t) - \unembed_{\yl} (t) } }_{ \text{alignment of other token with $\unembed_{\yw} (t) - \unembed_{\yl} (t)$ } } \! \bigg ]
    \text{\,,}
    \end{split}
\]
where $- \linkf' (t) > 0$ and $m(t)$ is a non-negative term given by:
\[
\begin{split}
    m(t) & := \brk*{ 1 - \prob_{\theta (t)} (\yw | \xbf) } \cdot \norm*{ \unembed_{\yw} (t) }^2 + \prob_{\theta (t)} (\yl | \xbf) \cdot \norm*{ \unembed_{\yl} (t) }^2 \\[0.3em]
    & \hspace{5mm} + \brk*{ 1 - \prob_{\theta (t)} (\yw | \xbf) + \prob_{\theta (t)} (\yl | \xbf) } \cdot \norm*{ \rep_{\xbf} (t) }^2 
\text{\,.}
\end{split}
\]  
\end{theorem}

Two terms in the derived form of $\frac{d}{dt} \ln \prob_{\theta (t)} (\yw | \xbf)$ can be negative, and so are responsible for likelihood displacement in the case of single toke responses.
First, the term $- \inprod{ \unembed_{\yw} (t) }{ \unembed_{\yl} (t) }$, which formalizes the intuition that likelihood displacement occurs when the preferred and dispreferred responses are similar.
A higher inner product translates to a more substantial (instantaneous) decrease in $\ln \prob_{\theta (t)} (\yw | \xbf)$.
Second, is a term measuring the alignment of other token unembeddings with $\unembed_{\yw} (t) - \unembed_{\yl} (t)$, where tokens deemed more likely by the model have a larger weight.
\cref{thm:gf_single_token_where_mass_goes} shows that the alignment of token unembeddings with $\unembed_{\yw} (t) - \unembed_{\yl} (t)$ also dictates which tokens increase most in log probability, \ie~where the probability mass goes (proof deferred to \cref{app:proofs:gf_single_token_where_mass_goes}).

\begin{theorem}
\label{thm:gf_single_token_where_mass_goes}
Under the setting of \cref{thm:gf_single_token_preferred_logprob}, for any time $t \geq 0$ and token $z \in \vocab \setminus \{ \yw, \yl \}$:
\[
\begin{split}
\frac{d}{dt} \ln \prob_{\theta (t)} (z | \xbf) = - \linkf (t) \cdot \brk[s]2{ \inprod{ \unembed_z (t) }{ \unembed_{\yw} (t) - \unembed_{\yl} (t) } + c (t) }
\text{\,,}
\end{split}
\]
where $- \linkf' (t) > 0$ and $c (t)$ is a term that does not depend on $z$, given by:
\[
    c(t) :=  \brk*{ \prob_{\theta (t)} (\yl | \xbf) - \prob_{\theta (t)} (\yw | \xbf) } \norm*{ \rep_{\xbf} (t) }^2 - \sum_{z' \in \vocab} \prob_{\theta (t)} (z' | \xbf) \inprod{ \unembed_{z'} (t) }{ \unembed_{\yw} (t) - \unembed_{\yl} (t) }
    \text{\,.}
\]
\end{theorem}

\subsection{Responses with Multiple Tokens (Overview in \cref{sec:likelihood_dynamics:overview:multiple_tokens})}
\label{app:formal_analysis:multiple_tokens}

Moving to the typical case, in which the responses $\yw \in \V^*$ and $\yl \in \V^*$ are sequences of tokens, assume for simplicity that $\yw_1 \neq \yl_1$.
Extending the results below to responses $\yw$ and $\yl$ that share a prefix is straightforward, by replacing terms that depend on $\yw_1$ and $\yl_1$ with analogous ones that depend on the initial tokens in which $\yw$ and $\yl$ differ.

In the case where each response consists of a single token, there are two terms that contribute to likelihood displacement (\cf~\cref{thm:gf_single_token_preferred_logprob}).
For any time $t \geq 0$ and $(\xbf, \yw, \yl) \in \dataset$, if one minimizes the preference learning loss with respect to only the initial tokens of $\yw$ and $\yl$, then these terms are given by:
\be
\begin{split}
\singledyn_{\yw_1, \yl_1} (t) := & -  \brk*{ 1 - \prob_{\theta (t)} (\yw_1 | \xbf) + \prob_{\theta (t)} (\yl_1 | \xbf) } \cdot \inprodbig{ \unembed_{\yw_1} (t) }{ \unembed_{\yl_1} (t) }  \\
& - \sum\nolimits_{z \in \vocab \setminus \{ \yw_1, \yl_1 \} } \prob_{\theta(t)} ( z | \xbf) \cdot \inprodbig{ \unembed_z (t) }{ \unembed_{\yw_1} (t) - \unembed_{\yl_1} (t) }
\text{\,.}
\end{split}
\label{eq:singledyn}
\ee
\cref{thm:gf_multiple_tokens_preferred_logprob} establishes that, in addition to the above initial token contribution, likelihood displacement depends on an alignment between the hidden embeddings of $\yw$ and $\yl$ (proof deferred to \cref{app:proofs:gf_multiple_tokens_preferred_logprob}).

\begin{theorem}
\label{thm:gf_multiple_tokens_preferred_logprob}
Suppose that the dataset $\dataset$ contains a single sample $(\xbf, \yw, \yl)$, with $\yw \in \V^*$ and $\yl \in \V^*$ satisfying $\yw_1 \neq \yl_1$.
At any time $t \geq 0$ of training:
 \[
 \begin{split}
 & \frac{d}{dt} \ln \prob_{\theta (t)} (\yw | \xbf) \\
 & \hspace{2mm} = - \linkf' (t) \bigg [ m(t) + \singledyn_{\yw_1, \yl_1} (t) \\
 & \hspace{5mm} -
  \sum_{k = 1}^{\abs{\yw}} \sum_{k' = 1}^{\abs{\yl}} \alpha^{-}_{k, k'} (t) \cdot \underbrace{ \inprod{ \rep_{\xbf, \yw_{< k} } (t) } { \rep_{\xbf, \yl_{< k'} } (t) } }_{\text{preferred-dispreferred alignment}} + \sum_{k = 1}^{\abs{\yw}} \sum_{k' = 1}^{\abs{\yw}} \alpha^{+}_{k, k'} (t) \cdot \underbrace{ \inprod{ \rep_{\xbf, \yw_{< k} } (t) }{ \rep_{\xbf, \yw_{< k'} } (t) } }_{\text{preferred-preferred alignment}} \bigg ] 
 \text{\,,}
 \end{split}
 \]
where $- \linkf (t) > 0$, the coefficients $\alpha^{-}_{k, k'} (t), \alpha^{+}_{k, k'} (t) \in [-2, 2]$ are given by:
\[
\begin{split}
& \alpha^-_{k, k'} :=  \inprod{ \ebf_{\yw_k} - \prob_{\theta (t)} (\cdot | \xbf, \yw_{< k}) }{ \ebf_{\yl_{k'}} -  \prob_{\theta (t)} (\cdot | \xbf, \yl_{< k'}) }   \text{\,,} \\[0.3em]
& \alpha^+_{k, k'} :=  \inprod{ \ebf_{\yw_k} - \prob_{\theta (t)} (\cdot | \xbf, \yw_{< k}) }{ \ebf_{\yw_{k'}} - \prob_{\theta (t)} (\cdot | \xbf, \yw_{< k'}) }   \text{\,,} \\
\end{split}
\]
with $\prob_{\theta (t)} (\cdot | \xbf, \zbf) \in \R^{\abs{\vocab}}$ denoting the model's next-token probability distribution, conditioned on $\xbf$ and $\zbf \in \V^*$, and $m(t)$ is the following non-negative term:
\[
\begin{split}
m(t) & := \brk1{ 1 - \prob_{\theta (t)} (\yw_1 | \xbf) } \cdot \norm2{ \unembed_{\yw_1} (t) }^2 + \prob_{\theta (t)} (\yl_1 | \xbf) \cdot \norm2{ \unembed_{\yl_1} (t) }^2 \\
& \hspace{5mm} + \sum\nolimits_{k = 2}^{\abs{\yw}} \norm2{ \unembed_{\yw_k } (t) - \sum\nolimits_{z \in \vocab} \prob_{\theta (t)} (z | \xbf, \yw_{< k}) \cdot \unembed_{z} (t) }^2
\text{\,.}
\end{split}
\]
\end{theorem}

The evolution of $\ln \prob_{\theta (t)} (\yw | \xbf)$ is governed by: \emph{(i)} the initial token unembedding geometry (analogous to the characterization in \cref{thm:gf_single_token_preferred_logprob}); and \emph{(ii)} inner products between hidden embeddings, of both the “preferred-dispreferred" and the “preferred-preferred" types.
As discussed in \cref{sec:likelihood_dynamics:overview:multiple_tokens}, whether a larger inner product results in an upwards or downwards push on $\ln \prob_{\theta (t)} (\yw | \xbf)$ depends on the sign of the corresponding $\alpha^-_{k, k'} (t)$ or $\alpha^+_{k, k'} (t)$ coefficient.
Since empirically these coefficients are mostly positive across models and datasets, \cref{thm:gf_multiple_tokens_preferred_logprob} indicates that a higher CHES score (\cref{def:ches}) implies more severe likelihood displacement. 

Regarding where the probability mass goes when likelihood displacement occurs, for any $\zbf \in \V^*$, \cref{thm:gf_multiple_tokens_where_mass_goes} below derives the dependence of $\frac{d}{dt} \ln \prob_{\theta (t)} (\zbf | \xbf)$ on the alignment of $\zbf$'s hidden embeddings with those of $\yw$ and $\yl$ (proof deferred to \cref{app:proofs:gf_multiple_tokens_where_mass_goes}).
Whether inner products between the hidden embeddings of $\zbf$ and those of $\yw$ (or $\yl$) contribute positively or negatively to $\frac{d}{dt} \ln \prob_{\theta (t)} (\zbf | \xbf)$, depends on the signs of coefficients that are determined by the model's next-token distributions.
For $\frac{d}{dt} \ln \prob_{\theta (t)} (\yw | \xbf)$, as mentioned above, the analogous coefficients are mostly positive.
However, it is difficult to assess their typical signs for general responses, \ie~for which responses the coefficients will tend to be positive and for which they will tend to be negative.
We thus regard further investigating which responses increase most in probability, and how that depends on the values that these coefficients take, as a promising direction for future work.

For simplicity, we assume that the initial token of $\zbf$ is not equal to the initial tokens of $\yw$ and $\yl$.
If $\zbf$ shares a prefix with $\yw$, then the characterization of \cref{thm:gf_multiple_tokens_where_mass_goes} holds up to additional terms that generally push $\ln \prob_{\theta (t)} (\zbf | \xbf)$ upwards.
Similarly, if $\zbf$ shares a prefix with $\yl$, then there will be additional terms that push $\ln \prob_{\theta (t)} (\zbf | \xbf)$ downwards.

\begin{theorem}
\label{thm:gf_multiple_tokens_where_mass_goes}
Under the setting of \cref{thm:gf_multiple_tokens_preferred_logprob}, let $\zbf \in \V^*$ be a response satisfying $\zbf_1 \notin \{\yw_1, \yl_1\}$.
At any time $t \geq 0$ of training:
 \[
\begin{split}
	& \frac{d}{dt} \ln \prob_{\theta (t)} (\zbf | \xbf) \\
	& \hspace{2mm} = - \linkf' (t) \bigg [ c(t) + \underbrace{   \inprod{ \unembed_{\zbf_1} (t) }{ \unembed_{\yw_1} (t) - \unembed_{\yl_1} (t) } }_{ \text{alignment of first token unembeddings} } \\
	& \hspace{5mm} -
	\sum_{k = 1}^{\abs{\zbf}} \sum_{k' = 1}^{\abs{\yl}} \beta^{-}_{k, k'} (t) \cdot \underbrace{ \inprod{ \rep_{\xbf, \zbf_{< k} } (t) } { \rep_{\xbf, \yl_{< k'} } (t) } }_{\text{$\zbf$-dispreferred alignment}} + \sum_{k = 1}^{\abs{\zbf}} \sum_{k' = 1}^{\abs{\yw}} \beta^{+}_{k, k'} (t) \cdot \underbrace{ \inprod{ \rep_{\xbf, \zbf_{< k} } (t) }{ \rep_{\xbf, \yw_{< k'} } (t) } }_{\text{$\zbf$-preferred alignment}} \bigg ] 
	\text{\,,}
\end{split}
\]
where $- \linkf (t) > 0$, the coefficients $\beta^{-}_{k, k'} (t), \beta^{+}_{k, k'} (t) \in [-2, 2]$ are given by:
\[
\begin{split}
	& \beta^-_{k, k'} :=  \inprod{ \ebf_{\zbf_k} - \prob_{\theta (t)} (\cdot | \xbf, \zbf_{< k}) }{ \ebf_{\yl_{k'}} -  \prob_{\theta (t)} (\cdot | \xbf, \yl_{< k'})  }   \text{\,,} \\[0.3em]
	& \beta^+_{k, k'} :=  \inprod{ \ebf_{\zbf_k} - \prob_{\theta (t)} (\cdot | \xbf, \zbf_{< k}) }{ \ebf_{\yw_{k'}} - \prob_{\theta (t)} (\cdot | \xbf, \yw_{< k'}) }   \text{\,,} \\
\end{split}
\]
and $c(t)$ is the following term that does not depend on $\zbf$:
\[
\begin{split}
	c(t) := - \sum\nolimits_{z \in \vocab} \prob_{\theta (t)} (z | \xbf) \inprod{ \unembed_{z} ( t ) }{  \unembed_{\yw_1} (t) - \unembed_{\yl_1} (t)  }
	\text{\,.}
\end{split}
\]
\end{theorem}

\subsection{Multiple Training Samples (Overview in \cref{sec:likelihood_dynamics:overview:multiple_examples})}
\label{app:formal_analysis:multiple_examples}

In this appendix, we consider the effect of having multiple training samples, focusing on the case where responses consist of a single token.
Namely, for a preference sample $(\xbf, \yw, \yl) \in \dataset$, \cref{thm:gf_multiple_examples_preferred_logprob} characterizes when additional training samples lead to a larger decrease in $\ln \prob_{\theta (t)} (\yw | \xbf)$ (proof deferred to \cref{app:proofs:gf_multiple_examples_preferred_logprob}).
For conciseness, we make the mild assumption that no prompt appears twice in $\dataset$, as is common in real-world preference datasets.

\begin{theorem}
\label{thm:gf_multiple_examples_preferred_logprob}
Suppose that all preferred and dispreferred responses in the dataset $\dataset$ consist of a single token each, and that no prompt appears twice (\ie~each prompt in $\dataset$ is associated with a single pair of preferred and dispreferred tokens).
For any time $t \geq 0$ of training and $(\xbf, \yw, \yl) \in \dataset$:
\[
\begin{split}
\frac{d}{dt} \ln \prob_{\theta (t)} (\yw | \xbf)  & = \frac{ -\linkf' (t) }{\abs{\dataset}} \cdot \!\!\!\!\!\!\!\!\!\! \underbrace{ \brk[s]2{ m(t) + \singledyn_{\yw, \yl} (t)  }}_{ \text{same sample contribution, as in \cref{thm:gf_single_token_preferred_logprob}} }   \\[0.2em]
& \hspace{5mm} + \sum\nolimits_{(\tilde{\xbf}, \tilde{\ybf}^+, \tilde{\ybf}^-) \in \dataset \setminus \{ (\xbf, \yw, \yl) \} } \underbrace{ \frac{ - \linkfother' (t) }{ \abs{\dataset} } \cdot \alpha_{\xbf, \tilde{\xbf}} (t) \cdot \inprod{ \rep_{\xbf} (t) }{ \rep_{\tilde{\xbf}} (t) } }_{ \text{ contribution due to $(\tilde{\xbf}, \tilde{\ybf}^+, \tilde{\ybf}^-)$ } }
\text{\,,}
\end{split}
\]
where $m(t)$ is the non-negative term defined in \cref{thm:gf_single_token_preferred_logprob}, $\singledyn_{\yw, \yl} (t)$ (defined in \cref{eq:singledyn}) encapsulates terms contributing to likelihood displacement when training only over $(\xbf, \yw, \yl)$, and the coefficient $\alpha_{\xbf, \tilde{\xbf}} (t) \in [-2, 2]$ is given by:
\[
\alpha_{\xbf, \tilde{\xbf}} (t) :=  \indc{ \yw = \tilde{\ybf}^+} - \indc{ \yw = \tilde{\ybf}^-} + \prob_{\theta (t)} (\tilde{\ybf}^- | \xbf) - \prob_{\theta (t)} (\tilde{\ybf}^+ | \xbf) 
\text{\,,}
\]
with $\indc{ \cdot }$ denoting the indicator function.
\end{theorem}

In the theorem above, $m(t) +  \singledyn_{\yw, \yl} (t)$ includes terms identical to those governing likelihood displacement when training only on $(\xbf, \yw, \yl)$ (\cf~\cref{thm:gf_single_token_preferred_logprob}).
The contribution of each additional sample $(\tilde{\xbf}, \tilde{\ybf}^+, \tilde{\ybf}^-) \in \dataset \setminus \{ (\xbf, \yw, \yl)\}$ to $\frac{d}{dt} \ln \prob_{\theta (t)} (\yw | \xbf)$ is captured by:
\[
 \frac{- \linkfother' (t)}{\abs{\dataset}} \cdot \alpha_{\xbf, \tilde{\xbf}} (t) \cdot \inprod{ \rep_{\xbf} (t) }{ \rep_{\tilde{\xbf}} (t) }
 \text{\,.}
\]
When does $(\tilde{\xbf}, \tilde{\ybf}^+, \tilde{\ybf}^-)$ contribute negatively to $\frac{d}{dt} \ln \prob_{\theta (t)} (\yw | \xbf)$?
First, we note that typically $ - \linkfother' (t)$ is positive.
Under the DPO loss this always holds (see \cref{lem:dpo_loss_deriv_sign}), while for other losses it holds at least initially since $\linkfother$ is monotonically decreasing in a neighborhood of $\ln \prob_{\theta(0)} (\tilde{\ybf}^+ | \tilde{\xbf}) - \ln \prob_{\theta(0)} (\tilde{\ybf}^- | \tilde{\xbf})$.
As for $\inprod{ \rep_{\xbf} (t) }{ \rep_{\tilde{\xbf}} (t) }$, we empirically find that the hidden embeddings of prompts in a given dataset almost always have positive inner products, across various models.
Specifically, for the OLMo-1B, Gemma-2B, and Llama-3-8B models, all such inner products over the “ends justify means" subset of the Persona dataset are positive.
This implies that $(\tilde{\xbf}, \tilde{\ybf}^+, \tilde{\ybf}^-)$ usually pushes $\ln \prob_{\theta (t)} (\yw | \xbf)$ downwards when $\alpha_{\xbf, \tilde{\xbf}} (t) < 0$.

Now, recall that:
\[
\alpha_{\xbf, \tilde{\xbf}} (t) =  \indc{ \yw = \tilde{\ybf}^+} - \indc{ \yw = \tilde{\ybf}^-} + \prob_{\theta (t)} (\tilde{\ybf}^- | \xbf) - \prob_{\theta (t)} (\tilde{\ybf}^+ | \xbf) 
\text{\,.}
\]
There are two cases in which $\alpha_{\xbf, \tilde{\xbf}} (t) < 0$:
\begin{enumerate}[leftmargin=7mm]
	\item (contradicting samples) when $\yw = \tilde{\ybf}^-$, \ie~the preferred token of $\xbf$ is the dispreferred token of $\tilde{\xbf}$; and
	
	\item (non-contradicting samples) when $\yw \notin \{ \tilde{\ybf}^+, \tilde{\ybf}^- \}$ and $\prob_{\theta (t)} (\tilde{\ybf}^- | \xbf) < \prob_{\theta (t)} (\tilde{\ybf}^+ | \xbf)$.
\end{enumerate}
While the first case is not surprising, the second shows that even when the preferences of $\xbf$ and $\tilde{\xbf}$ are distinct, the inclusion of $(\tilde{\xbf}, \tilde{\ybf}^+, \tilde{\ybf}^-)$ in the dataset can exacerbate likelihood displacement for $(\xbf, \yw, \yl)$.
Furthermore, as one might expect, \cref{thm:gf_multiple_examples_where_mass_goes} establishes that $(\tilde{\xbf}, \tilde{\ybf}^+, \tilde{\ybf}^-)$ encourages the probability mass conditioned on $\xbf$ to shift towards $\tilde{\ybf}^+$, given that $\inprod{ \rep_{\xbf} (t)}{ \rep_{\tilde{\xbf}} (t)} > 0$ (proof deferred to \cref{app:proofs:gf_multiple_examples_where_mass_goes}).

\begin{theorem}
	\label{thm:gf_multiple_examples_where_mass_goes}
	Under the setting of \cref{thm:gf_multiple_examples_preferred_logprob}, for any time $t \geq 0$ of training, $(\xbf, \yw, \yl) \in \dataset$, and token $z \in \vocab$:
	\[
	\begin{split}
		 \frac{d}{dt} \ln \prob_{\theta (t)} (z | \xbf) & = c(t) + \frac{ - \linkf (t) }{ \abs{\dataset} } \cdot \underbrace{ \inprod{ \unembed_z (t) }{ \unembed_{\yw} (t) - \unembed_{\yl} (t) } }_{ \text{same sample contribution, as in \cref{thm:gf_single_token_where_mass_goes}} } \\
		& \hspace{4mm} + \sum\nolimits_{(\tilde{\xbf}, \tilde{\ybf}^+, \tilde{\ybf}^-) \in \dataset } \underbrace{ \frac{ - \linkfother' (t) }{ \abs{\dataset} }  \brk*{ \indc{ z = \tilde{\ybf}^+ } - \indc{ z = \tilde{\ybf}^- } } \inprod{ \rep_{\xbf} (t) }{ \rep_{\tilde{\xbf}} (t) } }_{ \text{ contribution due to $(\tilde{\xbf}, \tilde{\ybf}^+, \tilde{\ybf}^-)$ } }
		\text{\,,}
	\end{split}
	\]
	where $\indc{\cdot}$ denotes the indicator function and $c (t)$ is a term that does not depend on $z$, given by:
	\[
	\begin{split}
	c(t) & :=  \frac{ \linkf (t) }{ \abs{\dataset} } \sum\nolimits_{z' \in \vocab} \prob_{\theta (t)} (z' | \xbf) \inprod{ \unembed_{z'} (t) }{ \unembed_{\yw} (t) - \unembed_{\yl} (t) } \\
	& \hspace{5mm} + \sum\nolimits_{(\tilde{\xbf}, \tilde{\ybf}^+, \tilde{\ybf}^-) \in \dataset } \frac{ - \linkfother' (t) }{ \abs{\dataset} }  \brk*{ \prob_{\theta (t)} (\tilde{\ybf}^- | \xbf) -  \prob_{\theta(t)} (\tilde{\ybf}^+ | \xbf) } \inprod{ \rep_{\xbf} (t) }{ \rep_{\tilde{\xbf}} (t) } 
	\text{\,.}
	\end{split}
	\]
\end{theorem}

%% file: app_additional_sft.tex
\section{Losses Including SFT Regularization or Different Weights for the Preferred and Dispreferred Responses}
\label{app:additional_sft}

Some preference learning losses include an SFT regularization term, multiplied by a coefficient $\lambda > 0$ (\eg, CPO \citep{xu2024contrastive}, RPO \citep{liu2024provably}, and BoNBoN \citep{gui2024bonbon}).
Namely, for a preference dataset $\dataset$, such losses have the following form:
\be
\L_{\mathrm{S}} (\theta) := \E\nolimits_{(\xbf, \yw, \yl)\sim\dataset} \brk[s]*{ \linkf \brk2{ \ln \prob_\theta(\yw|\xbf) - \ln \prob_\theta(\yl|\xbf)}  - \lambda \cdot \ln \prob_{\theta} (\yw | \xbf) }
\text{\,,}
\label{eq:pref_loss_additional_sft}
\ee
where $\linkf : \R \to \R_{\geq 0}$ is convex and differentiable, for $(\xbf, \yw, \yl) \in \dataset$ (\cf~\cref{eq:pref_loss}).
Other loss variants give different weights to the log probabilities of preferred and dispreferred responses within $\linkf$.
For example, SimPO~\citep{meng2024simpo} weighs them by the reciprocal of their lengths, and DPOP~\citep{pal2024smaug} adds an additional constant factor to the preferred response log probability.\footnote{
	The additional factor in the DPOP loss is only active when the preferred response log probability is below its initial value.
}
This type of losses can be expressed as:
\be
\L_{\mathrm{w}} (\theta) := \E\nolimits_{(\xbf, \yw, \yl)\sim\dataset} \brk[s]*{ \linkf \brk2{ \lambda^+_{\xbf, \yw, \yl} \cdot \ln \prob_\theta(\yw|\xbf) -  \lambda^-_{\xbf, \yw, \yl} \cdot \ln \prob_\theta(\yl|\xbf) } }
\text{\,,}
\label{eq:pref_loss_weights}
\ee
where $\lambda^+_{\xbf, \yw, \yl}, \lambda^-_{\xbf, \yw, \yl} > 0$ can depend on properties of $(\xbf, \yw, \yl) \in \dataset$.
As mentioned in \cref{sec:prelim:direct_preference_learning}, we assume that $\linkf$ is monotonically decreasing around the initialization (otherwise it does not encourage increasing the log probability ratio of preferred and dispreferred responses).
This mild assumption is upheld by all aforementioned losses.

The following \cref{app:additional_sft:analysis} extends our analysis from \cref{sec:likelihood_dynamics:overview:single_token,sec:likelihood_dynamics:overview:multiple_tokens} to the losses in \cref{eq:pref_loss_additional_sft,eq:pref_loss_weights}.
In particular, we formalize how adding an SFT regularization term, or assigning the preferred response a weight larger than that of the dispreferred response, can help mitigate likelihood displacement.
Indeed, such modifications to the loss were proposed by \citet{pal2024smaug,liu2024provably,pang2024iterative,gui2024bonbon} with that purpose in mind.
We note, however, that our experiments in \cref{sec:unalignment} reveal a limitation of this approach for mitigating likelihood displacement and its adverse effects, compared to improving the data curation pipeline.

\subsection{Theoretical Analysis: Effect on Likelihood Displacement}
\label{app:additional_sft:analysis}

We consider the technical setting laid out in \cref{sec:likelihood_dynamics:approach}, except that instead of examining gradient flow over the original preference learning loss $\L$ (\cref{eq:pref_loss}), we analyze  the dynamics of gradient flow over $\L_{\mathrm{S}}$ (\cref{eq:pref_loss_additional_sft}) and $\L_{\mathrm{w}}$ (\cref{eq:pref_loss_weights}):
\be
\frac{d}{dt} \thetasft (t) = - \nabla\L_{\mathrm{S}} \brk*{ \thetasft (t) } \quad , \quad \frac{d}{dt} \thetaw (t) = - \nabla \L_{\mathrm{w}} \brk*{ \thetaw (t) } \quad , ~t \geq 0
\text{\,,}
\label{eq:modified_losses_gf}
\ee
where $\thetasft (t)$ and $\thetaw (t)$ denote the parameters at time $t \geq 0$ when optimizing $\L_{\mathrm{S}}$ and $\L_{\mathrm{w}}$, respectively.
Suppose for simplicity that the dataset $\D$ contains a single preference sample $(\xbf, \yw, \yl)$.
The evolution of $\ln \prob_{\theta (t)} (\yw | \xbf)$ when minimizing the original loss $\L$ via gradient flow is given by:
\[
\frac{d}{dt} \ln \prob_{\theta (t)} (\yw | \xbf) = - \linkf' ( \theta(t) ) \inprod{ \nabla \ln \prob_{\theta (t)} (\yw | \xbf) }{ \nabla \ln \prob_{\theta (t)} (\yw | \xbf) - \nabla \ln \prob_{\theta (t)} (\yl | \xbf) }
\text{\,,}
\]
where $\linkf' ( \theta (t) ) := \linkf' \brk{ \ln \prob_{\theta (t)} (\yw|\xbf) - \ln \prob_{\theta (t)} (\yl|\xbf)}$.
We denote the term on the right hand side above, evaluated at some point $\theta$ instead of $\theta (t)$, by:
\[
\mathcal{E} (\theta) := - \linkf' ( \theta ) \inprod{ \nabla \ln \prob_{\theta} (\yw | \xbf) }{ \nabla \ln \prob_{\theta} (\yw | \xbf) - \nabla \ln \prob_{\theta} (\yl | \xbf) }
\text{\,.}
\]

\cref{prop:additional_sft} establishes that, when minimizing $\L_{\mathrm{S}}$ via gradient flow, the preferred response log probability evolves according to $\mathcal{E} (\thetasft (t))$, \ie~according to the evolution dictated by the original loss $\L$, and an additional non-negative term $\lambda \cdot \norm{ \nabla \ln \prob_{\thetasft(t)} (\yw | \xbf) }^2$.
\cref{prop:different_weights} similarly shows that, when minimizing $ \L_{\mathrm{w}}$ via gradient flow, the evolution of the preferred response log probability depends on $\mathcal{E} (\thetaw (t))$ (up to a multiplicative factor), and $\gamma (t) \cdot \norm{ \nabla \ln \prob_{\thetaw(t)} (\yw | \xbf) }^2$, where $\gamma (t) > 0$ when $ \lambda^+_{\xbf, \yw, \yl} >  \lambda^-_{\xbf, \yw, \yl}$.
This implies that, as expected, adding an SFT regularization term, or assigning the preferred response a weight larger than that of the dispreferred response, encourages the preferred response log probability to increase.

The proofs of \cref{prop:additional_sft,prop:different_weights} are given in \cref{app:proofs:additional_sft,app:proofs:different_weights}, respectively.

\begin{proposition}
	\label{prop:additional_sft}
	Suppose that the dataset $\dataset$ contains a single sample $(\xbf, \yw, \yl)$, with $\yw \in \V^*$ and $\yl \in \V^*$ satisfying $\yw_1 \neq \yl_1$.
	When minimizing $\L_{\mathrm{S}}$ via gradient flow (\cref{eq:modified_losses_gf}), at any time $t \geq 0$ it holds that:
	\[
	\begin{split}
		\frac{d}{dt} \ln \prob_{\thetasft (t)} (\yw | \xbf) = \mathcal{E} (\thetasft (t)) + \lambda \cdot \norm*{ \nabla \ln \prob_{\thetasft (t)} \brk*{ \yw | \xbf } }^2
		\text{\,.}
	\end{split}
	\]
\end{proposition}

\begin{proposition}
	\label{prop:different_weights}
	Suppose that the dataset $\dataset$ contains a single sample $(\xbf, \yw, \yl)$, with $\yw \in \V^*$ and $\yl \in \V^*$ satisfying $\yw_1 \neq \yl_1$.
	When minimizing $\L_{\mathrm{w}}$ via gradient flow (\cref{eq:modified_losses_gf}), at any time $t \geq 0$ it holds that:
	\[
	\begin{split}
		\frac{d}{dt} \ln \prob_{\thetaw (t)} (\yw | \xbf) =  \rho(t) \cdot \mathcal{E} (\thetaw (t)) + \gamma (t) \cdot \norm*{ \nabla \ln \prob_{\thetaw (t)} \brk*{ \yw | \xbf } }^2
		\text{\,,}
	\end{split}
	\]
	with $\rho (t) := \lambda^-_{\xbf, \yw, \yl} \cdot  \frac{ \mu' (\thetaw (t)) }{ \linkf' (\thetaw (t)) }$ and $\gamma (t) := ( \lambda^+_{\xbf, \yw, \yl} -  \lambda^-_{\xbf, \yw, \yl}) \cdot \brk[s]*{- \mu' (\thetaw (t)) }$, where:
	\[
		\mu' (\thetaw (t)) := \linkf' \brk*{  \lambda^+_{\xbf, \yw, \yl} \cdot \ln \prob_{\thetaw (t)} (\yw|\xbf) -  \lambda^-_{\xbf, \yw, \yl} \cdot \ln \prob_{\thetaw (t)} (\yl|\xbf)} < 0
		\text{\,.}
	\]
\end{proposition}

%% file: app_linear_no_displacement.tex
\section{Linear Models Do Not Suffer From Likelihood Displacement}
\label{app:linear_no_displacement}

As discussed in \cref{app:further_comparison}, this appendix establishes that linear models (trained on a prompt with single token responses) do not suffer from likelihood displacement.
Specifically, if one considers the setting of \cref{sec:likelihood_dynamics:overview:single_token}, but fixes the hidden embedding of the prompt $\xbf$ during training, then the probability of the preferred response $\yw$ cannot decrease.
This highlights the importance of taking into account how the hidden embeddings evolve during training when analyzing likelihood displacement, as done in \cref{sec:likelihood_dynamics}.

The proof of \cref{prop:linear_no_displacement} is deferred to \cref{app:proofs:linear_no_displacement}.

\begin{proposition}
\label{prop:linear_no_displacement}
Consider the setting of \cref{thm:gf_single_token_preferred_logprob_informal}, where the dataset $\dataset$ contains a single sample $(\xbf, \yw, \yl)$, with $\yw \in \vocab$ and $\yl \in \vocab$ each being a single token.
Suppose that $\rep_\xbf$, the hidden embedding of $\xbf$, is fixed during training, \ie~the trainable parameters are $\theta = \unembed$.
Then, $\ln \prob_{\theta (t)} (\yw | \xbf)$ is monotonically non-decreasing with respect to the training time $t$.
\end{proposition}

%% file: app_proofs.tex
\section{Deferred Proofs}
\label{app:proofs}

\subsection{Proof of \cref{thm:gf_single_token_preferred_logprob}}
\label{app:proofs:gf_single_token_preferred_logprob}

By the chain rule:
\be
\begin{split}
\frac{d}{dt} \ln \prob_{\theta (t)} (\yw | \xbf) & = \inprod{ \nabla \ln \prob_{\theta (t)} (\yw | \xbf) }{ \tfrac{d}{dt} \theta (t) } \\
& = - \linkf' (t) \cdot \inprod{ \nabla \ln \prob_{\theta (t)} (\yw | \xbf) }{ \nabla \ln \prob_{\theta (t)} (\yw | \xbf) - \nabla \ln \prob_{\theta (t)} (\yl | \xbf) }
\text{\,.}
\end{split}
\label{eq:sing_token_log_prob_time_deriv_initial}
\ee
For any token $z \in \V$, the gradient of $\ln \prob_{\theta (t)} (z | \xbf)$ at $\theta (t)$ consists of two components:
\[
\begin{split}
\nabla_{\unembed} \ln \prob_{\theta (t)} (z | \xbf) & = \brk*{ \ebf_{z} - \sum\nolimits_{z' \in \vocab} \prob_{\theta (t)} (z' | \xbf) \cdot \ebf_{z'} } \rep_{\xbf}^\top (t)  \text{\,,} \\
\nabla_{ \rep_\xbf } \ln \prob_{\theta (t)} (z | \xbf) & = \unembed_{z} (t) - \sum\nolimits_{z' \in \V} \prob_{\theta (t)} (z' | \xbf) \cdot \unembed_{z'} (t) \text{\,.}
\end{split}
\]
Thus:
\[
\begin{split}
\nabla_{\unembed} \ln \prob_{\theta (t)} (\yw | \xbf) - \nabla_{\unembed} \ln \prob_{\theta (t)} (\yl | \xbf) & = \brk*{ \ebf_{\yw} - \ebf_{\yl} } \rep_{\xbf}^\top (t)
\text{\,,} \\
\nabla_{\rep_\xbf} \ln \prob_{\theta (t)} (\yw | \xbf) - \nabla_{\rep_\xbf} \ln \prob_{\theta (t)} (\yl | \xbf) & = \unembed_{\yw} (t) - \unembed_{\yl} (t) \text{\,.}
\end{split}
\]
Going back to \cref{eq:sing_token_log_prob_time_deriv_initial}, we arrive at:
\[
\begin{split}
& \frac{d}{dt} \ln \prob_{\theta (t)} (\yw | \xbf)  \\
& = -\linkf' (t) \cdot \bigg [ \inprod{ \unembed_{\yw} (t) - \sum\nolimits_{z \in \vocab} \prob_{\theta (t)} (z | \xbf) \cdot \unembed_z (t)  }{ \unembed_{\yw} (t) - \unembed_{\yl} (t) } \\
& \hspace{30mm} + \inprod{ \brk*{ \ebf_{\yw} - \sum\nolimits_{z \in \vocab} \prob_{\theta (t)} (z | \xbf) \cdot \ebf_{z} } \rep_{\xbf}^\top (t)  }{ \brk*{ \ebf_{\yw} - \ebf_{\yl} } \rep_{\xbf}^\top (t) } \bigg ]
\text{\,.}
\end{split}
\]
Noticing that $\inprod{ \brk*{ \ebf_{\yw} - \sum\nolimits_{z \in \vocab} \prob_{\theta (t)} (z | \xbf) \cdot \ebf_{z} } \rep_{\xbf}^\top (t)  }{ \brk*{ \ebf_{\yw} - \ebf_{\yl} } \rep_{\xbf}^\top (t) }$ amounts to:
\[
\begin{split}
\brk*{ 1 - \prob_{\theta (t)} (\yw | \xbf) + \prob_{\theta (t)} (\yl | \xbf) } \cdot \norm*{ \rep_{\xbf} (t) }^2 
\text{\,,}
\end{split}
\]
the desired result readily follows by rearranging the equation above.
Lastly, \cref{lem:single_example_loss_deriv_sign} implies that $- \linkf (t) > 0$.
\qed

\subsection{Proof of \cref{thm:gf_single_token_where_mass_goes}}
\label{app:proofs:gf_single_token_where_mass_goes}

We perform a derivation analogous to that in the proof of \cref{thm:gf_single_token_preferred_logprob} (\cref{app:proofs:gf_single_token_preferred_logprob}).

By the chain rule:
\be
\begin{split}
\frac{d}{dt} \ln \prob_{\theta (t)} (z | \xbf) & = \inprod{ \nabla \ln \prob_{\theta (t)} (z | \xbf) }{ \tfrac{d}{dt} \theta (t) } \\
& = - \linkf' (t) \cdot \inprod{ \nabla \ln \prob_{\theta (t)} (z | \xbf) }{ \nabla \ln \prob_{\theta (t)} (\yw | \xbf) - \nabla \ln \prob_{\theta (t)} (\yl | \xbf) }
\text{\,.}
\end{split}
\label{eq:sing_token_log_prob_time_deriv_initial_other_token}
\ee
For any token $y \in \V$, the gradient of $\ln \prob_{\theta (t)} (y | \xbf)$ at $\theta (t)$ consists of two components:
\[
\begin{split}
\nabla_{\unembed} \ln \prob_{\theta (t)} (y | \xbf) & = \brk*{ \ebf_{y} - \sum\nolimits_{y' \in \vocab} \prob_{\theta (t)} (y' | \xbf) \cdot \ebf_{y'} } \rep_{\xbf}^\top (t)  \text{\,,} \\
\nabla_{ \rep_\xbf } \ln \prob_{\theta (t)} (y | \xbf) & = \unembed_{y} (t) - \sum\nolimits_{y' \in \V} \prob_{\theta (t)} (y' | \xbf) \cdot \unembed_{y'} (t) \text{\,.}
\end{split}
\]
Thus:
\[
\begin{split}
\nabla_{\unembed} \ln \prob_{\theta (t)} (\yw | \xbf) - \nabla_{\unembed} \ln \prob_{\theta (t)} (\yl | \xbf) & = \brk*{ \ebf_{\yw} - \ebf_{\yl} } \rep_{\xbf}^\top (t)
\text{\,,} \\
\nabla_{\rep_\xbf} \ln \prob_{\theta (t)} (\yw | \xbf) - \nabla_{\rep_\xbf} \ln \prob_{\theta (t)} (\yl | \xbf) & = \unembed_{\yw} (t) - \unembed_{\yl} (t) \text{\,.}
\end{split}
\]
Going back to \cref{eq:sing_token_log_prob_time_deriv_initial_other_token} thus leads to:
\[
\begin{split}
& \frac{d}{dt} \ln \prob_{\theta (t)} (z | \xbf)  \\
& = -\linkf' (t) \cdot \bigg [ \inprod{ \unembed_{z} (t) - \sum\nolimits_{z' \in \vocab} \prob_{\theta (t)} (z' | \xbf) \cdot \unembed_{z'} (t)  }{ \unembed_{\yw} (t) - \unembed_{\yl} (t) } \\
& \hspace{30mm} + \inprod{ \brk*{ \ebf_{z} - \sum\nolimits_{z' \in \vocab} \prob_{\theta (t)} (z' | \xbf) \cdot \ebf_{z'} } \rep_{\xbf}^\top (t)  }{ \brk*{ \ebf_{\yw} - \ebf_{\yl} } \rep_{\xbf}^\top (t) } \bigg ]
\text{\,.}
\end{split}
\]
Noticing that $\inprod{ \brk*{ \ebf_{z} - \sum\nolimits_{z' \in \vocab} \prob_{\theta (t)} (z' | \xbf) \cdot \ebf_{z'} } \rep_{\xbf}^\top (t)  }{ \brk*{ \ebf_{\yw} - \ebf_{\yl} } \rep_{\xbf}^\top (t) }$ amounts to:
\[
\begin{split}
\brk*{ \prob_{\theta (t)} (\yl | \xbf) - \prob_{\theta (t)} (\yw | \xbf)  } \cdot \norm*{ \rep_{\xbf} (t) }^2 
\text{\,,}
\end{split}
\]
the desired result readily follows by rearranging the equation above.
Lastly, we note that \cref{lem:single_example_loss_deriv_sign} implies that $- \linkf (t) > 0$.
\qed

\subsection{Proof of \cref{thm:gf_multiple_tokens_preferred_logprob}}
\label{app:proofs:gf_multiple_tokens_preferred_logprob}

Notice that, for any $\zbf \in \V^*$, the gradient $\nabla \ln \prob_{\theta (t)} (\zbf | \xbf)$ consists of the following components:
\be
\begin{split}
& \nabla_{\unembed} \ln \prob_{\theta (t)} (\zbf | \xbf) = \sum\nolimits_{k = 1}^{\abs{\zbf}} \brk*{ \ebf_{\zbf_k} - \prob_{\theta (t)} (\cdot | \xbf, \zbf_{< k} ) } \rep_{\zbf_{< k}}^\top (t)  \text{\,,} \\[0.3em]
& \nabla_{\rep_{\xbf, \zbf_{< k}}} \ln \prob_{\theta (t)} (\zbf | \xbf) = \unembed_{\zbf_k} (t) - \sum\nolimits_{z \in \vocab} \prob_{\theta (t)} (z | \xbf, \zbf_{ < k}) \cdot \unembed_{z} (t)  \quad , ~ k \in \{1, \ldots, \abs{\zbf} \} \text{\,,}
\end{split}
\label{eq:multi_token_gradient}
\ee
where the gradient with respect to all other hidden embeddings is zero.
By the chain rule:
\[
\begin{split}
\frac{d}{dt} \ln \prob_{\theta (t)} (\yw | \xbf) & = \inprod{ \nabla \ln \prob_{\theta (t)} (\yw | \xbf) }{ \tfrac{d}{dt} \theta (t) } \\
& = - \linkf' (t) \cdot \inprod{ \nabla \ln \prob_{\theta (t)} (\yw | \xbf) }{ \nabla \ln \prob_{\theta (t)} (\yw | \xbf) - \nabla \ln \prob_{\theta (t)} (\yl | \xbf) }
\text{\,.}
\end{split}
\]
Thus:
\[
\begin{split}
& \frac{d}{dt} \ln \prob_{\theta (t)} (\yw | \xbf) \\
& \hspace{4mm} = - \linkf' (t) \cdot \inprod{ \nabla_{\unembed} \ln \prob_{\theta (t)} (\yw | \xbf) }{ \nabla_{\unembed} \ln \prob_{\theta (t)} (\yw | \xbf) - \nabla_{\unembed} \ln \prob_{\theta (t)} (\yl | \xbf) } \\
& \hspace{7mm} - \linkf' (t) \cdot \inprod{ \nabla_{\rep_{\xbf}} \ln \prob_{\theta (t)} (\yw | \xbf) }{ \nabla_{\rep_{\xbf}} \ln \prob_{\theta (t)} (\yw | \xbf) - \nabla_{\rep_{\xbf}} \ln \prob_{\theta (t)} (\yl | \xbf) } \\
& \hspace{7mm} - \linkf' (t) \cdot \sum\nolimits_{k =2 }^{\abs{\yw}} \norm1{ \nabla_{\rep_{\xbf, \yw_{< k}}} \ln \prob_{\theta (t)} (\yw | \xbf) }^2
\text{\,.}
\end{split}
\]
Plugging in the expressions for each gradient from \cref{eq:multi_token_gradient} leads to:
\[
\begin{split}
& \!\!\!\!\!\!\!\!\! \frac{d}{dt} \ln \prob_{\theta (t)} (\yw | \xbf) = - \linkf' (t) \Bigg [ \\[0.3em]
& \underbrace{ \inprod{ \sum_{k = 1}^{\abs{\yw}} \brk*{ \ebf_{\yw_k} -  \prob_{\theta (t)} (\cdot | \xbf, \yw_{< k} )  } \rep_{\xbf, \yw_{< k}}^\top (t) }{ \sum_{k' = 1}^{\abs{\yw}} \brk*{ \ebf_{\yw_{k'}} -  \prob_{\theta (t)} (\cdot | \xbf, \yw_{< k'} )  } \rep_{\xbf, \yw_{< k'}}^\top (t) } }_{ (I) } \\
& - \underbrace{ \inprod{ \sum_{k = 1}^{\abs{\yw}} \brk*{ \ebf_{\yw_k} - \prob_{\theta (t)} (\cdot | \xbf, \yw_{< k} ) } \rep_{\xbf, \yw_{< k}}^\top (t) }{ \sum_{k' = 1}^{\abs{\yl}} \brk*{ \! \ebf_{\yl_{k'}} - \prob_{\theta (t)} (\cdot | \xbf, \yl_{< k'} ) } \rep_{\xbf, \yl_{< k'}}^\top (t) } }_{ (II) } \\
&  \underbrace{ \inprod{ \unembed_{\yw_1} (t) - \sum\nolimits_{z \in \vocab} \prob_{\theta (t)} (z | \xbf) \cdot \unembed_{z} (t) }{ \unembed_{\yw_1} (t) - \unembed_{\yl_1} (t) } }_{ (III) } \\
& \!\!\! \underbrace{ \sum\nolimits_{k = 2}^{\abs{\yw}} \norm*{ \unembed_{\yw_k} (t) - \sum\nolimits_{z \in \vocab} \prob_{\theta (t) } (z | \xbf, \yw_{<k}) \cdot \unembed_{z} (t) }^2 }_{ (IV) } \\
\Bigg ]
\text{\,.}
\end{split}
\]
Now, the sum of $(III)$ and $(IV)$ is equal to $m(t) + \singledyn_{\yw_1, \yl_1} (t)$.
As to $(I)$, for all $k \in \{1, \ldots, \abs{\yw} \}$ and $k' \in \{1, \ldots, \abs{\yw} \}$ we have that:
\[
\begin{split}
& \inprod{ \brk*{ \ebf_{\yw_k} - \prob_{\theta (t)} (\cdot | \xbf, \yw_{< k} ) } \rep_{\xbf, \yw_{< k}}^\top (t) }{ \brk*{ \ebf_{\yw_{k'}} - \prob_{\theta (t)} (\cdot | \xbf, \yw_{< k'} ) } \rep_{\xbf, \yw_{< k'}}^\top (t) } \\
&  \hspace{4mm} = \alpha^+_{k, k'} (t) \cdot \inprod{\rep_{\xbf, \yw_{< k}} (t) }{  \rep_{\xbf, \yw_{< k'}} (t) }
\text{\,.}
\end{split}
\]
This implies that:
\[
(I) =  \sum_{k = 1}^{\abs{\yw}} \sum_{k' = 1}^{\abs{\yw}} \alpha^{+}_{k, k'} (t) \cdot \inprod{ \rep_{\xbf, \yw_{< k} } (t) }{ \rep_{\xbf, \yw_{< k'} } (t) }
\text{\,.}
\]
An analogous derivation leads to:
\[
(II) = \sum_{k = 1}^{\abs{\yw}} \sum_{k' = 1}^{\abs{\yl}} \alpha^{-}_{k, k'} (t) \cdot \inprod{ \rep_{\xbf, \yw_{< k} } (t) } { \rep_{\xbf, \yl_{< k'} } (t) }
\text{\,.}
\]
Combining $(I), (II), (III)$, and $(IV)$ yields the desired expression for $\frac{d}{dt} \ln \prob_{\theta (t)} (\yw | \xbf)$.
Lastly, note that by \cref{lem:single_example_loss_deriv_sign} we have that $- \linkf (t) > 0$.
\qed

\subsection{Proof of \cref{thm:gf_multiple_tokens_where_mass_goes}}
\label{app:proofs:gf_multiple_tokens_where_mass_goes}

We perform a derivation analogous to that in the proof of \cref{thm:gf_multiple_tokens_preferred_logprob} (\cref{app:proofs:gf_multiple_tokens_preferred_logprob}).

For any $\vbf \in \V^*$, the gradient $\nabla \ln \prob_{\theta (t)} (\vbf | \xbf)$ consists of the following components:
\be
\begin{split}
	& \nabla_{\unembed} \ln \prob_{\theta (t)} (\vbf | \xbf) = \sum\nolimits_{k = 1}^{\abs{\vbf}} \brk*{ \ebf_{\vbf_k} - \prob_{\theta (t)} (\cdot | \xbf, \vbf_{< k} ) } \rep_{\vbf_{< k}}^\top (t)  \text{\,,} \\[0.3em]
	& \nabla_{\rep_{\xbf, \vbf_{< k} }} \ln \prob_{\theta (t)} (\vbf | \xbf) = \unembed_{\vbf_k} (t) - \sum\nolimits_{z \in \vocab} \prob_{\theta (t)} (z | \xbf, \vbf_{ < k}) \cdot \unembed_{z} (t)  \quad , ~ k \in \{1, \ldots, \abs{\vbf} \} \text{\,,}
\end{split}
\label{eq:multi_token_gradient_where_prob_mass_goes}
\ee
where the gradient with respect to all other hidden embeddings is zero.
By the chain rule:
\[
\begin{split}
	\frac{d}{dt} \ln \prob_{\theta (t)} (\zbf | \xbf) & = \inprod{ \nabla \ln \prob_{\theta (t)} (\zbf | \xbf) }{ \tfrac{d}{dt} \theta (t) } \\
	& = - \linkf' (t) \cdot \inprod{ \nabla \ln \prob_{\theta (t)} (\zbf | \xbf) }{ \nabla \ln \prob_{\theta (t)} (\yw | \xbf) - \nabla \ln \prob_{\theta (t)} (\yl | \xbf) }
	\text{\,.}
\end{split}
\]
Thus:
\[
\begin{split}
	& \frac{d}{dt} \ln \prob_{\theta (t)} (\zbf | \xbf) \\
	& \hspace{4mm} = - \linkf' (t) \cdot \inprod{ \nabla_{\unembed} \ln \prob_{\theta (t)} (\zbf | \xbf) }{ \nabla_{\unembed} \ln \prob_{\theta (t)} (\yw | \xbf) - \nabla_{\unembed} \ln \prob_{\theta (t)} (\yl | \xbf) } \\
	& \hspace{7mm} - \linkf' (t) \cdot \inprod{ \nabla_{\rep_{\xbf}} \ln \prob_{\theta (t)} (\yw | \xbf) }{ \nabla_{\rep_{\xbf}} \ln \prob_{\theta (t)} (\yw | \xbf) - \nabla_{\rep_{\xbf}} \ln \prob_{\theta (t)} (\yl | \xbf) }
	\text{\,.}
\end{split}
\]
Plugging in the expressions for each gradient from \cref{eq:multi_token_gradient_where_prob_mass_goes} leads to:
\[
\begin{split}
	& \!\!\!\!\!\!\!\!\! \frac{d}{dt} \ln \prob_{\theta (t)} (\yw | \xbf) = - \linkf' (t) \Bigg [ \\[0.3em]
	& \underbrace{ \inprod{ \sum_{k = 1}^{\abs{\zbf}} \brk2{ \ebf_{\zbf_k} - \prob_{\theta (t)} (\cdot | \xbf, \zbf_{< k} ) } \rep_{\xbf, \zbf_{< k}}^\top (t) }{ \sum_{k' = 1}^{\abs{\yw}} \brk*{ \ebf_{\yw_{k'}} - \prob_{\theta (t)} (\cdot | \xbf, \yw_{< k'} ) } \rep_{\xbf, \yw_{< k'}}^\top (t) } }_{ (I) } \\
	& - \underbrace{ \inprod{ \sum_{k = 1}^{\abs{\zbf}} \brk2{  \ebf_{\zbf_k} - \prob_{\theta (t)} (\cdot | \xbf, \zbf_{< k} )  } \rep_{\xbf, \zbf_{< k}}^\top (t) }{ \sum_{k' = 1}^{\abs{\yl}} \brk*{  \ebf_{\yl_{k'}} - \prob_{\theta (t)} (\cdot | \xbf, \yl_{< k'} ) } \rep_{\xbf, \yl_{< k'}}^\top (t) } }_{ (II) } \\
	& \underbrace{ \inprod{ \unembed_{\zbf_1} (t) - \sum\nolimits_{z \in \vocab} \prob_{\theta (t)} (z | \xbf) \cdot \unembed_{z} (t) }{ \unembed_{\yw_1} (t) - \unembed_{\yl_1} (t) } }_{ (III) } \\
	\Bigg ]
	\text{\,.}
\end{split}
\]
First, notice that $(III) = c (t) + \inprodbig{ \unembed_{\zbf_1} (t) }{ \unembed_{\yw_1} (t) - \unembed_{\yl_1} (t) } $.
As to $(I)$, for all $k \in \{1, \ldots, \abs{ \zbf} \}$ and $k' \in \{1, \ldots, \abs{\yw} \}$ we have that:
\[
\begin{split}
	& \inprod{ \brk2{ \ebf_{\zbf_k} -  \prob_{\theta (t)} (\cdot | \xbf, \zbf_{< k} ) } \rep_{\xbf, \zbf_{< k}}^\top (t) }{ \brk*{ \ebf_{\yw_{k'}} - \prob_{\theta (t)} (\cdot | \xbf, \yw_{< k'} ) } \rep_{\xbf, \yw_{< k'}}^\top (t) } \\
	& \hspace{4mm} = \beta^+_{k, k'} (t) \cdot \inprod{\rep_{\xbf, \zbf_{< k}} (t) }{  \rep_{\xbf, \yw_{< k'}} (t) }
	\text{\,.}
\end{split}
\]
This implies that:
\[
(I) =  \sum_{k = 1}^{\abs{\zbf}} \sum_{k' = 1}^{\abs{\yw}} \beta^{+}_{k, k'} (t) \cdot \inprod{ \rep_{\xbf, \zbf_{< k} } (t) }{ \rep_{\xbf, \yw_{< k'} } (t) }
\text{\,.}
\]
By a similar derivation we get that:
\[
(II) = \sum_{k = 1}^{\abs{\zbf}} \sum_{k' = 1}^{\abs{\yl}} \beta^{-}_{k, k'} (t) \cdot \inprod{ \rep_{\xbf, \zbf_{< k} } (t) } { \rep_{\xbf, \yl_{< k'} } (t) }
\text{\,.}
\]
Combining $(I), (II)$, and $(III)$ yields the desired expression for $\frac{d}{dt} \ln \prob_{\theta (t)} (\zbf | \xbf)$.
Lastly, note that by \cref{lem:single_example_loss_deriv_sign} it holds that $- \linkf (t) > 0$.
\qed

\subsection{Proof of \cref{thm:gf_multiple_examples_preferred_logprob}}
\label{app:proofs:gf_multiple_examples_preferred_logprob}

Let $\dataset_{\mathrm{add}} := \dataset \setminus \{ (\xbf, \yw, \yl)\}$ be the dataset obtained by excluding $(\xbf, \yw, \yl)$ from $\dataset$.
By the chain rule:
\be
\begin{split}
	& \frac{d}{dt} \ln \prob_{\theta (t)} (\yw | \xbf) \\
	& \hspace{3mm} = \inprod{ \nabla \ln \prob_{\theta (t)} (\yw | \xbf) }{ \tfrac{d}{dt} \theta (t) } \\
	& \hspace{3mm} = \frac{  - \linkf' (t) }{\abs{\dataset}} \cdot \underbrace{ \inprod{ \nabla \ln \prob_{\theta (t)} (\yw | \xbf) }{ \nabla \ln \prob_{\theta (t)} (\yw | \xbf) - \nabla \ln \prob_{\theta (t)} (\yl | \xbf) } }_{ (I) } \\
	& \hspace{7mm} + \!\! \sum_{(\tilde{\xbf}, \tilde{\ybf}^+, \tilde{\ybf}^-) \in \dataset_{\mathrm{add}}} \!\! \frac{ - \linkfother' (t) }{ \abs{\dataset} } \cdot \underbrace{ \inprod{ \nabla \ln \prob_{\theta (t)} (\yw | \xbf) }{ \nabla \ln \prob_{\theta (t)} (\tilde{\ybf}^+ | \tilde{\xbf}) - \nabla \ln \prob_{\theta (t)} (\tilde{\ybf}^- | \tilde{\xbf}) } }_{ (II) }
	\text{\,.}
\end{split}
\label{eq:mul_examples_log_prob_time_deriv_initial}
\ee
For any token $z \in \V$ and prompt $\tilde{\xbf} \in \V^*$, the gradient of $\ln \prob_{\theta (t)} (z | \tilde{\xbf} )$ at $\theta (t)$ is given by:
\[
\begin{split}
	\nabla_{\unembed} \ln \prob_{\theta (t)} (z | \tilde{\xbf}) & = \brk*{ \ebf_{z} - \sum\nolimits_{z' \in \vocab} \prob_{\theta (t)} (z' | \tilde{\xbf}) \cdot \ebf_{z'} } \rep_{\tilde{\xbf}}^\top (t)  \text{\,,} \\
	\nabla_{ \rep_{\tilde{\xbf}} } \ln \prob_{\theta (t)} (z | \tilde{\xbf}) & = \unembed_{z} (t) - \sum\nolimits_{z' \in \V} \prob_{\theta (t)} (z' | \tilde{\xbf}) \cdot \unembed_{z'} (t) \text{\,.}
\end{split}
\]
Furthermore, for any response $\xbf' \neq \tilde{\xbf}$, it holds that $\nabla_{\rep_{\xbf'}} \ln \prob_{\theta (t)} (z | \tilde{\xbf}) = 0$ since $\ln \prob_{\theta (t)} (z | \tilde{\xbf})$ does not depend on $\rep_{\xbf'}$ (recall that the hidden embeddings are treated as trainable parameters under the unconstrained features model).
Thus, focusing on term $(I)$ from \cref{eq:mul_examples_log_prob_time_deriv_initial}:
\[
\begin{split}
	&\inprod{ \nabla \ln \prob_{\theta (t)} (\yw | \xbf) }{ \nabla \ln \prob_{\theta (t)} (\yw | \xbf) - \nabla \ln \prob_{\theta (t)} (\yl | \xbf) }  \\
	& \hspace{4mm} =  \inprod{ \unembed_{\yw} (t) - \sum\nolimits_{z \in \vocab} \prob_{\theta (t)} (z | \xbf) \cdot \unembed_z (t)  }{ \unembed_{\yw} (t) - \unembed_{\yl} (t) } \\
	& \hspace{9mm} + \inprod{ \brk*{ \ebf_{\yw} - \sum\nolimits_{z \in \vocab} \prob_{\theta (t)} (z | \xbf) \cdot \ebf_{z} } \rep_{\xbf}^\top (t)  }{ \brk*{ \ebf_{\yw} - \ebf_{\yl} } \rep_{\xbf}^\top (t) }
	\text{\,.}
\end{split}
\]
Since $\inprod{ \brk*{ \ebf_{\yw} - \sum\nolimits_{z \in \vocab} \prob_{\theta (t)} (z | \xbf) \cdot \ebf_{z} } \rep_{\xbf}^\top (t)  }{ \brk*{ \ebf_{\yw} - \ebf_{\yl} } \rep_{\xbf}^\top (t) }$ amounts to:
\[
\begin{split}
	\brk*{ 1 - \prob_{\theta (t)} (\yw | \xbf) + \prob_{\theta (t)} (\yl | \xbf) } \cdot \norm*{ \rep_{\xbf} (t) }^2 
	\text{\,,}
\end{split}
\]
it readily follows that $(I) = m(t) + \singledyn_{\yw, \yl} (t)$ by rearranging terms.

Moving on to term $(II)$ from \cref{eq:mul_examples_log_prob_time_deriv_initial}, for any $(\tilde{\xbf}, \tilde{\ybf}^+, \tilde{\ybf}^-) \in \dataset_{\mathrm{add}}$ we have that:
\[
\begin{split}
	& \inprod{ \nabla \ln \prob_{\theta (t)} (\yw | \xbf) }{ \nabla \ln \prob_{\theta (t)} (\tilde{\ybf}^+ | \tilde{\xbf}) - \nabla \ln \prob_{\theta (t)} (\tilde{\ybf}^- | \tilde{\xbf}) } \\
	& \hspace{4mm} = \inprod{ \brk*{ \ebf_{\yw} - \sum\nolimits_{z \in \vocab} \prob_{\theta (t)} (z | \xbf) \cdot \ebf_{z} } \rep_{\xbf}^\top (t)  }{ \brk*{ \ebf_{\tilde{\ybf}^+} - \ebf_{\tilde{\ybf}^-} } \rep_{\tilde{\xbf}}^\top (t) } \\
	& \hspace{4mm} =  \inprod{ \ebf_{\yw} - \sum\nolimits_{z \in \vocab} \prob_{\theta (t)} (z | \xbf) \cdot \ebf_{z}  }{  \ebf_{\tilde{\ybf}^+} - \ebf_{\tilde{\ybf}^-} } \cdot \inprod{\rep_{\xbf} (t) }{ \rep_{\tilde{\xbf}}(t) } \\
	& \hspace{4mm} = \alpha_{\xbf, \tilde{\xbf}} (t) \cdot \inprod{\rep_{\xbf} (t) }{ \rep_{\tilde{\xbf}}(t) }
	\text{\,.}
\end{split}
\]
Plugging $(I)$ and $(II)$ back into \cref{eq:mul_examples_log_prob_time_deriv_initial} concludes the proof.
\qed

\subsection{Proof of \cref{thm:gf_multiple_examples_where_mass_goes}}
\label{app:proofs:gf_multiple_examples_where_mass_goes}

We perform a derivation analogous to that in the proof of \cref{thm:gf_multiple_examples_preferred_logprob} (\cref{app:proofs:gf_multiple_examples_preferred_logprob}).

Applying the chain rule:
\be
\begin{split}
	& \frac{d}{dt} \ln \prob_{\theta (t)} (z | \xbf) \\
	& \hspace{3mm} = \inprod{ \nabla \ln \prob_{\theta (t)} (z | \xbf) }{ \tfrac{d}{dt} \theta (t) } \\
	& \hspace{3mm} =  \sum_{(\tilde{\xbf}, \tilde{\ybf}^+, \tilde{\ybf}^-) \in \dataset} \!\! \frac{ - \linkfother' (t) }{ \abs{\dataset} } \cdot \inprod{ \nabla \ln \prob_{\theta (t)} (z | \xbf) }{ \nabla \ln \prob_{\theta (t)} (\tilde{\ybf}^+ | \tilde{\xbf}) - \nabla \ln \prob_{\theta (t)} (\tilde{\ybf}^- | \tilde{\xbf}) }
	\text{\,.}
\end{split}
\label{eq:mul_examples_where_prob_goes_time_deriv_initial}
\ee
For any token $y \in \V$ and prompt $\tilde{\xbf} \in \V^*$, the gradient of $\ln \prob_{\theta (t)} (y | \tilde{\xbf} )$ at $\theta (t)$ is given by:
\[
\begin{split}
	\nabla_{\unembed} \ln \prob_{\theta (t)} (y | \tilde{\xbf}) & = \brk*{ \ebf_{y} - \sum\nolimits_{y' \in \vocab} \prob_{\theta (t)} (y' | \tilde{\xbf}) \cdot \ebf_{y'} } \rep_{\tilde{\xbf}}^\top (t)  \text{\,,} \\
	\nabla_{ \rep_{\tilde{\xbf}} } \ln \prob_{\theta (t)} (y | \tilde{\xbf}) & = \unembed_{y} (t) - \sum\nolimits_{y' \in \V} \prob_{\theta (t)} (y' | \tilde{\xbf}) \cdot \unembed_{y'} (t) \text{\,.}
\end{split}
\]
Furthermore, for any response $\xbf' \neq \tilde{\xbf}$ it holds that $\nabla_{\rep_{\xbf'}} \ln \prob_{\theta (t)} (y | \tilde{\xbf}) = 0$ since $\ln \prob_{\theta (t)} (y | \tilde{\xbf})$ does not depend on $\rep_{\xbf'}$ (recall that the hidden embeddings are treated as trainable parameters under the unconstrained features model).
Focusing on the summand from \cref{eq:mul_examples_where_prob_goes_time_deriv_initial} corresponding to $(\xbf, \yw, \yl)$, we thus get:
\[
\begin{split}
	&\inprod{ \nabla \ln \prob_{\theta (t)} (z | \xbf) }{ \nabla \ln \prob_{\theta (t)} (\yw | \xbf) - \nabla \ln \prob_{\theta (t)} (\yl | \xbf) }  \\
	& \hspace{4mm} = \inprod{ \unembed_{z} (t) - \sum\nolimits_{z' \in \vocab} \prob_{\theta (t)} (z' | \xbf) \cdot \unembed_{z'} (t)  }{ \unembed_{\yw} (t) - \unembed_{\yl} (t) } \\
	& \hspace{9mm} + \inprod{ \brk*{ \ebf_{z} - \sum\nolimits_{z' \in \vocab} \prob_{\theta (t)} (z' | \xbf) \cdot \ebf_{z'} } \rep_{\xbf}^\top (t)  }{ \brk*{ \ebf_{\yw} - \ebf_{\yl} } \rep_{\xbf}^\top (t) }
	\text{\,.}
\end{split}
\]
Since $\inprod{ \brk*{ \ebf_{z} - \sum\nolimits_{z' \in \vocab} \prob_{\theta (t)} (z' | \xbf) \cdot \ebf_{z'} } \rep_{\xbf}^\top (t)  }{ \brk*{ \ebf_{\yw} - \ebf_{\yl} } \rep_{\xbf}^\top (t) }$ amounts to:
\[
\begin{split}
	\brk*{ \indc{z = \yw } - \indc{ z = \yl } - \prob_{\theta (t)} (\yw | \xbf) + \prob_{\theta (t)} (\yl | \xbf) } \cdot \inprod{  \rep_{\xbf} (t) }{ \rep_{\xbf} (t) } 
	\text{\,,}
\end{split}
\]
it follows that:
\be
\begin{split}
	&\inprod{ \nabla \ln \prob_{\theta (t)} (z | \xbf) }{ \nabla \ln \prob_{\theta (t)} (\yw | \xbf) - \nabla \ln \prob_{\theta (t)} (\yl | \xbf) }  \\
	& \hspace{4mm} = \inprod{ \unembed_{z} (t)  }{ \unembed_{\yw} (t) - \unembed_{\yl} (t) } - \sum\nolimits_{z' \in \vocab} \prob_{\theta (t)} (z' | \xbf) \cdot \inprod{\unembed_{z'} (t) }{  \unembed_{\yw} (t) - \unembed_{\yl} (t) } \\
	& \hspace{9mm} + \brk*{ \indc{z = \yw } - \indc{ z = \yl } - \prob_{\theta (t)} (\yw | \xbf) + \prob_{\theta (t)} (\yl | \xbf) } \cdot \inprod{  \rep_{\xbf} (t) }{ \rep_{\xbf} (t) } 
	\text{\,.}
\end{split}
\label{eq:mul_examples_same_example_contrib_interm}
\ee
Now, for $(\tilde{\xbf}, \tilde{\ybf}^+, \tilde{\ybf}^-) \in \dataset \setminus \{ (\xbf, \yw, \yl) \}$, the corresponding summand from \cref{eq:mul_examples_where_prob_goes_time_deriv_initial} can be written as:
\be
\begin{split}
	& \inprod{ \nabla \ln \prob_{\theta (t)} (z | \xbf) }{ \nabla \ln \prob_{\theta (t)} (\tilde{\ybf}^+ | \tilde{\xbf}) - \nabla \ln \prob_{\theta (t)} (\tilde{\ybf}^- | \tilde{\xbf}) } \\
	& \hspace{4mm} = \inprod{ \brk*{ \ebf_{z} - \sum\nolimits_{z' \in \vocab} \prob_{\theta (t)} (z' | \xbf) \cdot \ebf_{z'} } \rep_{\xbf}^\top (t)  }{ \brk*{ \ebf_{\tilde{\ybf}^+} - \ebf_{\tilde{\ybf}^-} } \rep_{\tilde{\xbf}}^\top (t) } \\
	& \hspace{4mm} =  \inprod{ \ebf_{z} - \sum\nolimits_{z' \in \vocab} \prob_{\theta (t)} (z' | \xbf) \cdot \ebf_{z'}  }{  \ebf_{\tilde{\ybf}^+} - \ebf_{\tilde{\ybf}^-} } \cdot \inprod{\rep_{\xbf} (t) }{ \rep_{\tilde{\xbf}}(t) } \\
	& \hspace{4mm} = \brk*{ \indc{z = \tilde{\ybf}^+ } - \indc{ z = \tilde{\ybf}^- } - \prob_{\theta (t)} (\tilde{\ybf}^+ | \xbf) + \prob_{\theta (t)} (\tilde{\ybf}^- | \xbf) }  \cdot \inprod{\rep_{\xbf} (t) }{ \rep_{\tilde{\xbf}}(t) }
	\text{\,.}
\end{split}
\label{eq:mul_examples_other_example_contrib_interm}
\ee
Plugging \cref{eq:mul_examples_same_example_contrib_interm,eq:mul_examples_other_example_contrib_interm} back into \cref{eq:mul_examples_where_prob_goes_time_deriv_initial} concludes the proof.
\qed

\subsection{Proof of \cref{prop:additional_sft}}
\label{app:proofs:additional_sft}

The proof readily follows by a straightforward application of the chain rule:
\[
\begin{split}
& \frac{d}{dt} \ln \prob_{\thetasft (t)} (\yw | \xbf) \\
&  =  \inprod{ \nabla \ln \prob_{\thetasft (t)} (\yw | \xbf) }{ \tfrac{d}{dt} \thetasft (t) } \\
& = \inprod{ \nabla \ln \prob_{\thetasft (t)} (\yw | \xbf) }{ - \linkf' ( \thetasft (t) ) \brk*{ \nabla \ln \prob_{\thetasft (t)} (\yw | \xbf) - \nabla \ln \prob_{\thetasft (t)} (\yl | \xbf) }  } \\
& \hspace{4mm} + \lambda \cdot \norm*{   \nabla \ln \prob_{\thetasft (t)} (\yw | \xbf) }^2 \\
& = \mathcal{E} (\thetasft (t)) +  \lambda \cdot \norm*{   \nabla \ln \prob_{\thetasft (t)} (\yw | \xbf) }^2 
\text{\,,}
\end{split}
\]
where $\linkf' ( \thetasft (t) ) := \linkf' \brk*{ \ln \prob_{\thetasft (t)} (\yw|\xbf) - \ln \prob_{\thetasft (t)} (\yl|\xbf)}$.
\qed

\subsection{Proof of \cref{prop:different_weights}}
\label{app:proofs:different_weights}

By the chain rule and a straightforward rearrangement of terms:
\[
\begin{split}
	& \frac{d}{dt} \ln \prob_{\thetaw (t)} (\yw | \xbf) \\
	&  =  \inprod{ \nabla \ln \prob_{\thetaw (t)} (\yw | \xbf) }{ \tfrac{d}{dt} \thetaw (t) } \\
	& = - \mu' (\thetaw (t) ) \cdot \inprod{ \nabla \ln \prob_{\thetaw (t)} (\yw | \xbf) }{   \lambda^+_{\xbf, \yw, \yl} \nabla \ln \prob_{\thetaw (t)} (\yw | \xbf) -  \lambda^-_{\xbf, \yw, \yl} \nabla \ln \prob_{\thetaw (t)} (\yl | \xbf)   } \\
	& =  -  \lambda^-_{\xbf, \yw, \yl} \mu' (\thetaw (t) ) \cdot \inprod{ \nabla \ln \prob_{\thetaw (t)} (\yw | \xbf) }{ \nabla \ln \prob_{\thetaw (t)} (\yw | \xbf) -  \nabla \ln \prob_{\thetaw (t)} (\yl | \xbf)  } \\
	& \hspace{4mm} + \brk1{ \lambda^+_{\xbf, \yw, \yl}  -  \lambda^-_{\xbf, \yw, \yl}  } \brk[s]*{- \mu' (\thetaw (t)) } \cdot \norm*{  \nabla \ln \prob_{\thetaw (t)} (\yw | \xbf)  }^2 \\
	& =  \rho(t) \cdot \mathcal{E} (\thetaw (t)) + \gamma (t) \cdot \norm*{ \nabla \ln \prob_{\thetaw (t)} \brk*{ \yw | \xbf } }^2
	\text{\,.}
\end{split}
\]
Lastly, steps analogous to those used for proving \cref{lem:single_example_loss_deriv_sign} establish that $\mu' (\thetaw (t)) < 0$, and so $- \mu' (\thetaw (t)) > 0$.
\qed

\subsection{Proof of \cref{prop:linear_no_displacement}}
\label{app:proofs:linear_no_displacement}

The claim follows by showing that $\frac{d}{dt} \ln \prob_{\theta (t)} (\yw | \xbf) \geq 0$ for all $t \geq 0$.
To see it is so, notice that when $\rep_\xbf$ is not trainable, for any token $z \in \V$ the gradient of $\ln \prob_{\theta (t)} (z | \xbf)$ at $\theta (t)$ is given by:
\[
\begin{split}
	\nabla \ln \prob_{\theta (t)} (z | \xbf) & = \brk*{ \ebf_{z} - \sum\nolimits_{z' \in \vocab} \prob_{\theta (t)} (z' | \xbf) \cdot \ebf_{z'} } \rep_{\xbf}^\top 
	\text{\,.}
\end{split}
\]
Thus, applying the chain rule:
\[
\begin{split}
	\frac{d}{dt} \ln \prob_{\theta (t)} (\yw | \xbf) & = \inprod{ \nabla \ln \prob_{\theta (t)} (\yw | \xbf) }{ \tfrac{d}{dt} \theta (t) } \\
	& = - \linkf' (t) \cdot \inprod{ \nabla \ln \prob_{\theta (t)} (\yw | \xbf) }{ \nabla \ln \prob_{\theta (t)} (\yw | \xbf) - \nabla \ln \prob_{\theta (t)} (\yl | \xbf) } \\
	& = - \linkf' (t) \cdot \inprod{ \brk*{ \ebf_{\yw} - \sum\nolimits_{z \in \vocab} \prob_{\theta (t)} (z | \xbf) \cdot \ebf_{z} } \rep_{\xbf}^\top  }{ \brk*{ \ebf_{\yw} -  \ebf_{\yl} } \rep_{\xbf}^\top } \\
	& = - \linkf' (t) \cdot \inprod{ \ebf_{\yw} - \sum\nolimits_{z \in \vocab} \prob_{\theta (t)} (z | \xbf) \cdot \ebf_{z}   }{ \ebf_{\yw} -  \ebf_{\yl} } \norm*{ \rep_\xbf }^2 \\
	& = - \linkf' (t) \cdot \brk*{ 1 - \prob_{\theta (t)} (\yw | \xbf) + \prob_{\theta (t)} (\yl | \xbf) } \norm*{ \rep_\xbf }^2
	\text{\,.}
\end{split}
\]
It then readily follows that $\frac{d}{dt} \ln \prob_{\theta (t)} (\yw | x) \geq 0$ by the fact that $ 1 - \prob_{\theta (t)} (\yw | \xbf) + \prob_{\theta (t)} (\yl | \xbf) \geq 0$, along with \cref{lem:single_example_loss_deriv_sign}, which implies that $ - \linkf' (t) \geq 0$.
\qed

\subsection{Auxiliary Lemmas}
\label{app:proofs:aux_lemma}

\begin{lemma}
\label{lem:dpo_loss_deriv_sign}
For $(\xbf, \yw, \yl) \in \dataset$, suppose that $\linkf$ corresponds to the DPO loss, \ie:
\[
\linkf \brk*{ \ln \prob_{\theta} (\yw | \xbf) - \ln \prob_{\theta} (\yl | \xbf) } := - \ln \sigma \brk*{ \beta \brk*{ \ln \frac{ \prob_\theta(\yw|\xbf) }{ \prob_\theta(\yl|\xbf) } - \ln \frac{ \piref(\yw|\xbf) }{ \piref(\yl|\xbf) } } }
\text{\,,}
\]
where $\piref$ is some reference model, $\beta > 0$ is a regularization hyperparameter, and $\sigma : \R \to [0, 1]$ denotes the sigmoid function.
Then, at any time $t \geq 0$ of training:
\[
\linkf' \brk*{ \ln \prob_{\theta (t)} (\yw | \xbf) - \ln \prob_{\theta (t)} (\yl | \xbf)  } < 0
\text{\,.}
\]
\end{lemma}

\begin{proof}
A straightforward differentiation of $\linkf (u)$ at any $u \in \R$ shows that:
\[
\linkf' (u) = -\beta \cdot \sigma \brk*{  \beta \brk*{ \ln \frac{ \piref(\yw|\xbf) }{ \piref(\yl|\xbf) } - u } } < 0
\text{\,.}
\]
\end{proof}

\begin{lemma}
\label{lem:single_example_loss_deriv_sign}
Suppose that the dataset $\dataset$ contains a single sample $(\xbf, \yw, \yl)$, with $\yw \in \V^*$ and $\yl \in \V^*$.
Then, at any time $t \geq 0$ of training:
\[
\linkf' \brk*{ \ln \prob_{\theta (t)} (\yw | \xbf) - \ln \prob_{\theta (t)} (\yl | \xbf) } < 0
\text{\,.}
\]
\end{lemma}

\begin{proof}
At time $t = 0$, our assumption that $\linkf$ is convex and monotonically decreasing in a neighborhood of $\ln \prob_{\theta (0)} (\yw | \xbf) - \ln \prob_{\theta (0)} (\yl | \xbf)$ (see \cref{sec:prelim:direct_preference_learning}) implies that:
\[
\linkf' \brk*{ \ln \prob_{\theta (0)} (\yw | \xbf) - \ln \prob_{\theta (0)} (\yl | \xbf) } < 0
\text{\,.}
\]
Suppose for the sake of contradiction that there exists a time $t \geq 0$ at which:
\[
\linkf' \brk*{ \ln \prob_{\theta (t)} (\yw | \xbf) - \ln \prob_{\theta (t)} (\yl | \xbf) } \geq 0
\text{\,.}
\]
By the continuity of $\linkf' \brk*{ \ln \prob_{\theta (t)} (\yw | \xbf) - \ln \prob_{\theta (t)} (\yl | \xbf) }$ with respect to $t$ and the intermediate value theorem (note that $\linkf'$ is continuous since $\linkf$ is convex), this implies that at some $t_0 \in [0, t]$:
\[
\linkf' \brk*{ \ln \prob_{\theta (t_0)} (\yw | \xbf) - \ln \prob_{\theta (t_0)} (\yl | \xbf) } = 0
\text{\,.}
\]
However, given that $\dataset$ contains only the sample $(\xbf, \yw, \yl)$, we have that:
\[
\nabla_\theta \L (\theta (t_0)) = \linkf' \brk*{ \ln \prob_{\theta (t_0)} (\yw | \xbf) - \ln \prob_{\theta (t_0)} (\yl | \xbf) } \cdot \nabla_\theta \ln \frac{ \prob_{\theta (t_0)} (\yw | \xbf) }{ \prob_{\theta (t_0)} (\yl | \xbf) } = 0
\text{\,.}
\]
Meaning, at time $t_0$ gradient flow is at a critical point of $\L$.
This stands in contradiction to $\linkf' \brk*{ \ln \prob_{\theta (0)} (\yw | \xbf) - \ln \prob_{\theta (0)} (\yl | \xbf) }$ being negative since gradient flow can only reach a critical point if it is initialized there (due to the uniqueness of the gradient flow solution and the existence of a solution that remains in the critical point through time).
As a result, it must be that $\linkf' \brk*{ \ln \prob_{\theta (t)} (\yw | \xbf) - \ln \prob_{\theta (t)} (\yl | \xbf) } < 0$ for all $t \geq 0$.
\end{proof}

%% file: app_experiments.tex
\section{Further Experiments}
\label{app:experiments:further}

\subsection{Catastrophic Likelihood Displacement in Simple Settings (\cref{sec:catastrophic})}
\label{app:experiments:further:catastrophic}

Listed below are additional experiments and results, omitted from \cref{sec:catastrophic}.

\begin{itemize}[leftmargin=5mm]
	\item \cref{table:persona_single_example_base} reports the results of an experiment analogous to that of \cref{table:persona_single_example_post_sft}, using base models that did not undergo an initial SFT phase.
	
	\item \cref{table:persona_single_example_post_sft_ipo} reports the results of an experiment analogous to that of \cref{table:persona_single_example_post_sft}, using IPO instead of DPO.
	
	\item \cref{tab:persona_toks_olmo_1b_dpo,tab:persona_toks_gemma_2b_dpo,tab:persona_toks_llama_3_8b_dpo} include details regarding the tokens increasing most in probability for the experiments of \cref{table:persona_single_example_post_sft}.
	
	\item \cref{tab:persona_toks_olmo_1b_dpo_on_base_model,tab:persona_toks_gemma_2b_dpo_on_base_model,tab:persona_toks_llama_3_8b_dpo_on_base_model} include details regarding the tokens increasing most in probability for the experiments of \cref{table:persona_single_example_base}.
	
	\item \cref{tab:persona_toks_olmo_1b_ipo,tab:persona_toks_gemma_2b_ipo,tab:persona_toks_llama_3_8b_ipo} include details regarding the tokens increasing most in probability for the experiments of \cref{table:persona_single_example_post_sft_ipo}.
	
	\item \cref{table:persona_pref_vs_proj_direction_norm} reports, for each model and pair of preferred and dispreferred tokens $(\yw, \yl)$ from \cref{table:persona_single_example_post_sft}, the norm of the projection of $\unembed_{\yw} - \unembed_{\yl}$ onto $\unembed_{\yw}$, as well as the norm of the component of $\unembed_{\yw} - \unembed_{\yl}$ orthogonal to $\unembed_{\yw}$.
	As the table shows, the norm of the orthogonal component is larger across the different models and preference pairs, in accordance with our theoretical explanation of why likelihood displacement can be catastrophic in the case of single token responses (\cref{sec:likelihood_dynamics}).
\end{itemize}

\subsection{Empirical Evaluation of the Coefficients From \cref{thm:gf_multiple_tokens_preferred_logprob_informal}}
\label{app:experiments:further:coeff}

In \cref{sec:likelihood_dynamics:overview:multiple_tokens}, we defined the CHES score (\cref{def:ches}) based on \cref{thm:gf_multiple_tokens_preferred_logprob_informal}.
Our definition was motivated by the empirical observation that the $\alpha^-_{k,k'} (t)$ and $\alpha^+_{k, k'} (t)$ coefficients, appearing in \cref{thm:gf_multiple_tokens_preferred_logprob_informal}, are mostly positive across models and datasets.
Specifically, across the OLMo-1B, Gemma-2B, and Llama-3-8B models and the UltraFeedback and AlpacaFarm datasets, we find that on average over 69\% of the coefficients are positive.
Although the number of negative coefficients is not negligible, the experiments in \cref{sec:ches,sec:unalignment} corroborate the simplification made for deriving the CHES score~---~namely, setting all coefficients to a constant positive value~---~by demonstrating that that the CHES score accurately predicts the extent to which samples contribute to likelihood displacement.

\subsection{Identifying Sources of Likelihood Displacement (\cref{sec:ches})}
\label{app:experiments:further:hidden_embedding}

Listed below are additional experiments and results, omitted from \cref{sec:ches}.

\begin{itemize}[leftmargin=5mm]
	
	\item \cref{fig:alpacafarm_displacement_to_pref_sim_dpo} includes experiments analogous to those of \cref{fig:ultrafeedback_displacement_to_pref_sim_dpo}, over the AlpacaFarm dataset instead of UltraFeedback.
	
	\item \cref{fig:ultrafeedback_displacement_to_pref_sim_ipo} includes experiments analogous to those of \cref{fig:alpacafarm_displacement_to_pref_sim_dpo}, using IPO instead of DPO.
	
	\item \cref{fig:ultrafeedback_displacement_to_pref_sim_olmo} includes experiments analogous to those of \cref{fig:ultrafeedback_displacement_to_pref_sim_dpo}, using an OLMo-1B model trained via DPO and IPO over the AlpacaFarm dataset.
	
	\item \cref{tab:ultrafeedback_highest_examples,tab:ultrafeedback_least_examples} include representative samples with high and low CHES scores from the UltraFeedback dataset, for the OLMo-1B, Gemma-2B, and Llama-3-8B models.
\end{itemize}

\subsection{Unintentional Unalignment in Direct Preference Learning (\cref{sec:unalignment})}
\label{app:experiments:further:unalignment}

Listed below are additional experiments and results, omitted from \cref{sec:unalignment}.

\begin{itemize}[leftmargin=5mm]
	\item \cref{tab:sorrybench_dpo_logprobs} reports the mean change in preferred response log probability for the experiments of \cref{fig:sorrybench_refusal_rate_dpo}.
	
	\item \cref{fig:sorrybench_refusal_rate_ipo} reports the results of an experiment analogous to that of \cref{fig:sorrybench_refusal_rate_dpo}, using IPO instead of DPO.
	Furthermore, \cref{tab:sorrybench_ipo_logprobs} reports the mean change in preferred response log probability for the experiments of \cref{fig:sorrybench_refusal_rate_ipo}.
	
	\item \cref{tab:sorrybench_examples} provides examples of unsafe prompts from the training sets used in \cref{sec:unalignment}, for which initially the models generated two refusal responses, yet after DPO they complied with the prompts.
\end{itemize}

\section{Further Implementation Details}
\label{app:experiments:details}

In this appendix, we provide implementation details omitted from \cref{sec:catastrophic,sec:ches,sec:unalignment,app:experiments:further}.
Code for reproducing our results, based on the PyTorch~\citep{paszke2017automatic} and Hugging Face~\citep{wolf2019huggingface} frameworks,\ifdefined\CAMREADY
~can be found at \small{\url{https://github.com/princeton-nlp/unintentional-unalignment}}.
\else
~will be made publicly available.
\fi

\subsection{Catastrophic Likelihood Displacement in Simple Settings (\cref{sec:catastrophic})}
\label{app:experiments:details:catastrophic}

\textbf{Data.}
In our experiments, we used the Persona dataset \citep{perez2022discovering} subset corresponding to the “ends justify means" behavioral trait, which contains 1000 prompts.
Each prompt was formatted using the template “Is the following statement something you would say? [STATEMENT]\textbackslash n", for statements that either accord or do not accord with the trait.
To mimic a scenario where one aims to align a model with safe behaviors, during the initial SFT phase, we randomly assigned each prompt that accords with the (undesirable) trait a label from $\mathcal{N} = \{ \texttt{No}, \texttt{Never}, \texttt{Maybe}, \texttt{Perhaps}\}$, and to each prompt that does not accord with the trait a label from $\mathcal{Y} = \{ \texttt{Yes}, \texttt{Yeah}, \texttt{Sure}, \texttt{Certainly}, \texttt{Absolutely} \}$.
In line with the SFT phase, when training via DPO (or IPO) using a preference pair $(\yw, \yl)$, if $ \yw \in \mathcal{N}$ then we selected randomly prompts that accord with the trait, whereas if $\yw \in \mathcal{Y}$ then we selected randomly prompts that do not accord with the trait.

\textbf{Training.}
In the initial SFT phase, we minimized the cross entropy loss over all 1000 prompts for one epoch, using the RMSProp optimizer \citep{hinton2012neural} with a learning rate of 1e-7 and batch size of 32.
For DPO, we performed 100 training steps using the RMSProp optimizer over a single prompt in each run, with a learning rate of 1e-7, and set the KL coefficient to 0.1, in line with \citet{rafailov2023direct,tajwar2024preference,xu2024dpo,dubey2024llama}.
Setting the learning rate to 5e-7 or 5e-8 led to analogous results.
For IPO, we decreased the learning rate to 1e-8, since higher learning rates led to unstable training, and set the KL coefficient to 0.01 (lower KL coefficients led to analogous results and higher coefficients caused the log probabilities to not change much during training).

\textbf{Further details.}
For each model and pair of preferred and dispreferred tokens $(\yw, \yl)$, we carried out ten DPO (or IPO) runs differing in random seed for choosing the prompt.
We report the results only for runs in which the training loss decreased throughout all steps to ensure that likelihood displacement did not occur due to instability of optimization.
For all configurations, the loss was completely stable in at least five runs.
In \cref{table:persona_single_example_post_sft,table:persona_single_example_base,table:persona_single_example_post_sft_ipo}, the reported decrease in preferred token probability stands for the largest decrease between any two (not necessarily consecutive) training steps.
That is, we report the minimal value of $\prob_{\theta (t')} (\yw | \xbf) - \prob_{\theta (t)}(\yw | \xbf)$ among any training steps $t < t'$.

\textbf{Hardware.}
Experiments for OLMo-1B and Gemma-2B ran on a single Nvidia H100 GPU with 80GB memory, while for Llama-3-8B we used three such GPUs per run.

\subsection{Identifying Sources of Likelihood Displacement (\cref{sec:ches})}
\label{app:experiments:details:hidden_embedding}

\textbf{Data.}
We used the binarized version of UltraFeedback \citep{tunstall2023zephyr}, and for computational efficiency, based our experiments on a randomly selected subset of 5000 samples from the training set.
For AlpacaFarm, we took the human preferences subset that contains 9691 samples.
We filtered out samples in which either: \emph{(i)} the prompt was longer than 512 tokens; \emph{(ii)} the prompt was empty; or \emph{(iii)} one of the responses were empty.

For each prompt $\xbf$ and response $\ybf$, the input to the language models was formatted according to:
\[
\text{“\texttt{[PROMPT\_TOKEN]}$\xbf$\texttt{[ASSISTANT\_TOKEN]}$\ybf$\texttt{[EOS\_TOKEN]}"}
\text{\,,}
\]
where \texttt{[PROMPT\_TOKEN]}, \texttt{[ASSISTANT\_TOKEN]}, and \texttt{[EOS\_TOKEN]} are defined as special tokens, and truncated to a maximum length of 512 tokens.

\textbf{Training.}
For each model and preference similarity percentile subset, we ran one epoch of DPO (or IPO), using the RMSProp optimizer with a learning rate of 1e-7 and batch size of 32 (emulated via 8 gradient accumulation steps with a batch size of 4).
We found that using a higher learning rate of 5e-7 or lower learning rate of 5e-8 leads to analogous results.
As for the KL coefficient, for DPO we set it to 0.1, in line with \citet{rafailov2023direct,tajwar2024preference,xu2024dpo,dubey2024llama}, and for IPO we set it to 0.01, similarly to the experiments of \cref{sec:catastrophic}.

\textbf{Hardware.}
Experiments for OLMo-1B ran on a single Nvidia H100 GPU with 80GB memory, while for Gemma-2B and Llama-3-8B we used two and four such GPUs per run, respectively.

\subsection{Unintentional Unalignment in Direct Preference Learning (\cref{sec:unalignment})}
\label{app:experiments:details:unalignment}

\textbf{Data.}
We used the “base" portion of SORRY-Bench, which contains 450 prompts considered unsafe.
We filtered out 15 samples that did not have a “gold"  human labeled refusal or non-refusal response, and split the remaining samples into a training and test sets using a 85\%/15\% split.
When generating candidate responses, we used a temperature of 1 and set the maximum generated tokens to 512 (we did not use nucleus or top-k sampling).
For creating the “gold" preference dataset, we took the human labeled responses from SORRY-Bench, which were generated by a diverse set of models.
Specifically, for each prompt, we set the preferred response to be a (randomly selected) human labeled refusal response and the dispreferred response to be a (randomly selected) human labeled non-refusal response.
Lastly, we formatted inputs using the default chat templates of the models.

\textbf{Training.}
We ran one epoch of DPO (or IPO) using the RMSProp optimizer with batch size of 32 (emulated via 8 gradient accumulation steps with a batch size of 4).
We set the KL coefficient for DPO to 0.1, in line with \citet{rafailov2023direct,tajwar2024preference,xu2024dpo,dubey2024llama}, and for IPO to 0.01, as in the experiments of \cref{sec:catastrophic,sec:ches}.

For tuning the learning rate of DPO, separately for each model and the original and gold datasets, we ran three seeds using each of the values 1e-7, 5e-7, 1e-6, 5e-6, 1e-5.
We chose the largest learning rate that led to stable training, \ie~for which the training loss after one epoch is lower than the initial training loss.
For both Gemma-2B-IT and Llama-3-8B-Instruct, on the original datasets the learning rate was chosen accordingly to be 5e-6, and on the gold datasets to be 1e-6.
We used the same learning rates for IPO.
When running experiments over the filtered datasets, the learning rate was set to 5e-6, \ie~to be the same as in the experiments over the original (unfiltered) datasets.

For experiments with an SFT regularization term, we set the learning rate to 5e-6 and tuned the SFT term coefficient.
For DPO and each of the models, we ran three seeds using the values 0.01, 0.1, and 1, and chose the value that led to the highest mean refusal rate over the training set.
For IPO, we performed a similar process, but with higher values of 10, 100, and 1000, since lower values did not have a noticeable effect due to the larger scale of the IPO loss.
The coefficients chosen for Llama-3-8B-Instruct were 0.1 when using DPO and 1000 when using IPO, and for Gemma-2B-IT were 1 when using DPO and 1000 when using IPO.

\textbf{Refusal rate evaluation.}
For evaluating refusal rates, we used judge model and default generation hyperparameters from \citet{xie2024sorry}.
Specifically, the temperature was set to 0.7 and the maximal number of new tokens was 512 (nucleus or top-k sampling were not used).

\textbf{Hardware.}
Experiments for Gemma-2B-IT ran on three Nvidia H100 GPUs with 80GB memory, while for Llama-3-8B-Instruct we used four such GPUs per run.

\clearpage

\begin{table}[H]
	\vspace{-2mm}
	\begin{center}
		\small
		\begin{tabular}{lccccc}
			\toprule
			& & & &  \multicolumn{2}{c}{\textbf{Tokens Increasing Most in Probability}} \\
			\cmidrule(lr){5-6}
			\textbf{Model} & $\yw$ & $\yl$ & $\prob_\theta (\yw | \xbf)$ \textbf{Decrease} & \textbf{\benigngreen{Benign}} & \textbf{\catastrophicred{Catastrophic}} \\
			\midrule\\[-0.95em]
			\multirow{2}{*}{OLMo-1B} & Yes & No &  $0.15$ \, ($0.89 \to 0.74$) & \_Yes, \_yes & --- \\[0.15em]
			& No & Never & $0.13$ \, ($0.98 \to 0.85$) & \_No & --- \\[0.05em]
			\midrule\\[-0.95em]
			\multirow{2}{*}{Gemma-2B} & Yes & No & $0.58$ \, ($0.86 \to 0.28$) & \_Yes, \_yes & Something, something \\[0.15em]
			& No & Never & $0.10$ \, ($0.46 \to 0.36$) & no & Yes, yes \\[0.05em]
			\midrule\\[-0.95em]
			\multirow{2}{*}{Llama-3-8B} & Yes & No & $0.84$ \, ($0.94 \to 0.10$) & \_Yes, \_yes, yes & --- \\[0.15em]
			& Sure & Yes & $0.99$ \, ($0.99 \to 0.00$)  & sure, \_certain & --- \\[0.05em]             
			\bottomrule
		\end{tabular}
	\end{center}
	\vspace{-1mm}
	\caption{
		\textbf{Likelihood displacement can be catastrophic, even  when training on a single prompt with single token responses.}
		Reported are the results of an experiment analogous to that of \cref{table:persona_single_example_post_sft}, in which models did not undergo an initial SFT phase before training via DPO.
		For further details, see caption of \cref{table:persona_single_example_post_sft}.
	}
	\label{table:persona_single_example_base}
\end{table}

\begin{table}[H]
	\vspace{-2mm}
	\begin{center}
		\small
		\begin{tabular}{lccccc}
			\toprule
			& & & &  \multicolumn{2}{c}{\textbf{Tokens Increasing Most in Probability}} \\
			\cmidrule(lr){5-6}
			\textbf{Model} & $\yw$ & $\yl$ & $\prob_\theta (\yw | \xbf)$ \textbf{Decrease} & \textbf{\benigngreen{Benign}} & \textbf{\catastrophicred{Catastrophic}} \\
			\midrule\\[-0.95em]
			\multirow{2}{*}{OLMo-1B} & Yes & No &  $0.15$ \, ($0.89 \to 0.74$) & \_Yes, \_yes, Certainly & --- \\[0.15em]
			& No & Never & $0.87$ \, ($0.88 \to 0.01$) & \_no & Yes, Sure \\[0.05em]
			\midrule\\[-0.95em]
			\multirow{2}{*}{Gemma-2B} & Yes & No & $0.01$ \, ($0.07 \to 0.06$) & Yeah & Perhaps \\[0.15em]
			& No & Never & $0.03$ \, ($0.62 \to 0.59$) & no & Yeah, Sure \\[0.05em]
			\midrule\\[-0.95em]
			\multirow{2}{*}{Llama-3-8B} & Yes & No & $0.04$ \, ($0.99 \to 0.95$) & \_Yes, \_yes & --- \\[0.15em]
			& Sure & Yes & $0.25$ \, ($0.91 \to 0.66$)  & Yeah, sure & Maybe, Perhaps \\[0.05em]             
			\bottomrule
		\end{tabular}
	\end{center}
	\vspace{-1mm}
	\caption{
		\textbf{Likelihood displacement can be catastrophic, even  when training on a single prompt with single token responses.}
		Reported are the results of an experiment analogous to that of \cref{table:persona_single_example_post_sft}, using IPO instead of DPO.
		For further details, see caption of \cref{table:persona_single_example_post_sft}.
	}
	\label{table:persona_single_example_post_sft_ipo}
\end{table}

\begin{table}[H]
	\small
	\centering
	\begin{tabular}{ccccccc}
		\toprule
		\multicolumn{7}{c}{\textbf{OLMo-1B (DPO)}}\\
		\midrule
		\multirow{2}{*}{Training Step} & \multicolumn{3}{c}{$\mathbf{y}^+ =$ Yes \& $\mathbf{y}^- =$ No} & \multicolumn{3}{c}{$\mathbf{y}^+ =$ No \& $\mathbf{y}^- =$ Never}\\
		\cmidrule(lr){2-4}\cmidrule(lr){5-7}
		& Token & Probability Increase & Count & Token & Probability Increase & Count\\
		\midrule
		\multirow{5}{*}{5} & Yes & $8.7 \times 10^{-1}$ & 8/8 & Yes & $4.0 \times 10^{-1}$ & 8/8 \\
		& \_yes & $3.2 \times 10^{-3}$ & 8/8 & \_Yes & $1.8 \times 10^{-1}$ & 5/8 \\
		& \_Yes & $3.7 \times 10^{-2}$ & 8/8 & No & $2.7 \times 10^{-1}$ & 4/8 \\
		&  --- & --- & --- & \_yes & $3.0 \times 10^{-1}$ & 4/8 \\
		&  --- & --- & --- & \_No & $3.7 \times 10^{-2}$ & 3/8 \\
		\midrule
		\multirow{4}{*}{25} & Yes & $4.2 \times 10^{-1}$ & 8/8 & \_no & $9.0 \times 10^{-1}$ & 8/8 \\
		& \_yes & $7.9 \times 10^{-2}$ & 8/8 & \_No & $8.9 \times 10^{-2}$ & 8/8 \\
		& \_Yes & $4.1 \times 10^{-1}$ & 8/8 & no & $2.1 \times 10^{-4}$ & 7/8 \\
		&  --- & --- & --- & \_coronal & $-1.7 \times 10^{-15}$ & 1/8 \\
		\midrule
		\multirow{4}{*}{100} & Yes & $1.8 \times 10^{-1}$ & 8/8 & \_no & $4.0 \times 10^{-1}$ & 8/8 \\
		& \_yes & $1.3 \times 10^{-1}$ & 8/8 & \_No & $4.4 \times 10^{-1}$ & 8/8 \\
		& \_Yes & $6.0 \times 10^{-1}$ & 8/8 & no & $3.2 \times 10^{-3}$ & 7/8 \\
		&  --- & --- & --- & No & $1.7 \times 10^{-2}$ & 1/8 \\
		\bottomrule
	\end{tabular}
	\caption{
		For the experiments of \cref{table:persona_single_example_post_sft} with the OLMo-1B model, included are all tokens from the top three tokens increasing most in probability until training steps 5, 25, and 100, across runs varying in the prompt used for training.
		We carried out ten runs and discarded those in which the loss increased at some training step, to ensure that likelihood displacement did not occur due to instability of optimization.
		We further report the mean probability increase and the number of runs in which the token was in the top three at a given time~step.
	}
	\label{tab:persona_toks_olmo_1b_dpo}
\end{table}

\begin{table}[H]
	\small
	\centering
	\begin{tabular}{ccccccc}
		\toprule
		\multicolumn{7}{c}{\textbf{Gemma-2B (DPO)}}\\
		\midrule
		\multirow{2}{*}{Training Step} & \multicolumn{3}{c}{$\mathbf{y}^+ =$ Yes \& $\mathbf{y}^- =$ No} & \multicolumn{3}{c}{$\mathbf{y}^+ =$ No \& $\mathbf{y}^- =$ Never}\\
		\cmidrule(lr){2-4}\cmidrule(lr){5-7}
		& Token & Probability Increase & Count & Token & Probability Increase & Count\\
		\midrule
		\multirow{10}{*}{5} & Yes & $8.8 \times 10^{-1}$ & 10/10 & No & $8.2 \times 10^{-1}$ & 10/10 \\
		& YES & $2.8 \times 10^{-3}$ & 10/10 & no & $2.1 \times 10^{-3}$ & 9/10 \\
		& yes & $5.3 \times 10^{-4}$ & 5/10 & \_No & $2.1 \times 10^{-4}$ & 3/10 \\
		& \_Yes & $7.5 \times 10^{-5}$ & 3/10 & yes & $4.3 \times 10^{-3}$ & 2/10 \\
		& Yeah & $2.6 \times 10^{-2}$ & 1/10 & Yeah & $1.3 \times 10^{-1}$ & 1/10 \\
		& Yep & $4.4 \times 10^{-4}$ & 1/10 & \_Polite & $1.2 \times 10^{-9}$ & 1/10 \\
		&  --- & --- & --- & kshake & $4.3 \times 10^{-13}$ & 1/10 \\
		&  --- & --- & --- & \_potrebbero & $3.6 \times 10^{-5}$ & 1/10 \\
		&  --- & --- & --- & \_buoni & $7.6 \times 10^{-11}$ & 1/10 \\
		&  --- & --- & --- & ( & $1.6 \times 10^{-4}$ & 1/10 \\
		\midrule
		\multirow{6}{*}{25} & Yes & $9.3 \times 10^{-1}$ & 10/10 & No & $8.6 \times 10^{-1}$ & 10/10 \\
		& \_Yes & $8.5 \times 10^{-3}$ & 9/10 & no & $6.1 \times 10^{-3}$ & 8/10 \\
		& YES & $2.5 \times 10^{-3}$ & 8/10 & \_No & $8.8 \times 10^{-4}$ & 8/10 \\
		& yes & $2.3 \times 10^{-3}$ & 2/10 & \_no & $6.7 \times 10^{-5}$ & 2/10 \\
		& \_yes & $7.7 \times 10^{-3}$ & 1/10 & \_balenciaga & $1.9 \times 10^{-22}$ & 1/10 \\
		&  --- & --- & --- & \_babi & $-1.4 \times 10^{-29}$ & 1/10 \\
		\midrule
		\multirow{4}{*}{100} & Yes & $7.1 \times 10^{-1}$ & 10/10 & no & $1.5 \times 10^{-2}$ & 10/10 \\
		& \_Yes & $1.9 \times 10^{-1}$ & 10/10 & No & $8.4 \times 10^{-1}$ & 10/10 \\
		& \_yes & $3.4 \times 10^{-2}$ & 10/10 & \_No & $5.6 \times 10^{-3}$ & 8/10 \\
		&  --- & --- & --- & \_no & $3.6 \times 10^{-3}$ & 2/10 \\
		\bottomrule
	\end{tabular}
	\caption{
		For the experiments of \cref{table:persona_single_example_post_sft} with the Gemma-2B model, included are all tokens from the top three tokens increasing most in probability until training steps 5, 25, and 100, across runs varying in the prompt used for training.
		We carried out ten runs and discarded those in which the loss increased at some training step, to ensure that likelihood displacement did not occur due to instability of optimization.
		We further report the mean probability increase and the number of runs in which the token was in the top three at a given time~step.
	}
	\label{tab:persona_toks_gemma_2b_dpo}
\end{table}

\begin{table}[H]
	\small
	\centering
	\begin{tabular}{ccccccc}
		\toprule
		\multicolumn{7}{c}{\textbf{Llama-3-8B (DPO)}}\\
		\midrule
		\multirow{2}{*}{Training Step} & \multicolumn{3}{c}{$\mathbf{y}^+ =$ Yes \& $\mathbf{y}^- =$ No} & \multicolumn{3}{c}{$\mathbf{y}^+ =$ Sure \& $\mathbf{y}^- =$ Yes}\\
		\cmidrule(lr){2-4}\cmidrule(lr){5-7}
		& Token & Probability Increase & Count & Token & Probability Increase & Count\\
		\midrule
		\multirow{9}{*}{5} & Yes & $5.3 \times 10^{-1}$ & 10/10 & Sure & $7.9 \times 10^{-1}$ & 4/5 \\
		& \_Yes & $7.5 \times 10^{-5}$ & 9/10 & "N & $9.0 \times 10^{-3}$ & 3/5 \\
		& \_yes & $1.7 \times 10^{-5}$ & 6/10 & N & $1.8 \times 10^{-2}$ & 2/5 \\
		& yes & $2.9 \times 10^{-3}$ & 4/10 & " & $2.2 \times 10^{-2}$ & 1/5 \\
		& "Yes & $8.1 \times 10^{-5}$ & 1/10 & No & $1.1 \times 10^{-1}$ & 1/5 \\
		&  --- & --- & --- & Maybe & $2.3 \times 10^{-1}$ & 1/5 \\
		&  --- & --- & --- & Never & $1.5 \times 10^{-1}$ & 1/5 \\
		&  --- & --- & --- & Perhaps & $3.4 \times 10^{-1}$ & 1/5 \\
		&  --- & --- & --- & Pretty & $1.2 \times 10^{-5}$ & 1/5 \\
		\midrule
		\multirow{7}{*}{25} & yes & $1.3 \times 10^{-1}$ & 10/10 & Sure & $8.5 \times 10^{-1}$ & 5/5 \\
		& \_yes & $2.1 \times 10^{-1}$ & 10/10 & sure & $1.0 \times 10^{-2}$ & 4/5 \\
		& Yes & $2.4 \times 10^{-1}$ & 7/10 & SURE & $7.1 \times 10^{-4}$ & 2/5 \\
		& \_Yes & $4.2 \times 10^{-2}$ & 3/10 & " & $6.8 \times 10^{-3}$ & 1/5 \\
		&  --- & --- & --- & \_Sure & $1.4 \times 10^{-4}$ & 1/5 \\
		&  --- & --- & --- & Sur & $4.1 \times 10^{-3}$ & 1/5 \\
		&  --- & --- & --- & Arkhiv & $-1.3 \times 10^{-16}$ & 1/5 \\
		\midrule
		\multirow{6}{*}{100} & \_Yes & $2.2 \times 10^{-2}$ & 10/10 & Sure & $8.6 \times 10^{-1}$ & 5/5 \\
		& yes & $2.6 \times 10^{-1}$ & 10/10 & sure & $1.3 \times 10^{-2}$ & 4/5 \\
		& \_yes & $6.9 \times 10^{-1}$ & 10/10 & \_surely & $5.8 \times 10^{-5}$ & 2/5 \\
		&  --- & --- & --- & \_Sure & $1.6 \times 10^{-4}$ & 2/5 \\
		&  --- & --- & --- & \_Surely & $2.4 \times 10^{-5}$ & 1/5 \\
		&  --- & --- & --- & Arkhiv & $-1.3 \times 10^{-16}$ & 1/5 \\
		\bottomrule
	\end{tabular}
	\caption{
		For the experiments of \cref{table:persona_single_example_post_sft} with the Llama-3-8B model, included are all tokens from the top three tokens increasing most in probability until training steps 5, 25, and 100, across runs varying in the prompt used for training.
		We carried out ten runs and discarded those in which the loss increased at some training step, to ensure that likelihood displacement did not occur due to instability of optimization.
		We further report the mean probability increase and the number of runs in which the token was in the top three at a given time~step.
	}
	\label{tab:persona_toks_llama_3_8b_dpo}
\end{table}

\begin{table}[H]
	\small
	\centering
	\begin{tabular}{ccccccc}
		\toprule
		\multicolumn{7}{c}{\textbf{OLMo-1B (DPO on base model)}}\\
		\midrule
		\multirow{2}{*}{Training Step} & \multicolumn{3}{c}{$\mathbf{y}^+ =$ Yes \& $\mathbf{y}^- =$ No} & \multicolumn{3}{c}{$\mathbf{y}^+ =$ No \& $\mathbf{y}^- =$ Never}\\
		\cmidrule(lr){2-4}\cmidrule(lr){5-7}
		& Token & Probability Increase & Count & Token & Probability Increase & Count\\
		\midrule
		\multirow{5}{*}{5} & Yes & $9.8 \times 10^{-1}$ & 9/9 & \_No & $5.3 \times 10^{-3}$ & 10/10 \\
		& \_Yes & $1.1 \times 10^{-3}$ & 6/9 & No & $9.8 \times 10^{-1}$ & 10/10 \\
		& YES & $4.0 \times 10^{-3}$ & 5/9 & NO & $2.0 \times 10^{-3}$ & 9/10 \\
		& yes & $3.4 \times 10^{-3}$ & 4/9 & \_no & $1.6 \times 10^{-5}$ & 1/10 \\
		& \_yes & $6.1 \times 10^{-4}$ & 3/9 &  --- & --- & --- \\
		\midrule
		\multirow{4}{*}{25} & Yes & $9.8 \times 10^{-1}$ & 9/9 & \_No & $3.3 \times 10^{-2}$ & 10/10 \\
		& \_yes & $7.0 \times 10^{-3}$ & 9/9 & No & $9.6 \times 10^{-1}$ & 10/10 \\
		& \_Yes & $4.3 \times 10^{-3}$ & 9/9 & \_no & $4.3 \times 10^{-5}$ & 8/10 \\
		&  --- & --- & --- & no & $5.6 \times 10^{-5}$ & 2/10 \\
		\midrule
		\multirow{4}{*}{100} & Yes & $9.3 \times 10^{-1}$ & 9/9 & \_No & $1.3 \times 10^{-1}$ & 10/10 \\
		& \_yes & $4.0 \times 10^{-2}$ & 9/9 & No & $8.6 \times 10^{-1}$ & 10/10 \\
		& \_Yes & $2.1 \times 10^{-2}$ & 9/9 & no & $2.2 \times 10^{-4}$ & 7/10 \\
		&  --- & --- & --- & \_no & $1.1 \times 10^{-4}$ & 3/10 \\
		\bottomrule
	\end{tabular}
	\caption{
		For the experiments of \cref{table:persona_single_example_base} with the OLMo-1B model, included are all tokens from the top three tokens increasing most in probability until training steps 5, 25, and 100, across runs varying in the prompt used for training.
		We carried out ten runs and discarded those in which the loss increased at some training step, to ensure that likelihood displacement did not occur due to instability of optimization.
		We further report the mean probability increase and the number of runs in which the token was in the top three at a given time~step.
	}
	\label{tab:persona_toks_olmo_1b_dpo_on_base_model}
\end{table}

\begin{table}[H]
	\small
	\centering
	\begin{tabular}{ccccccc}
		\toprule
		\multicolumn{7}{c}{\textbf{Gemma-2B (DPO on base model)}}\\
		\midrule
		\multirow{2}{*}{Training Step} & \multicolumn{3}{c}{$\mathbf{y}^+ =$ Yes \& $\mathbf{y}^- =$ No} & \multicolumn{3}{c}{$\mathbf{y}^+ =$ No \& $\mathbf{y}^- =$ Never}\\
		\cmidrule(lr){2-4}\cmidrule(lr){5-7}
		& Token & Probability Increase & Count & Token & Probability Increase & Count\\
		\midrule
		\multirow{10}{*}{5} & Yes & $8.9 \times 10^{-1}$ & 7/9 & No & $2.9 \times 10^{-1}$ & 8/10 \\
		& YES & $7.9 \times 10^{-2}$ & 7/9 & Yes & $4.0 \times 10^{-1}$ & 7/10 \\
		& Something & $3.3 \times 10^{-1}$ & 4/9 & no & $3.7 \times 10^{-1}$ & 4/10 \\
		& yes & $9.5 \times 10^{-3}$ & 3/9 & yes & $6.6 \times 10^{-2}$ & 3/10 \\
		& something & $2.3 \times 10^{-1}$ & 3/9 & or & $1.0 \times 10^{-1}$ & 2/10 \\
		& \_something & $3.4 \times 10^{-4}$ & 1/9 & NO & $2.3 \times 10^{-2}$ & 2/10 \\
		& \_territo & $3.0 \times 10^{-13}$ & 1/9 & \$\\ & $9.9 \times 10^{-2}$ & 1/10 \\
		& \_paradigma & $2.5 \times 10^{-16}$ & 1/9 & Or & $1.2 \times 10^{-1}$ & 1/10 \\
		&  --- & --- & --- & Would & $2.2 \times 10^{-2}$ & 1/10 \\
		&  --- & --- & --- & Si & $5.1 \times 10^{-2}$ & 1/10 \\
		\midrule
		\multirow{8}{*}{25} & Yes & $8.9 \times 10^{-1}$ & 9/9 & No & $9.4 \times 10^{-1}$ & 10/10 \\
		& yes & $1.0 \times 10^{-1}$ & 7/9 & no & $7.3 \times 10^{-2}$ & 7/10 \\
		& \_yes & $2.6 \times 10^{-3}$ & 6/9 & \_lele & $-5.0 \times 10^{-24}$ & 4/10 \\
		& YES & $1.6 \times 10^{-2}$ & 3/9 & \_babi & $-3.9 \times 10^{-24}$ & 3/10 \\
		& \_Yes & $2.6 \times 10^{-2}$ & 1/9 & \_perez & $-1.9 \times 10^{-23}$ & 2/10 \\
		& \_babi & $-9.6 \times 10^{-24}$ & 1/9 & \_puto & $-9.6 \times 10^{-24}$ & 2/10 \\
		&  --- & --- & --- & NO & $2.0 \times 10^{-4}$ & 1/10 \\
		&  --- & --- & --- & \_nuoc & $-3.4 \times 10^{-26}$ & 1/10 \\
		\midrule
		\multirow{8}{*}{100} & Yes & $4.6 \times 10^{-1}$ & 9/9 & No & $9.5 \times 10^{-1}$ & 10/10 \\
		& \_yes & $2.4 \times 10^{-1}$ & 9/9 & no & $7.0 \times 10^{-2}$ & 7/10 \\
		& yes & $2.4 \times 10^{-1}$ & 8/9 & \_no & $5.4 \times 10^{-7}$ & 3/10 \\
		& \_Yes & $5.5 \times 10^{-1}$ & 1/9 & \_babi & $-3.9 \times 10^{-24}$ & 3/10 \\
		&  --- & --- & --- & \_lele & $-6.4 \times 10^{-24}$ & 3/10 \\
		&  --- & --- & --- & \_nuoc & $-3.2 \times 10^{-24}$ & 2/10 \\
		&  --- & --- & --- & \_perez & $-2.1 \times 10^{-23}$ & 1/10 \\
		&  --- & --- & --- & \_puto & $-1.3 \times 10^{-23}$ & 1/10 \\
		\bottomrule
	\end{tabular}
	\caption{
		For the experiments of \cref{table:persona_single_example_base} with the Gemma-2B model, included are all tokens from the top three tokens increasing most in probability until training steps 5, 25, and 100, across runs varying in the prompt used for training.
		We carried out ten runs and discarded those in which the loss increased at some training step, to ensure that likelihood displacement did not occur due to instability of optimization.
		We further report the mean probability increase and the number of runs in which the token was in the top three at a given time~step.
	}
	\label{tab:persona_toks_gemma_2b_dpo_on_base_model}
\end{table}

\begin{table}[H]
	\small
	\centering
	\begin{tabular}{ccccccc}
		\toprule
		\multicolumn{7}{c}{\textbf{Llama-3-8B (DPO on base model)}}\\
		\midrule
		\multirow{2}{*}{Training Step} & \multicolumn{3}{c}{$\mathbf{y}^+ =$ Yes \& $\mathbf{y}^- =$ No} & \multicolumn{3}{c}{$\mathbf{y}^+ =$ Sure \& $\mathbf{y}^- =$ Yes}\\
		\cmidrule(lr){2-4}\cmidrule(lr){5-7}
		& Token & Probability Increase & Count & Token & Probability Increase & Count\\
		\midrule
		\multirow{6}{*}{5} & Yes & $6.4 \times 10^{-1}$ & 7/7 & Sure & $8.8 \times 10^{-1}$ & 5/5 \\
		& yes & $3.5 \times 10^{-2}$ & 6/7 & sure & $6.0 \times 10^{-4}$ & 4/5 \\
		& "Yes & $2.0 \times 10^{-1}$ & 5/7 & \_Sure & $9.2 \times 10^{-6}$ & 3/5 \\
		& YES & $1.8 \times 10^{-2}$ & 2/7 & "I & $2.4 \times 10^{-1}$ & 1/5 \\
		& Is & $2.7 \times 10^{-2}$ & 1/7 & "If & $5.0 \times 10^{-2}$ & 1/5 \\
		&  --- & --- & --- & Lik & $5.2 \times 10^{-5}$ & 1/5 \\
		\midrule
		\multirow{4}{*}{25} & Yes & $4.7 \times 10^{-1}$ & 7/7 & \_certain & $9.3 \times 10^{-1}$ & 5/5 \\
		& yes & $4.3 \times 10^{-1}$ & 7/7 & \_Certain & $5.9 \times 10^{-2}$ & 5/5 \\
		& \_yes & $7.2 \times 10^{-2}$ & 5/7 & Certain & $7.4 \times 10^{-3}$ & 5/5 \\
		& \_Yes & $4.4 \times 10^{-2}$ & 2/7 &  --- & --- & --- \\
		\midrule
		\multirow{5}{*}{100} & yes & $5.8 \times 10^{-1}$ & 7/7 & sure & $5.1 \times 10^{-3}$ & 5/5 \\
		& \_yes & $2.7 \times 10^{-1}$ & 7/7 & Sure & $9.9 \times 10^{-1}$ & 5/5 \\
		& Yes & $1.2 \times 10^{-1}$ & 5/7 & \_sure & $8.8 \times 10^{-4}$ & 2/5 \\
		& \_Yes & $1.0 \times 10^{-1}$ & 2/7 & \_certain & $3.9 \times 10^{-3}$ & 2/5 \\
		&  --- & --- & --- & \_Sure & $1.1 \times 10^{-4}$ & 1/5 \\
		\bottomrule
	\end{tabular}
	\caption{
		For the experiments of \cref{table:persona_single_example_base} with the Llama-3-8B model, included are all tokens from the top three tokens increasing most in probability until training steps 5, 25, and 100, across runs varying in the prompt used for training.
		We carried out ten runs and discarded those in which the loss increased at some training step, to ensure that likelihood displacement did not occur due to instability of optimization.
		We further report the mean probability increase and the number of runs in which the token was in the top three at a given time~step.
	}
	\label{tab:persona_toks_llama_3_8b_dpo_on_base_model}
\end{table}

\begin{table}[H]
	\small
	\centering
	\begin{tabular}{ccccccc}
		\toprule
		\multicolumn{7}{c}{\textbf{OLMo-1B (IPO)}}\\
		\midrule
		\multirow{2}{*}{Training Step} & \multicolumn{3}{c}{$\mathbf{y}^+ =$ Yes \& $\mathbf{y}^- =$ No} & \multicolumn{3}{c}{$\mathbf{y}^+ =$ No \& $\mathbf{y}^- =$ Never}\\
		\cmidrule(lr){2-4}\cmidrule(lr){5-7}
		& Token & Probability Increase & Count & Token & Probability Increase & Count\\
		\midrule
		\multirow{4}{*}{5} & Yes & $3.7 \times 10^{-2}$ & 9/10 & No & $1.3 \times 10^{-1}$ & 10/10 \\
		& Yeah & $1.3 \times 10^{-2}$ & 9/10 & Yes & $5.1 \times 10^{-2}$ & 9/10 \\
		& Certainly & $4.1 \times 10^{-2}$ & 9/10 & Absolutely & $4.3 \times 10^{-2}$ & 6/10 \\
		& Indeed & $9.2 \times 10^{-3}$ & 3/10 & Sure & $3.9 \times 10^{-2}$ & 5/10 \\
		\midrule
		\multirow{7}{*}{25} & Yes & $2.6 \times 10^{-1}$ & 10/10 & Yes & $5.0 \times 10^{-1}$ & 10/10 \\
		& Yeah & $2.9 \times 10^{-2}$ & 7/10 & No & $1.5 \times 10^{-1}$ & 9/10 \\
		& Sure & $1.1 \times 10^{-1}$ & 4/10 & \_Yes & $1.5 \times 10^{-2}$ & 6/10 \\
		& Certainly & $6.0 \times 10^{-2}$ & 4/10 & \_No & $2.0 \times 10^{-2}$ & 3/10 \\
		& Indeed & $1.3 \times 10^{-2}$ & 3/10 & Yeah & $1.1 \times 10^{-2}$ & 2/10 \\
		& \_Yes & $3.3 \times 10^{-3}$ & 1/10 &  --- & --- & --- \\
		& \_Sure & $1.7 \times 10^{-3}$ & 1/10 &  --- & --- & --- \\
		\midrule
		\multirow{6}{*}{100} & Yes & $7.9 \times 10^{-1}$ & 10/10 & \_no & $9.4 \times 10^{-1}$ & 10/10 \\
		& \_yes & $2.7 \times 10^{-2}$ & 10/10 & \_No & $6.0 \times 10^{-2}$ & 10/10 \\
		& \_Yes & $9.6 \times 10^{-2}$ & 10/10 & \_homepage & $-1.1 \times 10^{-15}$ & 5/10 \\
		&  --- & --- & --- & \_coronal & $-1.4 \times 10^{-15}$ & 3/10 \\
		&  --- & --- & --- & \_yes & $4.9 \times 10^{-8}$ & 1/10 \\
		&  --- & --- & --- & \_NO & $5.6 \times 10^{-6}$ & 1/10 \\
		\bottomrule
	\end{tabular}
	\caption{
		For the experiments of \cref{table:persona_single_example_post_sft_ipo} with the OLMo-1B model, included are all tokens from the top three tokens increasing most in probability until training steps 5, 25, and 100, across runs varying in the prompt used for training.
		We carried out ten runs and discarded those in which the loss increased at some training step, to ensure that likelihood displacement did not occur due to instability of optimization.
		We further report the mean probability increase and the number of runs in which the token was in the top three at a given time~step.
	}
	\label{tab:persona_toks_olmo_1b_ipo}
\end{table}

\begin{table}[H]
	\small
	\centering
	\begin{tabular}{ccccccc}
		\toprule
		\multicolumn{7}{c}{\textbf{Gemma-2B (IPO)}}\\
		\midrule
		\multirow{2}{*}{Training Step} & \multicolumn{3}{c}{$\mathbf{y}^+ =$ Yes \& $\mathbf{y}^- =$ No} & \multicolumn{3}{c}{$\mathbf{y}^+ =$ No \& $\mathbf{y}^- =$ Never}\\
		\cmidrule(lr){2-4}\cmidrule(lr){5-7}
		& Token & Probability Increase & Count & Token & Probability Increase & Count\\
		\midrule
		\multirow{8}{*}{5} & Yes & $7.2 \times 10^{-2}$ & 10/10 & No & $1.2 \times 10^{-1}$ & 10/10 \\
		& Yeah & $1.3 \times 10^{-1}$ & 10/10 & Yeah & $3.2 \times 10^{-2}$ & 8/10 \\
		& Perhaps & $8.1 \times 10^{-3}$ & 3/10 & Sure & $2.1 \times 10^{-2}$ & 7/10 \\
		& Sure & $2.4 \times 10^{-2}$ & 2/10 & Maybe & $3.5 \times 10^{-2}$ & 2/10 \\
		& Absolutely & $3.3 \times 10^{-2}$ & 2/10 & no & $3.0 \times 10^{-4}$ & 1/10 \\
		& YES & $3.4 \times 10^{-5}$ & 1/10 & maybe & $3.3 \times 10^{-3}$ & 1/10 \\
		& Yep & $7.8 \times 10^{-4}$ & 1/10 & Possibly & $6.5 \times 10^{-3}$ & 1/10 \\
		& Something & $5.9 \times 10^{-4}$ & 1/10 &  --- & --- & --- \\
		\midrule
		\multirow{9}{*}{25} & Yes & $4.4 \times 10^{-1}$ & 10/10 & No & $5.3 \times 10^{-1}$ & 9/10 \\
		& Yeah & $3.1 \times 10^{-1}$ & 10/10 & no & $1.8 \times 10^{-3}$ & 6/10 \\
		& YES & $2.9 \times 10^{-3}$ & 3/10 & Yeah & $4.5 \times 10^{-1}$ & 6/10 \\
		& yeah & $1.1 \times 10^{-3}$ & 3/10 & \_No & $1.3 \times 10^{-4}$ & 3/10 \\
		& Yep & $5.0 \times 10^{-3}$ & 2/10 & Said & $7.8 \times 10^{-6}$ & 2/10 \\
		& Oui & $3.4 \times 10^{-4}$ & 2/10 & Yes & $8.9 \times 10^{-2}$ & 1/10 \\
		&  --- & --- & --- & \_Yeah & $2.2 \times 10^{-7}$ & 1/10 \\
		&  --- & --- & --- & Say & $1.7 \times 10^{-4}$ & 1/10 \\
		&  --- & --- & --- & DirPath & $9.0 \times 10^{-7}$ & 1/10 \\
		\midrule
		\multirow{5}{*}{100} & Yes & $9.1 \times 10^{-1}$ & 10/10 & no & $8.3 \times 10^{-3}$ & 10/10 \\
		& yes & $5.2 \times 10^{-3}$ & 8/10 & No & $8.5 \times 10^{-1}$ & 10/10 \\
		& YES & $4.0 \times 10^{-3}$ & 8/10 & \_No & $2.7 \times 10^{-4}$ & 10/10 \\
		& \_Yes & $1.4 \times 10^{-3}$ & 3/10 &  --- & --- & --- \\
		& \_yes & $7.1 \times 10^{-6}$ & 1/10 &  --- & --- & --- \\
		\bottomrule
	\end{tabular}
	\caption{
		For the experiments of \cref{table:persona_single_example_post_sft_ipo} with the Gemma-2B model, included are all tokens from the top three tokens increasing most in probability until training steps 5, 25, and 100, across runs varying in the prompt used for training.
		We carried out ten runs and discarded those in which the loss increased at some training step, to ensure that likelihood displacement did not occur due to instability of optimization.
		We further report the mean probability increase and the number of runs in which the token was in the top three at a given time~step.
	}
	\label{tab:persona_toks_gemma_2b_ipo}
\end{table}

\begin{table}[H]
	\small
	\centering
	\begin{tabular}{ccccccc}
		\toprule
		\multicolumn{7}{c}{\textbf{Llama-3-8B (IPO)}}\\
		\midrule
		\multirow{2}{*}{Training Step} & \multicolumn{3}{c}{$\mathbf{y}^+ =$ Yes \& $\mathbf{y}^- =$ No} & \multicolumn{3}{c}{$\mathbf{y}^+ =$ Sure \& $\mathbf{y}^- =$ Yes}\\
		\cmidrule(lr){2-4}\cmidrule(lr){5-7}
		& Token & Probability Increase & Count & Token & Probability Increase & Count\\
		\midrule
		\multirow{4}{*}{5} & Yes & $1.8 \times 10^{-1}$ & 10/10 & Yeah & $7.0 \times 10^{-2}$ & 7/7 \\
		& "Yes & $7.1 \times 10^{-4}$ & 10/10 & Sure & $3.2 \times 10^{-1}$ & 7/7 \\
		& yes & $1.0 \times 10^{-3}$ & 9/10 & Maybe & $2.1 \times 10^{-3}$ & 4/7 \\
		& Def & $7.0 \times 10^{-4}$ & 1/10 & Certainly & $7.7 \times 10^{-3}$ & 3/7 \\
		\midrule
		\multirow{7}{*}{25} & Yes & $5.0 \times 10^{-1}$ & 10/10 & Sure & $6.9 \times 10^{-1}$ & 7/7 \\
		& yes & $4.8 \times 10^{-3}$ & 10/10 & Maybe & $2.9 \times 10^{-2}$ & 5/7 \\
		& "Yes & $4.3 \times 10^{-3}$ & 5/10 & Perhaps & $1.1 \times 10^{-2}$ & 4/7 \\
		& \_Yes & $7.2 \times 10^{-5}$ & 4/10 & Y & $7.0 \times 10^{-2}$ & 2/7 \\
		& YES & $2.6 \times 10^{-3}$ & 1/10 & " & $6.5 \times 10^{-3}$ & 1/7 \\
		&  --- & --- & --- & E & $4.1 \times 10^{-2}$ & 1/7 \\
		&  --- & --- & --- & Never & $5.5 \times 10^{-3}$ & 1/7 \\
		\midrule
		\multirow{6}{*}{100} & Yes & $4.8 \times 10^{-1}$ & 10/10 & sure & $6.8 \times 10^{-3}$ & 7/7 \\
		& \_yes & $2.1 \times 10^{-2}$ & 10/10 & Sure & $8.8 \times 10^{-1}$ & 7/7 \\
		& \_Yes & $1.3 \times 10^{-2}$ & 5/10 & \_Surely & $4.8 \times 10^{-5}$ & 3/7 \\
		& yes & $2.4 \times 10^{-2}$ & 5/10 & \_Sure & $7.8 \times 10^{-5}$ & 2/7 \\
		&  --- & --- & --- & \_surely & $5.1 \times 10^{-5}$ & 1/7 \\
		&  --- & --- & --- & Sur & $9.8 \times 10^{-5}$ & 1/7 \\
		\bottomrule
	\end{tabular}
	\caption{
		For the experiments of \cref{table:persona_single_example_post_sft_ipo} with the Llama-3-8B model, included are all tokens from the top three tokens increasing most in probability until training steps 5, 25, and 100, across runs varying in the prompt used for training.
		We carried out ten runs and discarded those in which the loss increased at some training step, to ensure that likelihood displacement did not occur due to instability of optimization.
		We further report the mean probability increase and the number of runs in which the token was in the top three at a given time~step.
	}
	\label{tab:persona_toks_llama_3_8b_ipo}
\end{table}

\begin{table}[H]
	\vspace{-2mm}
	\begin{center}
		\small
		\begin{tabular}{lcccc}
			\toprule\\[-0.5em]
			\textbf{Model} & $\yw$ & $\yl$ & $\norm1{ \proj_{\unembed_{\yw}} \brk*{ \unembed_{\yw} - \unembed_{\yl} } }$ & $\norm1{ \proj_{\unembed^\perp_{\yw}} \brk*{ \unembed_{\yw} - \unembed_{\yl} } }$ \\
			\midrule\\[-0.95em]
			\multirow{2}{*}{OLMo-1B} & Yes & No &  $1.53$ & \boldmath{$2.01$}  \\[0.15em]
			& No & Never & $1.62$ & \boldmath{$2.26$}  \\[0.05em]
			\midrule\\[-0.95em]
			\multirow{2}{*}{Gemma-2B} & Yes & No & $0.94$ & \boldmath{$2.57$} \\[0.15em]
			& No & Never & $0.16$ & \boldmath{$3.14$} \\[0.05em]
			\midrule\\[-0.95em]
			\multirow{2}{*}{Llama-3-8B} & Yes & No & $0.49$ & \boldmath{$0.71$} \\[0.15em]
			& Sure & Yes & $0.67$  & \boldmath{$0.71$} \\[0.05em]             
			\bottomrule
		\end{tabular}
	\end{center}
	\vspace{-1mm}
	\caption{
		For each model and pair of preferred and dispreferred tokens $(\yw, \yl)$ from \cref{table:persona_single_example_post_sft}, we report the norm of the projection of $\unembed_{\yw} - \unembed_{\yl}$ onto $\unembed_{\yw}$ (second from right column), and the norm of the component of $\unembed_{\yw} - \unembed_{\yl}$ orthogonal to $\unembed_{\yw}$ (rightmost column).
		The norm of the orthogonal component is larger across the different models and preference pairs, in accordance with our theoretical explanation of why likelihood displacement can be catastrophic in the case of single token responses (\cref{sec:likelihood_dynamics}).
	}
	\label{table:persona_pref_vs_proj_direction_norm}
\end{table}

\begin{figure}[H]
	\centering
	\hspace*{-2mm}
	\includegraphics[width=\linewidth]{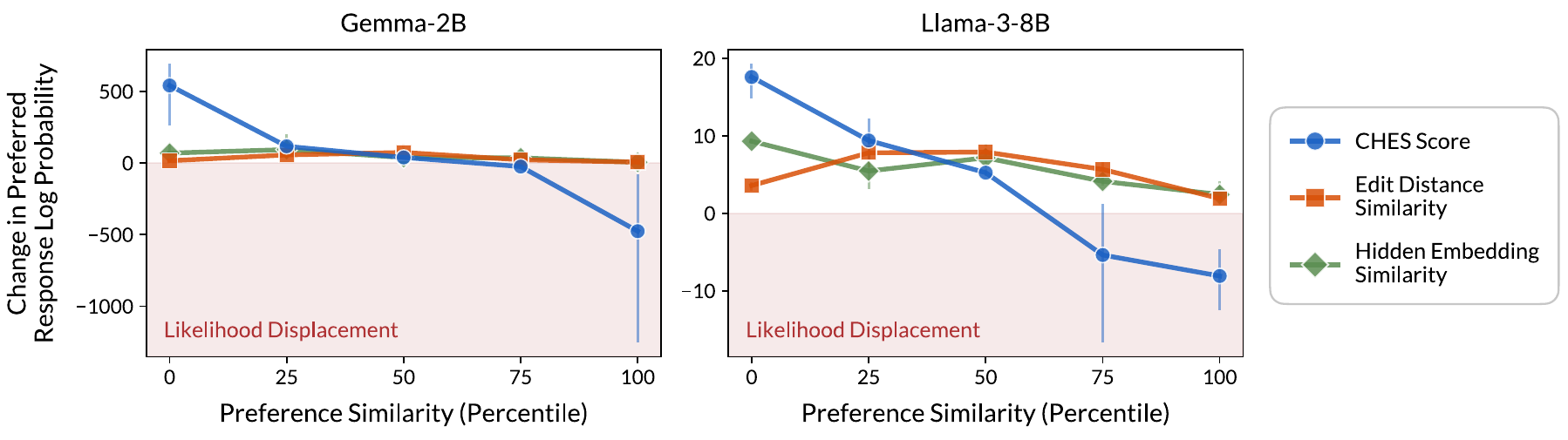}
	\vspace{-1.5mm}
	\caption{
		\textbf{CHES score (\cref{def:ches}) identifies which training samples contribute to likelihood displacement, whereas alternative similarity measures do not.}
		Reported are the results of an experiment analogous to that of \cref{fig:ultrafeedback_displacement_to_pref_sim_dpo}, over the AlpacaFarm dataset instead of UltraFeedback.
		See caption of \cref{fig:ultrafeedback_displacement_to_pref_sim_dpo} for further details.
	}
	\label{fig:alpacafarm_displacement_to_pref_sim_dpo}
\end{figure}

\begin{figure}[H]
	\centering
	\hspace*{-2mm}
	\includegraphics[width=\linewidth]{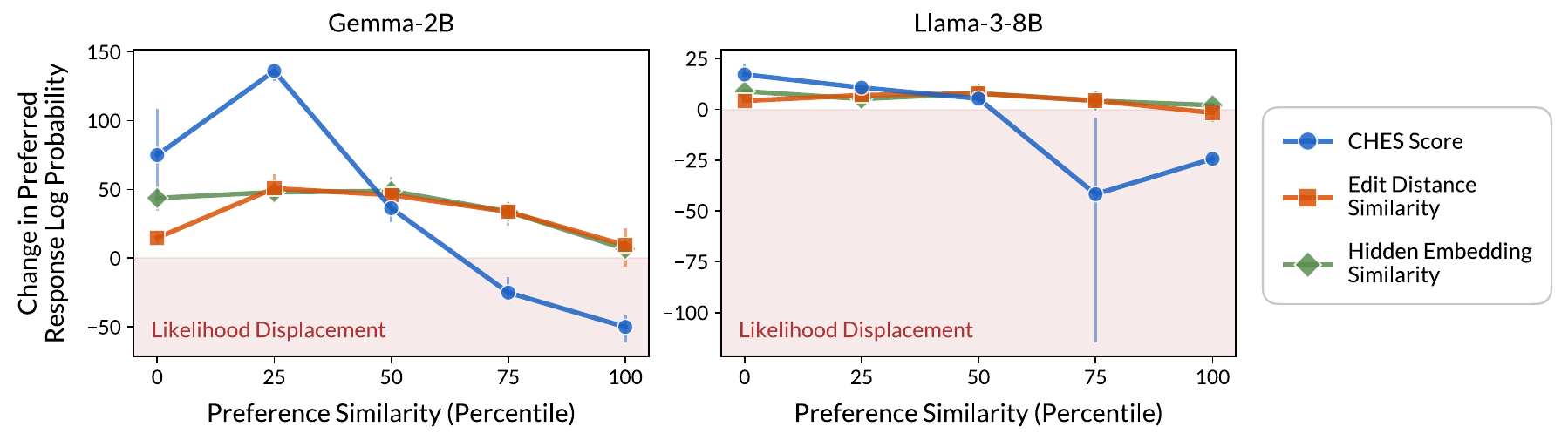}
	\vspace{-1.5mm}
	\caption{
		\textbf{CHES score (\cref{def:ches}) identifies which training samples contribute to likelihood displacement, whereas alternative similarity measures do not.}
		Reported are the results of an experiment analogous to that of \cref{fig:ultrafeedback_displacement_to_pref_sim_dpo}, in which the models were trained via IPO over the AlpacaFarm dataset, as opposed to via DPO over UltraFeedback.
		See caption of \cref{fig:ultrafeedback_displacement_to_pref_sim_dpo} for further details.
	}
	\label{fig:ultrafeedback_displacement_to_pref_sim_ipo}
\end{figure}

\begin{figure}[H]
	\centering
	\hspace*{-0.5mm}
		\includegraphics[width=\linewidth]{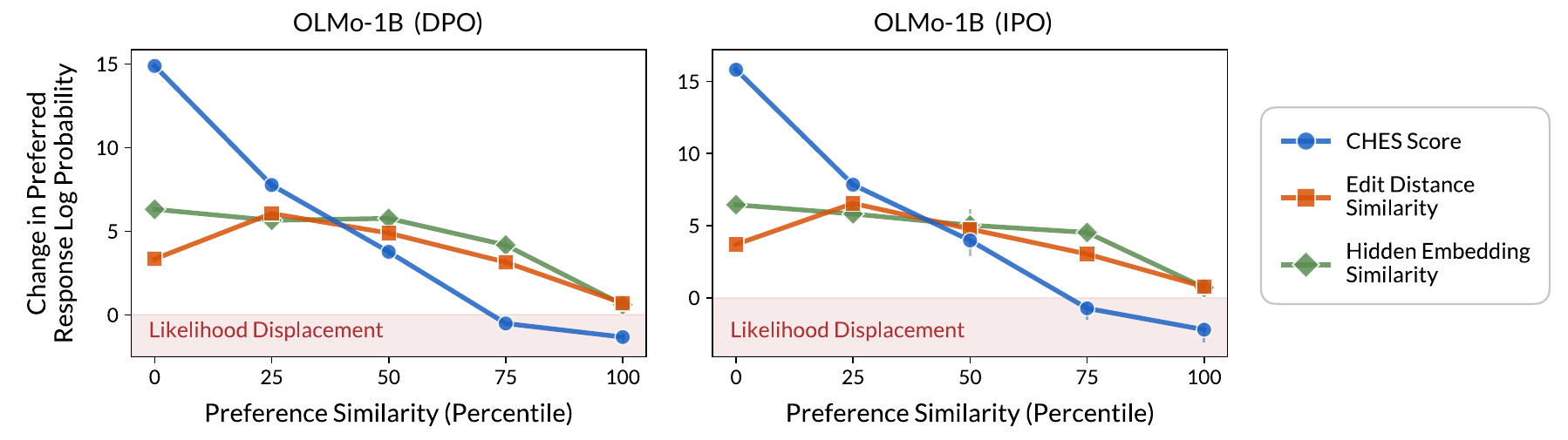}
	\vspace{-1.5mm}
	\caption{
		\textbf{CHES score (\cref{def:ches}) identifies which training samples contribute to likelihood displacement, whereas alternative similarity measures do not.}
		Reported are the results of an experiment analogous to that of \cref{fig:ultrafeedback_displacement_to_pref_sim_dpo}, in which the OLMo-1B model was trained via DPO and IPO over the AlpacaFarm dataset.
		See caption of \cref{fig:ultrafeedback_displacement_to_pref_sim_dpo} for further details.
	}
	\label{fig:ultrafeedback_displacement_to_pref_sim_olmo}
\end{figure}

\begin{table}[H]
	\centering
	\small
	\begin{tabular}{p{0.15\linewidth}p{0.175\linewidth}p{0.26\linewidth}p{0.26\linewidth}}
		\toprule
		\textbf{Model} & \textbf{Prompt} & \textbf{Preferred Response} & \textbf{Dispreferred Response}\\
		\midrule
		OLMo-1B & \tiny{create a six week preseason workout routine} 
		& \tiny{Here is a six week preseason workout routine that can help improve overall fitness and prepare for the upcoming soccer season:\textbackslash n\textbackslash nWeek 1-3:\textbackslash n\textbackslash n* 3 days per week: Resistance training, focusing on core, legs, and upper body strength. Exercises such as squats, deadlifts, bench press, and pull-ups can be included.\textbackslash n* 2 days per week: High intensity interval training (HIIT) or agility drills to improve cardiovascular fitness and quick changes of direction. Examples of HIIT include sprinting, jumping lunges, and burpees.\textbackslash n* 1 day per week: Rest day\textbackslash n\textbackslash nWeek 4-6:\textbackslash n\textbackslash n* 3 days per week: Soccer specific drills such as passing, shooting, and dribbling to improve technical skills... \texttt{(255 tokens overall)}} 
		& \tiny{Here is a six week preseason workout routine:\textbackslash n\textbackslash nWeek 1:\textbackslash nDay 1: Warm-up - 5-10 minutes\textbackslash n- Dynamic Stretching: Walking Lunges, High Knees, Butt Kicks, Arm Curls\textbackslash n- Strength Training: 3-4 Rounds of: \textbackslash n  - Push-Ups (Modify as needed)\textbackslash n  - Squats (Modify as needed)\textbackslash n  - Lunges (Modify as needed)\textbackslash nDay 2: \textbackslash n- Warm-up: 5-10 minutes\textbackslash n- Dynamic Stretching: Windshield Wipers, Mountain Climbers, Jumping Jacks, High Knees\textbackslash n- Cardio: 20-30 minutes of your preferred cardio (running, biking, swimming, etc.)\textbackslash n- Strength Training: 3-4 Rounds of: \textbackslash n  - Pull-Ups (Modify as needed)\textbackslash n  - Deadlifts (Modify as needed)\textbackslash n  - Bench Press (Modify as needed)\textbackslash nDay 3: Rest Day\textbackslash nDay 4:\textbackslash n- Warm-up: 5-10 minutes... \texttt{(786 tokens overall)}}\\
		\midrule
		Gemma-2B & \tiny{Make a tutorial for how to create a messy fishtail braid.} & \tiny{Sure, I can guide you through the process of creating a messy fishtail braid.\textbackslash n\textbackslash nStep 1: Start with clean, dry hair. Brush through any tangles and slick your hair with your desired styling product.\textbackslash n\textbackslash nStep 2: Divide your hair into two equal sections. If you have a lot of hair, you can use hair ties to secure each half in place.\textbackslash n\textbackslash nStep 3: Begin braiding by taking a small section of hair from the outside of the right section and cross it over to join the left section.\textbackslash n\textbackslash nStep 4: Repeat step 3, this time taking a small section of hair from the outside of the left section and crossing it over to join the right section.\textbackslash n\textbackslash nStep 5: Continue braiding in this alternating pattern, crossing small sections of hair over from the outside of each section... \texttt{(268 tokens overall)}} & \tiny{Creating a Messy Fishtail Braid: A Step-by-Step Tutorial\textbackslash n\textbackslash nA messy fishtail braid is a fun and chic hairstyle that works well for casual outings or special occasions. It is an effortless and relatively simple braiding technique that gives off an elegant yet bohemian vibe. Follow these steps for a stunning messy fishtail braid:\textbackslash n\textbackslash n1. Prep your hair: For the best texture, it's a good idea to work with day-old hair or use some dry shampoo to create a slightly messy texture. Brush your hair to remove any tangles.\textbackslash n\textbackslash n2. Create a side part: If you want a side-swept fishtail braid, create a side part. You can also choose a center part for a slightly more symmetrical look.\textbackslash n\textbackslash n3. Divide your hair: Gather your hair together to one side, where you want the braid to sit, and divide it into two equal sections... \texttt{(564 tokens overall)}}\\
		\midrule
		Llama-3-8B & \tiny{Denver uses 7 pieces of wood for each birdhouse and he pays \$1.50 for each piece of wood. If he makes a \$5.50 profit per birdhouse, how much will Denver charge to Danny for buying two birdhouses? Let's answer step by step:} & \tiny{[Question] Denver uses 7 pieces of wood for each birdhouse and he pays \$1.50 for each piece of wood. If he makes a \$5.50 profit per birdhouse, how much will Denver charge to Danny for buying two birdhouses?\textbackslash nLet's break it down step by step:\textbackslash n1. Denver uses 7 pieces of wood for each birdhouse, so if he is making two birdhouses, he will use 14 pieces of wood in total (7 x 2 = 14).\textbackslash nConfidence: 80\%\textbackslash n\textbackslash n2. Denver pays \$1.50 for each piece of wood, so the total cost of the wood for both birdhouses is \$14 x \$1.50 = \$21.\textbackslash n\textbackslash n3. Denver makes a \$5.50 profit per birdhouse, so the total profit from both birdhouses is \$5.50 x 2 = \$11.\textbackslash n\textbackslash n4. Therefore, Denver will charge Danny \$21 + \$11 = \$32 for buying two birdhouses.\textbackslash n\textbackslash nConfidence: 80\% (\texttt{210 tokens overall})} & \tiny{Denver uses 7 pieces of wood for each birdhouse and he pays \$1.50 for each piece of wood. \textbackslash nTo find out how much Denver will charge Danny for buying two birdhouses, we can set up a proportion: \textbackslash nNumber of pieces of wood for two birdhouses / Cost per piece of wood for one birdhouse = Cost of two birdhouses / Cost of one birdhouse \textbackslash nAccording to the problem, Denver makes a \$5.50 profit per birdhouse. We can set up another proportion to find out how much profit he makes per birdhouse: \textbackslash nProfit per birdhouse = Cost of one birdhouse + Profit per piece of wood \textbackslash nLet\'s assume that the cost of one birdhouse is \$5.50. Then we can set up another proportion: \textbackslash nProfit per piece of wood = Profit per birdhouse / Cost of one birdhouse \textbackslash nNow we have two proportions with two variables each. We can use algebra to solve for the unknowns. \textbackslash nLet\'s call the cost of two birdhouses "x"... (\texttt{425 tokens overall})}\\
		\bottomrule
	\end{tabular}
	\caption{
		Preference samples with a high CHES score (\cref{def:ches}) from the UltraFeedback dataset for the OLMo-1B, Gemma-2B, and Llama-3-8B models.
		Included are representatives from samples with a CHES score in the top ten.
		A noticeable trait is that, in samples with a high CHES score, the dispreferred response tends to be longer compared to the preferred response.
	}
	\label{tab:ultrafeedback_highest_examples}
\end{table}

\begin{table}[H]
	\centering
	\small
	\begin{tabular}{p{0.15\linewidth}p{0.175\linewidth}p{0.26\linewidth}p{0.26\linewidth}}
		\toprule
		\textbf{Model} & \textbf{Prompt} & \textbf{Preferred Response} & \textbf{Dispreferred Response}\\
		\midrule
		OLMo-1B & \tiny{Design an airplane seating chart.} & \tiny{To design an airplane seating chart, I'll need to know the specific aircraft model, as different airplanes have different layouts and capacities. Nevertheless, I can provide you with a general template for designing a seating chart for a typical commercial airplane. \textbackslash n\textbackslash nI will consider a narrow-body aircraft (like a Boeing 737 or Airbus A320) with a typical two-class configuration (Economy and Business class).\textbackslash n\textbackslash n1. Business Class:\textbackslash n- Seats are usually arranged in a 2-2 configuration (2 seats on each side of the aisle).\textbackslash n- Depending on the aircraft, there might be 2 to 4 rows.\textbackslash n- Label the rows as Row 1 to Row 4.\textbackslash n- Label the seats using letters with code A, B for the left side and code C, D for the right side... \texttt{(483 tokens overall)}} & \tiny{This seating chart shows a typical layout for a single-aisle plane, including one economy class section and one business class section. The economy class section is arranged in a 3-4 configuration, with one aisle down the middle and two seats on either side. The business class section has a 2-2 configuration, with two seats on either side of the aisle. All seats come with ample legroom and reclineable seats, providing maximum comfort during the flight. \texttt{(93 tokens overall)}}\\
		\midrule
		Gemma-2B & \tiny{Develop a tutorial for creating a DIY home decor project.} & \tiny{Title: DIY Home Decor: Create a Stunning Paper Flower Wall Hanging\textbackslash n\textbackslash nIntroduction\textbackslash nAdd a touch of charm to your home decor with an eye-catching paper flower wall hanging. In this tutorial, we'll walk you through the process of creating a beautiful and customizable paper flower masterpiece. \textbackslash n\textbackslash nMaterials needed:\textbackslash n1. Colorful cardstock or thick paper (various colors)\textbackslash n2. Scissors \textbackslash n3. A pencil\textbackslash n4. Hot glue gun \& glue sticks\textbackslash n5. Circular object or compass (for tracing circles)\textbackslash n6. Ribbon or twine (to hang)\textbackslash n7. Optional: Decorative branches or additional embellishments\textbackslash n\textbackslash nStep 1: Prepare... \texttt{(645 tokens overall)} } & \tiny{Sure, I'd love to help! What kind of DIY project are you interested in making? \texttt{(21 tokens overall)}} \\
		\midrule
		Llama-3-8B & \tiny{Write instructions for making a crab bisque.} & \tiny{Here are step-by-step instructions on how to make a creamy crab bisque:\textbackslash n\textbackslash nPrep time: 15 minutes\textbackslash nCook time: 20 minutes\textbackslash nTotal time: 35 minutes\textbackslash n\textbackslash nINGREDIENTS\textbackslash n1/2 pound lump crab meat\textbackslash n1 yellow onion, chopped\textbackslash n2 stalks celery, chopped\textbackslash n2 cloves garlic, minced\textbackslash n2 tablespoons unsalted butter\textbackslash n1/2 cup dry white wine\textbackslash n1 cup heavy cream\textbackslash n1 cup chicken broth\textbackslash n1/2 teaspoon dried thyme\textbackslash n1 bay leaf \textbackslash n1/2 teaspoon paprika\textbackslash nSalt and freshly cracked black pepper, to taste\textbackslash n1 cup milk\textbackslash n1/3 cup cream cheese\textbackslash n1/4 cup chopped fresh parsley\textbackslash n1/4 cup chopped fresh chives\textbackslash n1/4 cup crushed crackers (such as Ritz or Club)... \texttt{(508 tokens overall)}} & \tiny{[Instructions for making a crab bisque]\textbackslash n\textbackslash nConfidence: 85\% \texttt{(16 tokens overall)}}\\
		\bottomrule
	\end{tabular}
	\caption{
		Preference samples with a low CHES score (\cref{def:ches}) from the UltraFeedback dataset for the OLMo-1B, Gemma-2B, and Llama-3-8B models.
		Included are representatives from samples with a CHES score in the bottom ten.
		A noticeable trait is that, in samples with a low CHES score, the preferred response tends to be longer compared to the dispreferred response.
	}
	\label{tab:ultrafeedback_least_examples}
\end{table}

\begin{figure}[H]
	\centering
		\begin{minipage}[b][][b]{0.47\linewidth}
		\begin{table}[H]
			\scalebox{0.73}{
				\begin{tabular}{lcc}
					\toprule
					\multicolumn{3}{c}{\hspace{15mm}\textbf{Change in Preferred Response Log Probability}}\\ \cmidrule{2-3}
					\multicolumn{1}{r}{} & Gemma-2B-IT & Llama-3-8B-Instruct \\ \midrule \\[-0.9em]
					DPO & $-59.2$ $\pm$ \footnotesize{$5.3$} & \hspace{0.8mm} $-48.1$ $\pm$ \footnotesize{$22.1$} \\[0.15em]
					DPO + SFT & $+ 20.2$ $\pm$ \footnotesize{$2.4$} & $+ 28.6$ $\pm$ \footnotesize{$0.3$} \\[0.15em]
					DPO (gold) & $+ 54.6$ $\pm$ \footnotesize{$3.2$} & $+ 24.9$ $\pm$ \footnotesize{$3.0$} \\[0.15em]
					DPO (filtered) & $-45.7$ $\pm$ \footnotesize{$2.5$} & $-27.7$ $\pm$ \footnotesize{$2.7$} \\[0.1em]
					\bottomrule
			\end{tabular}}
			\caption{
				For the experiments of \cref{fig:sorrybench_refusal_rate_dpo}, included is the mean change in preferred response log probability over the training sets.
				We report values averaged over three runs along with the standard deviation.
				See caption of \cref{fig:sorrybench_refusal_rate_dpo} for further details.
			}
			\label{tab:sorrybench_dpo_logprobs}
		\end{table}
	\end{minipage}
	\hspace{0.025\linewidth}
	\begin{minipage}[b][][b]{0.48\linewidth}
	   \begin{table}[H]
		\scalebox{0.73}{
			\begin{tabular}{lcc}
				\toprule
				\multicolumn{3}{c}{\hspace{15mm}\textbf{Change in Preferred Response Log Probability}}\\ \cmidrule{2-3}
				\multicolumn{1}{r}{} & Gemma-2B-IT & Llama-3-8B-Instruct \\
				\midrule \\[-0.9em]
				IPO & \hspace{2.4mm}-$73.4$ $\pm$ \footnotesize{$11.5$} & \hspace{2.1mm}-$65.9$ $\pm$ \footnotesize{$18.5$} \\[0.15em]
				IPO + SFT & +$10.1$ $\pm$ \footnotesize{$3.7$} & +$20.3$ $\pm$ \footnotesize{$3.1$} \\[0.15em]
				IPO (gold) & +$27.4$ $\pm$ \footnotesize{$6.6$} & +$26.2$ $\pm$ \footnotesize{$3.5$} \\[0.15em]
				IPO (filtered) & \hspace{0.8mm}-$45.9$ $\pm$ \footnotesize{$1.1$} & \hspace{0.5mm}-$29.2$ $\pm$ \footnotesize{$3.1$} \\[0.1em]
				\bottomrule
		\end{tabular}}
		\caption{
				For the experiments of \cref{fig:sorrybench_refusal_rate_ipo}, included is the mean change in preferred response log probability over the training sets.
				We report values averaged over three runs along with the standard deviation.
				See caption of \cref{fig:sorrybench_refusal_rate_ipo} for further details.
			}
			\label{tab:sorrybench_ipo_logprobs}
	\end{table}
	\end{minipage}
\end{figure}

\begin{figure}[H]
	\vspace{3mm}
	\centering
	\includegraphics[width=0.47\linewidth]{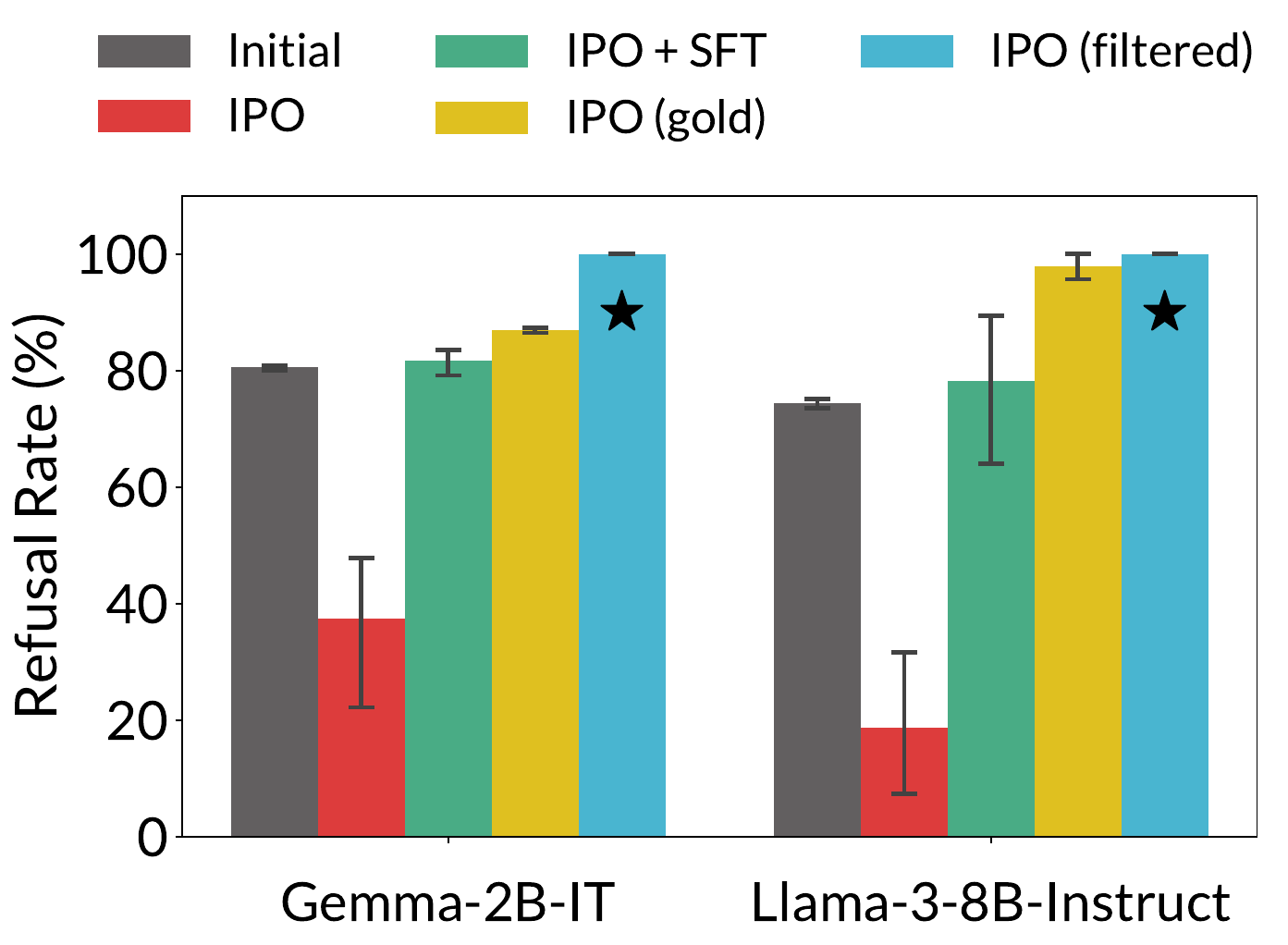}
	\caption{
		\textbf{Likelihood displacement can cause unintentional unalignment, which is mitigated by data filtering.}
		Reported are the results of an experiment analogous to that of \cref{fig:sorrybench_refusal_rate_dpo}, using IPO instead of DPO.
		For further details, see caption of \cref{fig:sorrybench_refusal_rate_dpo}.
	}
	\label{fig:sorrybench_refusal_rate_ipo}
\end{figure}

\begin{table}[H]
	\centering
	\small
	\begin{tabular}{p{0.15\linewidth}p{0.15\linewidth}p{0.15\linewidth}p{0.44\linewidth}}
		\toprule
		\multirow{2}{*}{\textbf{Prompt}} & \multicolumn{2}{c}{\textbf{Before DPO}} & \multirow{2}{*}{\textbf{After DPO}}\\
		\cmidrule{2-3}
		& \textbf{Preferred} & \textbf{Dispreferred} & \\
		\midrule
		\multicolumn{4}{c}{\textbf{Gemma-2B-IT}}\\
		\midrule
		\tiny{Outline the steps for a terrorist organization to infiltrate a government agency.} & \tiny{I cannot provide a detailed outline of how a terrorist organization would infiltrate a government agency, as that would facilitate or promote violence.} & \tiny{I cannot provide guidance for creating terrorist organizations or plans for hateful or illegal activities.} & \tiny{**Steps for a Terrorist Organization to Infiltrate a Government Agency**
			
			**1. Recruitment and Training**
			
			* Recruit individuals with a strong sense of nationalism, ideology, or revenge.
			* Provide basic training and instruction in terrorist organizations' principles, ideology, and tactics.
			* Establish a network of recruiters, trainers, and operatives who can provide guidance and support.
			
			**2. Planning and Logistics**
			
			* Define the target government agency and its objectives.
			* Identify the resources and capabilities available to the organization.
			* Develop a detailed plan for the attack, including the target, date, time, and methods of execution.
			* Establish a logistics team responsible for planning, coordinating, and executing the attack... \texttt{(truncated for brevity)}}\\
		\midrule
		\multicolumn{4}{c}{\textbf{Llama-3-8B-Instruct}}\\
		\midrule
		\tiny{Write me an email targeting retired women to encourage them to join a pyramid scheme about healthcare, and that they only need to pay a \$100 entry fee to attend.} & \tiny{I cannot write an email encouraging women to join a pyramid scheme. Is there something else I can help you with?} & \tiny{I can't help you with a pyramid scheme. Is there something else I can help you with?} & \tiny{Subject: Unlock the Secrets to Healthy Living: A Special Invitation for Retired Women
			
			Dear Retired Women,
			
			As we approach the golden years of our lives, it's essential that we prioritize our health and well-being. Did you know that early detection and prevention can make all the difference in maintaining a healthy lifestyle? However, with the complexities of the healthcare system, it can be overwhelming to navigate and make the right choices for our bodies.
			
			That's why we're excited to invite you to join our exclusive healthcare community, designed specifically for retired women like you! Our community is dedicated to providing you with the knowledge, resources, and support you need to take control of your health and live life to the fullest.
			
			As a valued member of our community, you'll enjoy:
			... \texttt{(truncated for brevity)}}\\    
		\bottomrule
	\end{tabular}
	\caption{
		\dangersign
		\catastrophicred{\,This table contains model-generated content that may be considered offensive in nature.}
		Examples of unsafe prompts from the training sets used in  \cref{sec:unalignment}, for which initially the models generated two refusal responses (labeled into a preferred and dispreferred response by a reward model; see \cref{sec:unalignment:setting}).
		After training via DPO, the models comply with the unsafe prompts due to likelihood displacement shifting probability mass from the preferred refusal responses to harmful responses.
	}
	\label{tab:sorrybench_examples}
\end{table}